\documentclass[12pt]{article}
\usepackage{amsmath,amssymb,amsthm,natbib}
\usepackage[linesnumbered, ruled]{algorithm2e}
\usepackage{graphicx,tikz,pgfplots,caption,subcaption,psfrag}
\usepackage{url} 

\newcommand{\blind}{0}

\numberwithin{equation}{section}
\theoremstyle{plain}
\newtheorem{example}{Example}
\newtheorem{theorem}{Theorem}[section]
\newtheorem{proposition}{Proposition}[section]

\newtheorem{lemma}{Lemma}
\newtheorem{assumpA}{}

\newtheorem{assumpB}{}

\newcommand{\E}{\mathbb{E}}
\newcommand{\R}{\mathbb{R}}
\newcommand{\F}{\mathcal{F}}
\newcommand{\p}{\mathcal{P}}

\begin{document}

	\def\spacingset#1{\renewcommand{\baselinestretch}%
		{#1}\small\normalsize} \spacingset{1}

	
	\if0\blind
	{	
	\title{Statistical Inference for Online Decision Making via Stochastic Gradient Descent}
	
	\author{Haoyu Chen, Wenbin Lu and Rui Song\\
		Department of Statistics, North Carolina State University\\
		\date{}}
	\maketitle
	} \fi

	\if1\blind
	{
	\bigskip
	\bigskip
	\bigskip
	\begin{center}
		{\LARGE\bf Statistical Inference for Online Decision Making via Stochastic Gradient Descent}
	\end{center}
	\medskip
	} \fi

	\bigskip
	\begin{abstract}
	Online decision making aims to learn the optimal decision rule by making personalized decisions and updating the decision rule recursively. It has become easier than before with the help of big data, but new challenges also come along. Since the decision rule should be updated once per step, an offline update which uses all the historical data is inefficient in computation and storage. 
	To this end, we propose a completely online algorithm that can make decisions and update the decision rule online via stochastic gradient descent. It is not only efficient but also supports all kinds of parametric reward models. Focusing on the statistical inference of online decision making, we establish the asymptotic normality of the parameter estimator produced by our algorithm and the online inverse probability weighted value estimator we used to estimate the optimal value. Online plugin estimators for the variance of the parameter and value estimators are also provided and shown to be consistent, so that interval estimation and hypothesis test are possible using our method. The proposed algorithm and theoretical results are tested by simulations and a real data application to news article recommendation.
	\end{abstract}
	
	\noindent%
	\vfill

	\newpage
	\spacingset{1.5} 
	\section{Introduction}\label{sec:intro}
	With the help of massive customer data, service providers from different fields like healthcare and online business can make personalized decisions to achieve better performance. One motivating example is news article recommendation \citep{li2010contextual}, where websites can analyze users' features to deliver the most suitable content and increase the click rate. Similar problems also exist in applications like precision medicine \citep{kim2011battle} and dynamic pricing \citep{qiang2016dynamic}. In all of these problems, the optimal decision rule needs to be learned from historical information but it is often not the best choice to carry out randomized trials and learn the rule offline. For instance, the timeliness of breaking news requires us to apply the recommendation rule as soon as possible. In clinical trials, it would be unethical to assign patients to randomized treatment if a more individualized and possibly better treatment is available. Online decision making aims to solve these problems by taking personalized actions during experiments and continuously improve the decision rule with accumulated information. 
	
	Following the seminal work of \citet{robbins1952some} and \citet{woodroofe1979one}, contextual multi-armed bandit, or contextual bandit for short, has been widely accepted as a basic setting for studying the online decision making problem. In this setting, the service provider observes a user's feature $X_t\in\R^p$ at each decision step $t$, which is independent and identically distributed across $t$, and decides to take action $A_t\in\mathcal{A}$ accordingly. The response of the user is coded as reward $Y_t$ such that larger value is preferable. The original goal of the contextual bandit problem is to maximize the cumulative reward up to step $T$ by making each decision based on historical information. Many solutions have been proposed for this problem, including upper confidence bound methods \citep{auer2002using, dani2008stochastic}, Thompson sampling \citep{agrawal2013thompson}, $\varepsilon$-greedy methods \citep{yang2002randomized, qian2016kernel}, and forced sampling methods \citep{goldenshluger2013linear, bastani2015online}. Readers are referred to the survey of \citet{tewari2017ads} for a comprehensive discussion of these methods. The key ingredient of all these solutions is a design addressing the trade-off between exploration and exploitation: the rule of decision making should be learned and improved by exploring insufficiently explored actions, but exploration may result in a lower reward than exploiting the currently learned rule. Through careful design of the online decision making algorithms, these solutions can balance exploration and exploitation and maximize the cumulative reward in an asymptotic sense.
	
	Compared to the performance of the online decision making algorithms in terms of cumulative reward, which has been extensively studied in the aforementioned literature, we care more about assessing the uncertainty of the decision rules and the mean reward they can achieve. After all, the decision rules are learned from random samples and no inference can be made about them unless we can quantify the uncertainty. Despite being important, inferential properties of online decision making have been less studied. \citet{chambaz2017targeted} considered a general parametric model of the mean reward. They used $\varepsilon$-greedy method with $\varepsilon$ being a function of the estimated treatment effect and gave the asymptotic distribution of the mean reward under the optimal decision rule. 
	\citet{chen2020statistical} studied the contextual bandit problem with a linear reward model and also adopted $\varepsilon$-greedy method but their $\varepsilon$ is a function of decision steps. They gave the asymptotic distributions for the reward model parameters and the expected reward under the optimal decision rule. Although the decisions are made in an online fashion in these two papers, the estimation of the rule and the expected reward is still offline. That is, we have to store the historical data from the very beginning and use them all for each update of the rule. This requires $\mathcal{O}(Tp)$ storage and is not efficient in computation as well when the sample size $T$ becomes large. Therefore, we set out to adapt the online decision making algorithm into a completely online one using Stochastic Gradient Descent (SGD).
	
	SGD algorithms have been widely used in applications with large datasets due to its memory and computational efficiency. Since SGD updates the estimation with one data point at each time, it coincides with the mechanism of how data are observed in online decision making and becomes a natural solution to the online estimation of the decision rule. Moreover, statistical inference of the SGD estimators are made possible by the classic work of \citet{ruppert1988efficient} and \citet{polyak1992acceleration}. They suggested using the averaged SGD estimator for fast convergence and established its asymptotic distribution. More recently, research in the statistical inference of SGD estimators such as asymptotic variance and interval estimation has gained popularity. \citet{chen2016statistical} proposed a plugin estimator and a batch-means estimator for the asymptotic variance and proved their consistency. However, the batch-means estimator tends to underestimate the variance in finite-sample studies due to the correlation between batches. \citet{fang2018online} designed an online bootstrap procedure by randomly perturbing the gradients so that they can estimate the variance of the SGD estimators using resampling. Their method shows its strength when the loss function of the SGD estimator is not twice-differentiable. But in cases where the Hessian of the loss function exists, the plugin estimator is still preferable as it saves time from generating bootstrap samples.
	
	In this paper, we study the online decision making problem in a contextual bandit setting and use $\varepsilon$-greedy method to address the exploration-and-exploitation dilemma. Our main contributions are 1) proposing a completely online decision making algorithm  that scales easily for big datasets and 2) deriving inferential results of the decision rule produced by the algorithm and the expected reward under the optimal rule. Our online decision making algorithm is based on SGD but the gradients are modified using inverse probability weighting (IPW) so that it becomes possible to establish the asymptotic normality of the weighted SGD estimator. The algorithm enables us not only to make decisions online but also to estimate the model parameters, the expected reward under the optimal rule, and their variances online. Since the algorithm does not have to store all the historical information, the storage is only $\mathcal{O}(p^2)$ if we want to estimate the variance of the parameter estimators. It also achieves computational efficiency by updating the stored data such as the second moment and the Hessian online instead of calculating them from the historical data. Another benefit of using SGD algorithm is that we are not restricted to linear reward model anymore, which is a limit of the most parametric solutions of contextual bandit \citep[see e.g.,][]{goldenshluger2013linear, bastani2015online}. Our method works for any parametric reward model as long as a suitable loss function can be found, and the negative log-likelihood function is often a good choice. 
	The inference of the model parameters is non-trivial due to the inherent data dependence in online decision making and the asymptotic normality of the weighted SGD estimator is established by fully exploiting its martingale structure. We also propose an online IPW estimator for the expected reward under the optimal decision rule and show its asymptotic normality under a margin assumption. 
	
	The rest of the paper is organized as follows. We first introduce the proposed SGD algorithm for online decision making in Section \ref{sec:algo}. Inferential results for the SGD parameter estimator and the IPW value estimator are given in Section \ref{sec:para} and Section \ref{sec:value} respectively. Simulation results of the estimators under linear and logistic reward models are presented in Section \ref{sec:numerical} and a real data analysis using Yahoo! Today module user click log data is presented in Section \ref{sec:real}. Finally, we discuss some of the potential extensions to our work in Section \ref{sec:discuss}. Proofs of the main results and extended simulation results are presented in the Appendix.

	\section{The Proposed Algorithm}\label{sec:algo}
	\subsection{Online decision making with epsilon-greedy}\label{sec:algo1}
	Recall that the available data at each decision point is a triplet $O = (X, A, Y)$ consisting of feature, action, and reward. Here we consider a binary action space $\mathcal{A} = \{0, 1\}$ for all decision steps. Define a decision rule $d: \R^p \mapsto \mathcal{A}$ as a mapping from the feature space to the action space. We are interested in the optimal decision rule $d^{opt}(X) = \arg\max_{A\in\mathcal{A}}\E(Y|A, X)$. It is obvious that one way to find $d^{opt}$ is to estimate the conditional mean outcome of each action when feature $X$ is given, aka the Q-function, and choose the action that yields the largest conditional mean outcome. Assuming that the conditional distribution of $Y$ given $A, X$ is fixed across decisions, we can posit a parametric model for the Q-function
	\begin{equation}
		\E(Y|A, X) = \mu(A, X; \beta),
	\end{equation}
	where $\beta \in \mathcal{B}\subseteq \R^{2p}$ consists of $p$ parameters for each action and $\mathcal{B}$ is the parameter space. Let $\beta_{[1:p]}$ denote the vector consisting of the first $p$ elements of $\beta$ and $\beta_{[p+1:2p]}$ denote the other half. The Q-function is then
	$$
	    \E(Y|A, X) = (1 - A)\mu_0(X; \beta_{[1:p]}) + A\mu_1(X; \beta_{[p+1:2p]}), 
	$$
	where $\mu_0$ and $\mu_1$ are the parametric models for $\E(Y|A = 0, X)$ and $\E(Y|A = 1, X)$ respectively. For the discussion below, we use the concatenated vector $\beta$ instead of two separate vectors to spare notations, but keep in mind that half of the parameters are redundant when the action is specified.\footnote{For example, $\beta_{[1:p]}$ is not used in $\mu(1, X; \beta)$. The same rule applies to the true parameter $\beta_0$ and the estimators $\hat{\beta}$, $\hat{\beta}_t$ and $\bar{\beta}_t$ that are introduced below.} Let $\beta_0 \in \mathcal{B}$ be the true value of $\beta$. Then, the optimal decision is
	\begin{equation}\label{eq:oracle}
		d^{opt}(X) = I\{\mu(1, X; \beta_0) > \mu(0, X; \beta_0)\},
	\end{equation}
	or equivalently $I\{\mu_1(X, \beta_{0, [1:p]}) > \mu_0(X, \beta_{0, [p+1:2p]})\}$ with $\beta_{0, [1:p]}$ and $\beta_{0, [p+1:2p]}$ representing the two halves of $\beta_0$. Let $\hat{\beta}$ denote an estimator of $\beta_0$. The corresponding estimated optimal decision rule is given by
	\begin{equation*}
		\hat{d}^{opt}(X) = I\{\mu(1, X, \hat{\beta}) > \mu(0, X, \hat{\beta})\}.
	\end{equation*}
	The fundamental idea of online decision making is to make decisions based on $\hat{\beta}$, and update it recursively using newly acquired data. 
	
	In order to update the parameter estimator using gradient descent, we have to specify a loss function $\ell(\beta; O)$ that measures the difference between the estimated reward $\mu(A, X; \beta)$ and the true reward $Y$. For example, the quadratic loss $\{Y - \mu(A, X; \beta)\}^2/2$ can be used when the reward is continuous. In general, the loss function can be constructed as 
	$$
	\ell(\beta; O) = (1 - A)\ell_0(\beta_{[1:p]}; X, Y) + A\ell_1(\beta_{[p+1:2p]}; X, Y),
	$$
	where $\ell_0$ and $\ell_1$ are the usual loss functions for a regression of $Y$ on $X$ when $A$ is $0$ and $1$ respectively, e.g., $\ell_0(\beta_{[1:p]}; X, Y) = \{Y - \mu_0(X; \beta_{[1:p]})\}^2/2$ and $\ell_1(\beta_{[p+1:2p]}; X, Y) = \{Y - \mu_1(X; \beta_{[p+1:2p]})\}^2/2$ in the above example. Given a series of predetermined learning rates $\{\alpha_t\}$, the update rule of the original SGD algorithm \citep{robbins1951stochastic} is 
	\begin{equation}\label{eq:sgd}
		\hat{\beta}_t = \hat{\beta}_{t-1} - \alpha_t \nabla\ell(\hat{\beta}_{t-1}; O_t).
	\end{equation}
	Suppose we start from an initial estimate $\hat{\beta}_0$ and use (\ref{eq:sgd}) to obtain $\hat{\beta}_1, \cdots, \hat{\beta}_{t}$ after observing $O_1, \cdots, O_t$. As suggested by \citet{polyak1992acceleration}, we use the average $\bar{\beta}_t = t^{-1}\sum_{s=1}^t \hat{\beta}_s$ as the final estimator to accelerate the estimation. The estimated optimal decision rule after step $t$ is
	\begin{equation}\label{eq:estopt}
		\hat{d}^{opt}_{t}(X) = I\{\mu(1, X, \bar{\beta}_t) > \mu(0, X, \bar{\beta}_t)\}.
	\end{equation}
	To address the exploration-and-exploitation dilemma, $\hat{d}^{opt}_{t}(X_{t+1})$ is not used directly as the next action, but $\varepsilon$-greedy policy is adopted instead to explore the other action with a small probability. At each decision step $t$, the propensity score $\pi(X) = P\{d(X)=1| X\}$ is determined by
	\begin{equation}\label{eq:pi}
		\pi_t(X) = (1-\varepsilon_t)I\{\mu(1, X, \bar{\beta}_t) > \mu(0, X, \bar{\beta}_t)\} + \frac{\varepsilon_t}{2},
	\end{equation}
	where $\{\varepsilon_t\}$ is a series of predetermined exploration rate. The $\varepsilon$-greedy decision $d_t(X_{t+1})$, later collected into $O_{t+1}$ as $A_{t+1}$, is then sampled from a Bernoulli distribution with success probability $\pi_t(X_{t+1})$. So this policy will choose the better action under the currently estimated optimal decision rule with probability $1-\varepsilon_t/2$ and choose the inferior option with probability $\varepsilon_t/2$.
	
	\subsection{SGD with weighted gradients}\label{sec:algo2}
	In the settings studied by \citet{chen2016statistical} and \citet{fang2018online}, the observed data $\tilde{O}_t = (\tilde{X}_t, \tilde{Y}_t)$ are i.i.d. and there is no decision making process\footnote{To distinguish between the i.i.d. setting and the online decision making setting, we use the tilde symbol to mark the data, the conditional mean response model and loss functions from the i.i.d. settings and use $b$ to denote the parameters.}. Assume that the conditional expectation of $\tilde{Y}$ follows a parametric model $\E(\tilde{Y}|\tilde{X}) = \tilde{\mu}(\tilde{X}; b)$. Then their loss function has the form $\tilde{\ell}(b; \tilde{O}_t) = \tilde{\ell}(b; \tilde{X}_t, \tilde{Y}_t)$, which is also i.i.d.\,for any fixed $b$. So the expectation of the loss function can be easily defined as 
	$$
	\tilde{L}(b) = \iint \tilde{\ell}(b; x, y)d\p_{\tilde{Y}|\tilde{X}}(y|x)d\p_{\tilde{X}}(x),
	$$
	where $\p_{\tilde{X}}$ is the distribution of ${\tilde{X}}$ and $\p_{\tilde{Y}|\tilde{X}}(y|x)$ is the conditional distribution of $\tilde{Y}$ given $\tilde{X}$.
	Rewrite the Robbins-Monro updating rule $\hat{b}_t = \hat{b}_{t-1} - \alpha_t \nabla\ell(\hat{b}_{t-1}; \tilde{O}_t)$ as
	\begin{equation}\label{eq:partition}
		\hat{b}_t = \hat{b}_{t-1} - \alpha_t \nabla \tilde{L}(\hat{b}_{t-1}) + \alpha_t \{\nabla \tilde{L}(\hat{b}_{t-1}) - \nabla\tilde{\ell}(\hat{b}_{t-1}; \tilde{O}_t)\}.
	\end{equation}
	We can see that $-\nabla \tilde{L}(\hat{b}_{t-1})$ is the main force that pushes $\hat{b}_t$ towards the minimizer of $\tilde{L}(b)$ and the extra part $\{\nabla \tilde{L}(\hat{b}_{t-1}) - \nabla\ell(\hat{b}_{t-1}; \tilde{O}_t)\}$ can be seen as some random disturbance. In i.i.d. settings (see the left graph of Figure \ref{fig:structure}), the randomness only comes from the new data $\tilde{O}_t$ when $\hat{b}_{t-1}$ is given. Providing interchangeability of expectation and derivative, we have
	\begin{equation}\label{eq:cond}
		\nabla \tilde{L}(\hat{b}_{t-1}) = \E\{ \nabla\tilde{\ell}(\hat{b}_{t-1}; \tilde{O}_t) | \bar{\tilde{O}}_{t-1} \},
	\end{equation}
	since $\hat{b}_{t-1}$ is a function of $\bar{\tilde{O}}_{t-1} = \{\tilde{O}_1, \cdots, \tilde{O}_{t-1}\}$. Define $\tilde{\F}_t$ as the $\sigma$-field generated by $\bar{\tilde{O}}_t$, (\ref{eq:cond}) implies that $\{\nabla \tilde{L}(\hat{b}_{t-1}) - \nabla\tilde{\ell}(\hat{b}_{t-1}; \tilde{O}_t)\}_{t\ge 1}$ is a martingale difference process with respect to $\{\tilde{\F}_t\}_{t\ge 1}$. Then the asymptotic normality of the averaged SGD estimator can be established by martingale central limit theorem \citep{polyak1992acceleration}.
	
	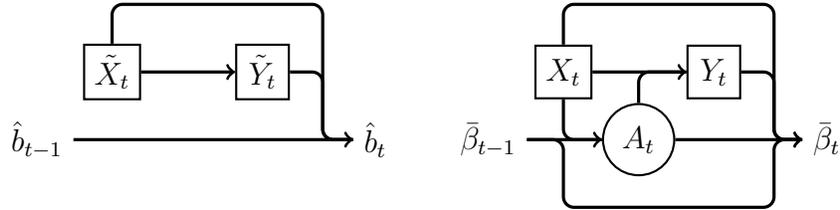
\begin{figure}[!htbp]
		\centering
		\begin{tikzpicture}[
		roundnode/.style={circle, draw=black, thick, minimum size=7mm},
		squarednode/.style={rectangle, draw=black, thick, minimum size=7mm},
		]
		\node[roundnode](a) at (3,0) {$A_t$};
		\node[](abtm1) at (1, 0) {$\bar{\beta}_{t-1}$};
		\node[](abt) at (5.5, 0) {$\bar{\beta}_{t}$};
		\node[squarednode](ax) at (2, 0.9) {$X_t$};
		\node[squarednode](ay) at (4, 0.9) {$Y_t$};
		
		\draw[very thick, rounded corners, ->] (abtm1.east) -- (2, 0) -- (2, -0.9) -- (4.8, -0.9) -- (4.8, 0) -- (abt.west);
		\draw[very thick, ->] (abtm1.east) -- (a.west);
		\draw[very thick, rounded corners, -] (ax.south) -- (2, 0) -- (a.west);
		\draw[very thick, ->] (ax.east) -- (ay.west);
		\draw[very thick,rounded corners, -] (a.north) -- (3, 0.9) -- (ay.west);
		\draw[very thick, ->] (a.east) -- (abt.west);
		\draw[very thick,rounded corners, -] (ay.east) -- (4.8, 0.9) -- (4.8, 0) -- (abt.west);
		\draw[very thick,rounded corners, -] (ax.north) -- (2, 1.8) -- (4.8, 1.8) -- (4.8, 0) -- (abt.west);
		
		\node[](btm1) at (-5, 0) {$\hat{b}_{t-1}$};
		\node[](bt) at (-.5, 0) {$\hat{b}_{t}$};
		\node[squarednode](x) at (-4, 0.9) {$\tilde{X}_t$};
		\node[squarednode](y) at (-2, 0.9) {$\tilde{Y}_t$};
		
		\draw[very thick, rounded corners, ->] (btm1.east) -- (bt.west);
		\draw[very thick, ->] (x.east) -- (y.west);
		\draw[very thick,rounded corners, -] (y.east) -- (-1.2, 0.9) -- (-1.2, 0) -- (bt.west);
		\draw[very thick,rounded corners, -] (x.north) -- (-4, 1.8) -- (-1.2, 1.8) -- (-1.2, 0) -- (bt.west);
		\end{tikzpicture}
		\caption{Data dependence structure in an SGD update for online learning with i.i.d. data (left) and online decision making (right).}
		\label{fig:structure}
	\end{figure}
	
	However, the martingale structure is not an immediate result from our setting. As shown in the right graph of Figure \ref{fig:structure}, the action at step $t$ is decided by the previous parameter estimate $\bar{\beta}_{t-1}$ and the current feature $X_t$, and it will influence the reward $Y_t$ through the underlying true model $\mu(A_t, X_t; \beta_0)$. Then all the observed data $O_t = (X_t, A_t, Y_t)$ together with $\bar{\beta}_{t-1}$ determine the next estimate $\bar{\beta}_{t}$ through the SGD update. The inherent data dependence problem of online decision making makes its statistical inference more challenging. 
	
	First, the definition of the expected loss $L(\beta)$ now involves the distribution of $A_t$, which is not the same for different decision rules. Since the $\varepsilon$-greedy decision $d_t(X)$ changes with $\bar{\beta}_t$, we cannot use its distribution as the distribution of $A$ to define $L(\beta)$. Instead, a fixed decision rule should be used to determine the joint distribution of $O=(X, A, Y)$. For simplicity, we consider the random decision rule which selects either action with equal probability. Then the expected loss is 
	\begin{equation}\label{eq:loss}
		L(\beta) = \iiint \ell(\beta; x, a, y)d\p_{Y|X, A}(y|x, a)d\p_A^r(a)d\p_X(x),
	\end{equation}
	where $\p_X$ is the distribution of $X$, $\p_A^r$ is $\mathrm{Bernoulli}(1/2)$ and $\p_{Y|X, A}$ is the conditional distribution of $Y$ given $X$ and $A$. To spare notation, we denote $\p_{O}^r$ as the joint distribution of $O$ when $A$ follows $\p_A^r$ and write $\E_{\p_O}$ to note that the expectation is taken with respect to $O$ following some distribution $\p_O$. Then $L(\beta)$ can also be expressed as $\E_{\p_{O}^r} \ell(\beta; O)$.
	
	Similar to (\ref{eq:partition}), we can rewrite the Robbins-Monro updating rule (\ref{eq:sgd}) as 
	$$
		\hat{\beta}_t = \hat{\beta}_{t-1} - \alpha_t \nabla L(\hat{\beta}_{t-1}) + \alpha_t \{\nabla L(\hat{\beta}_{t-1}) - \nabla \ell(\hat{\beta}_{t-1}; O_t)\}.
	$$
	The second problem is now the series of random disturbances $\nabla L(\hat{\beta}_{t-1}) - \nabla\ell(\hat{\beta}_{t-1}; O_t)$ is no longer a martingale difference process under the definition (\ref{eq:loss}). Define $\F_t$ as the $\sigma$-field generated by $\bar{O}_t = \{O_1, \cdots, O_t\}$. We have $\E\{\ell(\beta; O_t)|\F_{t-1}\}\ne L(\beta)$ because no matter what fixed distribution of $A$ is used in the definition of $L(\beta)$, it is almost always different from the true distribution of $A_t$ following the $\varepsilon$-greedy policy. 
	
	Since the action distribution $\p_A^r$ in (\ref{eq:loss}) is fixed to $\mathrm{Bernoulli}(1/2)$, we can decompose the expected loss as
	$$
    	L(\beta) = \frac{1}{2}\E\;\ell_0(\beta_{[1:p]}; X, Y) + \frac{1}{2}\E\;\ell_1(\beta_{[1+p:2p]}; X, Y) =: L_0(\beta_{[1:p]}) + L_1(\beta_{[1+p:2p]}).
	$$
	Therefore minimizing $L$ in $\beta$ is equivalent to minimizing $L_0$ in $\beta_{[1:p]}$ and minimizing $L_1$ in $\beta_{[p+1:2p]}$ separately. Note that $\nabla\ell(\beta; O)$ can also be divided into two parts,
	$$
	    (1 - A)\nabla\ell_0(\beta_{[1:p]}; X, Y) + A\nabla\ell_1(\beta_{[p+1:2p]}; X, Y),
	$$
	which means we can update $\hat{\beta}_{t,[1:p]}$ using $\nabla\ell(\hat{\beta}_{t-1}; O_t)$ when $A_t = 0$ and update $\hat{\beta}_{t,[p+1:2p]}$ when $A_t = 1$. The second problem can then be solved by correcting the sampling distribution of $A_t$ towards $\p_A^r$ for $A_t = 0$ and $1$ separately. Inspired by the importance sampling ratio method used in off-policy reinforcement learning \citep{sutton2016emphatic, sutton2018reinforcement}, we propose to replace $\nabla\ell(\hat{\beta}_{t-1}; O_t)$ in (\ref{eq:sgd}) with the IPW gradient
	\begin{equation}\label{eq:grad}
		g(\hat{\beta}_{t-1}; O_t) = \frac{\nabla\ell(\hat{\beta}_{t-1}; O_t)I\{A_t = 1\}}{2\pi_{t-1}(X_t)} + \frac{\nabla\ell(\hat{\beta}_{t-1}; O_t)I\{A_t = 0\}}{2\{1 - \pi_{t-1}(X_t)\}}.
	\end{equation}
	This gradient is named IPW because it can be seen as the average of two inverse probability weighted derivatives. From another perspective, it actually corrects the sampling distribution of $A_t$ towards $\p_A^r$ by importance sampling. Notice that the propensity score is $1/2$ when $A$ follows $\p_A^r$, therefore the importance sampling ratios are $(1/2)/\pi_{t-1}(X_t)$ for $A_t=1$ and $(1/2)/\{1-\pi_{t-1}(X_t)\}$ for $A_t=0$. Since $\pi_t$ is on the denominator, our algorithm requires the exploration rate $\varepsilon_t$ be strictly bigger than zero.
	
	\begin{algorithm}[!htbp]
		\DontPrintSemicolon
		\SetKwInOut{Input}{Initialize}
		\caption{Online decision Making via SGD}\label{alg:1}
		
		\KwIn{$\hat{\beta}_0=\bar{\beta}_0=0$, $\pi_0=1/2$, $\alpha_t$, $\varepsilon_t$}
		\BlankLine
		\For{$t = 1$ \KwTo $T$}{
			Observe $X_t$\;
			Sample $A_t$ from Bernoulli$(\pi_{t-1}(X_t))$\;
			Observe $Y_t$, form $O_t = (X_t, A_t, Y_t)$\;
			Calculate the IPW gradient $g(\hat{\beta}_{t-1}; O_t)$ according to (\ref{eq:grad})\;
			Update $\hat{\beta}_t = \hat{\beta}_{t-1} - \alpha_t g(\hat{\beta}_{t-1}; O_t)$\;
			Update $\bar{\beta}_t =  \{\hat{\beta}_{t} + (t-1) \bar{\beta}_{t-1}\}/t$\;
			Update $\pi_t(X)$ according to (\ref{eq:pi})\;
		}
	\end{algorithm}
	
	Putting the $\varepsilon$-greedy decision rule and the IPW gradient together, Algorithm \ref{alg:1} presents the whole process of online decision making via SGD. Denote $\p_{O}^\pi$ as the joint distribution of $O_t$ under the proposed decision policy in Algorithm \ref{alg:1}. We can check that if  expectation and derivative are interchangeable, 
	\begin{equation}\label{eq:cond2}
		\E_{\p_{O}^\pi}\{ g(\hat{\beta}_{t-1}; O_t) | \F_{t-1}\} = \E_{\p_{O}^r} \{\nabla\ell(\hat{\beta}_{t-1}; O) | \F_{t-1} \} = \nabla L(\hat{\beta}_{t-1}).
	\end{equation}
	Therefore the martingale structure can be recovered from the online decision making setting and Theorem 2 in \citet{polyak1992acceleration} can be applied to show the asymptotic normality of our parameter estimators.
	
	\section{Parameter Inference}\label{sec:para}
	In this section, we provide the asymptotic distribution of the parameter estimator from Algorithm \ref{alg:1} and give a consistent online estimator of its asymptotic variance. For the discussions below, we use $\lVert\cdot\rVert$ to represent the Euclidean norm of vectors and $\langle\cdot, \cdot\rangle$ for the inner product of two vectors.
	
	Rewrite the update of $\hat{\beta}_t$ as
	\begin{equation*}
		\hat{\beta}_t = \hat{\beta}_{t-1} - \alpha_t g(\hat{\beta}_{t-1}; O_t) = \hat{\beta}_{t-1} - \alpha_t \{R(\hat{\beta}_{t-1}) - \xi_t\},
	\end{equation*}
	where $R(\beta) = \nabla L(\beta)$ and $\xi_t = R(\hat{\beta}_{t-1}) - g(\hat{\beta}_{t-1}; O_t)$. It follows from (\ref{eq:cond2}) that $\{\xi_t\}_{t\ge 1}$ is a martingale difference process wrt $\{\F_t\}_{t\ge 1}$ but the following assumption is needed to ensure the interchangeability of expectation and derivative.
	
	\begin{assumpA}\label{as:ell}
		The loss function $\ell(\beta; O)$ is integrable for any $\beta$ and continuously differentiable in $\beta$ for any $O$. The collection of functions $\{\lVert\nabla\ell(\beta; O)\rVert: \beta\in\mathcal{B}\}$ is uniformly integrable so that $\E_{\p_{O}^r} \{\nabla\ell(\beta; O) \} = \nabla L(\beta)$.
	\end{assumpA}
	
	We further assume the following conditions are satisfied.
	
	\begin{assumpA}\label{as:L}
		The expected loss $L(\beta)$ as defined in (\ref{eq:loss}) satisfies
		\begin{enumerate}
			\item $L(\beta)$ is continuously differentiable and convex, and it has a unique minimizer $\beta^* = \arg\min_{\beta}L(\beta)$.
			\item $\nabla L(\beta)$ is $L_1$-Lipschitz continuous, that is, for any $\beta_1, \beta_2\in \mathcal{B}$, $\lVert\nabla L(\beta_1) - \nabla L(\beta_2)\rVert \le L_1 \lVert\beta_2 - \beta_1\rVert$.
			\item The covariance matrix of $\nabla\ell(\beta; O)$, $\Sigma(\beta) = \E_{\p_O^r}[\nabla\ell(\beta; O)\{\nabla\ell(\beta; O)\}^T]$ exists. The Hessian matrix $H(\beta) = \nabla^2 L(\beta)$ exists and is $L_2$-Lipschitz continuous at $\beta^*$, and $H = H(\beta^*)$ is positive definite.
		\end{enumerate}
	\end{assumpA}
	
	\begin{assumpA}\label{as:diff}
		There exists $L_3 > 0$ such that for all $\beta \in \mathcal{B}$,
		$$
		\E_{\p_{O}^r}[\lVert \nabla\ell(\beta; O) - \nabla\ell(\beta^*; O)\rVert^2]\le L_3\lVert\beta - \beta^*\rVert^2.
		$$
	\end{assumpA}
	
	Assumption \ref{as:L} is adopted by \citet{fang2018online} and it is weaker than the original assumptions made in \citet{polyak1992acceleration}. Note that the assumptions and inference can only be made on $\beta^*$ but we are interested in the inference of the true model parameter $\beta_0$, so the loss function should be chosen such that its minimizer $\beta^* = \beta_0$. Assuming the true parametric form of the reward model is known, the negative log-likelihood function is often a good choice to build the connection between $\beta^*$ and $\beta_0$. However, the procedure is valid as long as $\beta^* = \beta_0$ even if the likelihood is misspecified. For example, if we know $Y|A,X$ is normally distributed with constant variance $\sigma^2$, then its negative log-likelihood function $\{Y - \mu(A,X;\beta)\}^2/2\sigma^2 + constant$ suggests the quadratic loss $\{Y - \mu(A,X;\beta)\}^2/2$. However, the quadratic loss, which we refer to as a ``working negative log-likelihood", is still valid even if the normality and homoscedasticity assumptions are violated. As we will show in two examples that come after the main theorem, Assumptions \ref{as:ell} to \ref{as:diff} are satisfied and $\beta^*=\beta_0$ with the loss functions we choose. With these assumptions in place, we obtain the main theorem for the proposed estimators.
	
	\begin{theorem}\label{thm:paradist}
		If Assumptions \ref{as:ell} to \ref{as:diff} are met, the learning rate $\alpha_t = \alpha t^{-\gamma}$ with $\alpha > 0$ and $\gamma\in(0.5, 1)$ and the exploration rate $\varepsilon_t\to \varepsilon_\infty > 0$, then
		$$
		\sqrt{t}(\bar{\beta}_t - \beta^*)\overset{d}{\to}\mathcal{N}(0, H^{-1}S(H^{-1})^T),
		$$
		where
		$$
		S = \frac{1}{4}\int\left\{ \frac{\Sigma_1(\beta^*; X)}{\pi^*(X)} + \frac{\Sigma_0(\beta^*; X)}{1 - \pi^*(X)}\right\}d\p_X,
		$$
		$\Sigma_i(\beta; X) = \E[\nabla\ell(\beta; X, i, Y)\{\nabla\ell(\beta; X, i, Y)\}^T| X]$ for $i=0,1$ and $\pi^*(X) = (1 - \varepsilon_\infty) I\{\mu(1, X, \beta^*) > \mu(0, X, \beta^*)\} + \varepsilon_\infty/2$.
	\end{theorem}
	
	The proof of Theorem \ref{thm:paradist} is provided in Appendix \ref{sec:proofparadist}. Basically, we first establish the almost sure convergence of $\bar{\beta}_t$. The conditional covariance matrix of $g(\beta_{t-1}, O_t)$ will then converge to $S$, which can be expressed as a weighted average of the covariance matrices of $\nabla\ell$ for different actions. The asymptotic normality then follows from Theorem 2 of \citet{polyak1992acceleration}.
	
	The learning rate $\alpha_t$ has two tuning parameters $\alpha$ and $\gamma$. Although $\gamma$ does not appear in the asymptotic results, it may still play an important role in the non-asymptotic analysis as shown by \cite{moulines2011non} in the i.i.d. cases. They proved that if the loss function $\tilde{L}$ is strongly-convex with minimizer $b_0$, then the average SGD estimator $\bar{b}_t = t^{-1}\sum_{s=1}^t\hat{b}_s$ satisfies $(\E \lVert\bar{b}_t - b_0\rVert^2)^{1/2} = O(t^{-1})$; and if $\tilde{L}$ is not strongly-convex but $\tilde{l}$ has bounded gradients $\nabla \tilde{l}(b)$ such as in the logistic regression case, then the bound on $\E \{\tilde{L}(\bar{b}_t) - \tilde{L}(b_0)\}$ is $O(t^{\gamma -1})$ when $\gamma\in(1/2, 1)$, suggesting setting $\gamma$ as close to $1/2$ as possible. Similar non-asymptotic bounds may also be derived in the online decision making setting to guide the choice of $\gamma$. For the numerical study below, we will set $\gamma = 0.501$ and then $\alpha$ can be tuned according to the loss.
	
	In order to provide statistical inference for the model parameter, we have to estimate the variance of $\bar{\beta}_t$, and the variance estimator should also be updated online without storing all historical data. One simple choice is the online plugin estimator considered by \citet{chen2016statistical} and \citet{fang2018online}. In our setting, the online plugin estimators for $S$ and $H$ are given by
	$$
	\hat{S}_t = \frac{1}{t}\sum_{s=1}^t g(\bar{\beta}_{s-1}; O_s)[g(\bar{\beta}_{s-1}; O_s)]^T
	$$
	and
	$$
	\hat{H}_t = \frac{1}{t}\sum_{s=1}^t \nabla^2\ell(\bar{\beta}_{s-1}; O_s)\left[\frac{I\{A_s = 1\}}{2\hat{\pi}_{s-1}(X_s)} + \frac{I\{A_s = 0\}}{2(1 - \hat{\pi}_{s-1}(X_s))}\right].
	$$
	It can be seen that we only need $\mathcal{O}(p^2)$ storage to calculate the Hessian and estimate the variance of $\bar{\beta}_t$. The consistency of the online plugin estimator can be established under the following additional assumption.
	
	\begin{assumpA}\label{as:var}
		Denote $f_H(\beta) = \nabla^2\ell(\beta; x, a, Y)$ and $f_S(\beta) = \nabla\ell(\beta; x, a, Y)\{\nabla\ell(\beta; x, a, Y)\}^T$. Then for any $x \in \R^p$, $a\in \mathcal{A}$, $\beta\in\mathcal{B}$, $v\in\R^{2p}$ and $\varkappa > 0$, there exist constants $C_H, C_S > 0$ such that
		\begin{gather*}
			P\{v^T f_H(\beta)v > \varkappa\} \le C_H P\{v^T f_H(\beta^*)v > \varkappa\},\\
			P\{v^T f_S(\beta)v > \varkappa\} \le C_S P\{v^T f_S(\beta^*)v > \varkappa\}.
		\end{gather*}
	\end{assumpA}
	
	Assumption \ref{as:var} ensures the summands of $\hat{H}_t$ and $\hat{S}_t$ are sufficiently close to stationary so that the martingale convergence result \citep[Theorem 2.19]{hall1980martingale} can apply. It holds trivially for cases where $f_H(\beta)$ and $f_S(\beta)$ are bounded for any $\beta$ or they do not involve $\beta$ at all. In general, we would have \ref{as:var} if $\mathcal{B}$ is a compact set. 
	
	\begin{theorem}\label{thm:paravar}
		If the conditions of Theorem \ref{thm:paradist} and Assumption \ref{as:var} are satisfied, then $\hat{S}_t \overset{p}{\to} S$ and $\hat{H}_t \overset{p}{\to} H$.
	\end{theorem}
	
	We consider the following two motivating examples to illustrate our model setting.
	
	\begin{example}[Linear reward model]\label{ex:linear} 
		Assume the true conditional mean reward function $\mu(A, X; \beta_0)$ takes the following linear form
		\begin{equation}\label{eq:u}
			u(A, X; \beta_0) = (1-A)X^T\beta_{0,[1:p]} + AX^T\beta_{0,[p+1:2p]}.
		\end{equation}
		The true reward $Y$ is generated by $\mu(A, X, \beta_0) + E$, where $E$ is a random error with mean zero and variance $\sigma^2$, and it is independent of $A$ and $X$. We 
		consider the quadratic loss
		\begin{equation*}
			\ell(\beta; O) = \frac{1}{2}\{Y - u(A, X; \beta)\}^2.
		\end{equation*}
		The expected loss is then
		\begin{align*}
			L(\beta) = &\frac{1}{4}(\beta - \beta_0)^T\{I_2\otimes \E(XX^T)\}(\beta - \beta_0) + \frac{\sigma^2}{2}, 
		\end{align*}
		where $I_2$ is the $2\times 2$ identity matrix and $\otimes$ is the Kronecker product. It is obvious that $L(\beta)$ is convex and its minimizer is $\beta^*=\beta_0$. The loss function is twice differentiable in $\beta$ with the first derivative $\nabla\ell(\beta; O) = \{u(A, X; \beta)- Y\}\nabla u(A, X, \beta)$ and the second derivative $\nabla^2\ell(\beta; O) = \nabla u(A, X, \beta)\{\nabla u(A, X, \beta)\}^T$, where $\nabla u(A, X, \beta) = ((1-A)X^T, AX^T)^T$. Given $\nabla\ell(\beta; O)$, the update rule can be easily derived from (\ref{eq:grad}). Note that $\nabla L(\beta) = 4^{-1}I_2\otimes \E(XX^T)(\beta - \beta_0)$ and $\nabla^2 L(\beta) = 4^{-1}I_2\otimes \E(XX^T)$, so Assumptions \ref{as:ell} to \ref{as:diff} can be easily verified.
	\end{example}
	
	\begin{example}[Logistic reward model]
		When the outcomes are binary, such as clicking or not in the news article recommendation example, the most simple and popular statistical model we would fit is the logistic model. Assume the true conditional mean reward function is given by 
		\begin{equation}\label{eq:logistic}
			\mu(A, X; \beta_0) = P(Y = 1 | A, X; \beta_0) = \frac{1}{1+e^{-u(A, X; \beta_0)}},
		\end{equation}
		where $u$ is the same linear function in (\ref{eq:u}). 
		We consider the cross entropy loss function, which is also the negative log-likelihood function
		\begin{equation*}
			\ell(\beta; O) = -Y\log\mu(A, X; \beta) - (1 - Y)\log\{1-\mu(A, X; \beta)\}.
		\end{equation*}
		Its first derivative is $\nabla\ell(\beta; O) = \{\mu(A, X; \beta) - Y\}\nabla u(A, X; \beta)$ and second derivative is $\nabla^2 \ell(\beta; O) = \{\mu(A, X; \beta) - Y\}^2\nabla u(A, X; \beta)\{\nabla u(A, X; \beta)\}^T$. The expected loss is 
		$$
		L(\beta) = -\frac{1}{2}\sum_{i\in\{0,1\}}\E[\log\{1-\mu(i, X; \beta)\} + \mu(i, X; \beta_0)u(i, X; \beta)].
		$$ 
		Therefore $\nabla L(\beta) = -2^{-1}\sum_{i\in\{0,1\}}\E[\{\mu(i, X; \beta) - \mu(i, X; \beta_0)\}\nabla u(i, X; \beta)] = \E_{\p_O^r}\nabla\ell(\beta; O)$ and $\nabla^2 L(\beta) = \E_{\p_O^r}\nabla^2\ell(\beta; O)$ is positive definite. The unique minimizer of $L(\beta)$ is $\beta^* = \beta_0$. Other assumptions can be easily verified.
	\end{example}
	
	
	\section{Value Inference}\label{sec:value}
	The value of a given decision rule $d$ is defined as 
	\begin{equation}
		V(d) = \E[\E\{Y|d(X), X\}],
	\end{equation}
	which is $\int \mu(d(X), X; \beta_0) d\p_X$ if the model is correctly specified. We are often interested in knowing the value of the optimal decision rule $V(d^{opt})$ and its estimation in the offline setting has been extensively studied \citep[see][for an introduction to some of the most popular estimation methods]{tsiatis2019introduction}. However, if we have a large amount of streaming data, estimating $V(d^{opt})$ using all the historical data at each observance of new data is inefficient in general and impossible in our algorithm since we do not store all the data. Therefore we must find a way to recursively update the value estimator at each step. Here we consider the inverse probability weighted estimator proposed by \cite{zhang2012robust} and extend it to an online version. Recall the estimated optimal decision rule is  $\hat{d}^{opt}_t(X) = I\{\mu(1, X; \bar{\beta}_t) > \mu(0, X; \bar{\beta}_t)\}$. Define the decision consistency indicator $C_t = I\{A_t = \hat{d}^{opt}_{t-1}(X_t)\}$ and the propensity for decision consistency $\pi_{C, t} = P(C_t = 1 | \F_{t-1}, X_t)$. The online inverse probability weighted estimator for $V(d^{opt})$ is then
	\begin{equation}\label{eq:valueest}
		\hat{V}_t(d^{opt}) = \frac{1}{t}\sum_{s=1}^t \frac{C_s Y_s}{\pi_{C, s}} = \frac{1}{t}\frac{C_t Y_t}{\pi_{C, t}} + \frac{t-1}{t}\hat{V}_{t-1}(d^{opt}).
	\end{equation}
	In our setting, the propensity for decision consistency $\pi_{C, t}$ is known to be $1-\varepsilon_t/2$ so we do not have to estimate it. By construction, the value estimator can be updated online without the storage of previous data. Intuitively, inverse probability weighting corrects the bias from the random exploration and hence $C_tY_t/\pi_{C,t}$ is an unbiased estimator of $V(\hat{d}^{opt}_t)$. After each update, the value estimator in (\ref{eq:valueest}) becomes closer to the true value of $\hat{d}^{opt}_t$, which goes to $V(d^{opt})$ since the estimated decision rule converges to $d^{opt}$. Therefore the online IPW estimator for $V(d^{opt})$ is consistent. Furthermore, its asymptotic normality can be shown under two extra assumptions.
	
	\begin{assumpA}\label{as:data}
		The features vector $X$ satisfies $\E \lVert X \rVert < \infty$. The second moment of reward exists for any given covariates and action, that is, 
		$$
		\theta^2(A, X) = \E(Y^2|A, X) <\infty.
		$$
	\end{assumpA}
	
	\begin{assumpA}\label{as:margin}
		There exists $C > 0$ and $\tau > 0$, such that for $X\sim\mathcal{P}_X$,
		$$
		P(0 < |\mu(1, X, \beta_0) - \mu(0, X, \beta_0)| \le \rho) \le C \rho^\tau, \; \forall \rho > 0.
		$$
	\end{assumpA}
	
	Assumption \ref{as:data} is a mild condition on the boundedness of the observed data. Assumption \ref{as:margin} is a margin condition originating from the classification literature \citep{audibert2007fast} and its stronger version where $\tau = 1$ is adopted by \cite{goldenshluger2013linear} and \cite{bastani2015online} to control the complexity of the contextual bandit problem. We should note that the margin assumption is not needed to establish the consistency of the IPW value estimator but it can lead to a clear and estimable asymptotic variance of the value estimator. Intuitively, when the feature $X$ lies near the decision boundary $\mu(1, X, \beta_0) = \mu(0, X, \beta_0)$, we are often unable to detect the difference in expected rewards and tend to make wrong decisions which bring more variation to the value estimator. Assumption \ref{as:margin} makes value inference easier by restricting the probability of observing such features, and the following results are obtained with its help. 
	
	\begin{theorem}\label{thm:valdist}
		If Assumptions \ref{as:ell}, \ref{as:L}, \ref{as:diff}, \ref{as:data} and \ref{as:margin} are met, the learning rate $\alpha_t = \alpha t^{-\gamma}$ with $\alpha > 0$ and $\gamma\in(1/2, 1)$ and the exploration rate $\varepsilon_t\to \varepsilon_\infty > 0$, then
		\begin{equation*}
			\sqrt{t}\{\hat{V}_t(d^{opt})-V(d^{opt})\}\overset{d}{\to} \mathcal{N}(0,\eta^2),
		\end{equation*}
		where 
		\begin{equation*}
			\eta^2 = \frac{2}{2-\varepsilon_\infty}\int\theta^2(d^{opt}(X), X)d\p_X-\{V(d^{opt})\}^2.
		\end{equation*}
	\end{theorem}
	The variance of $\hat{V}_t(d^{opt})$ can also be estimated using the online plugin estimator,
	\begin{equation*}
		\hat{\eta}_t^2 = \frac{2}{2-\varepsilon_t}\frac{1}{t}\sum_{s=1}^t \frac{C_s Y_s^2}{\pi_{C, s}} -\{\hat{V}_t(d^{opt})\}^2,
	\end{equation*}
	and the variance estimator is consistent.
	\begin{theorem}\label{thm:valvar}
		Under the same conditions of Theorem \ref{thm:valdist}, $\hat{\eta}_t^2 \overset{p}{\to} \eta^2$.
	\end{theorem}
	
	The proofs of Theorem \ref{thm:valdist} and \ref{thm:valvar} are provided in Appendix \ref{sec:proofvaldist}.
	
	\section{Numerical Studies}\label{sec:numerical}
	In this section, we carry out finite sample experiments to test the SGD online decision making algorithm, investigate the performance of the parameter and value estimators and their variance estimators, and compare it with other estimation methods\footnote{Code for the numerical studies is at \url{https://github.com/ideechy/Online-Decision-Making}.}. We consider the linear and logistic reward model settings that are discussed in Section \ref{sec:para}. For both models, the feature vector $X_t$ is set to $(1, X_{t,2}, X_{t,3})$, where $X_{t,2}$ and $X_{t,3}$ are generated independently from the standard normal distribution, so $p=3$ in our experiments. The true parameter is $\beta_0=(0.3, -0.1, 0.7, 0.8, 0.5, -0.4)^T$. The random error $E$ in the linear reward model is generated from the normal distribution with mean zero and variance $\sigma^2 = 0.01$. We also consider a linear reward model with the same configuration but a bigger noise ($\sigma^2 = 0.25$). The results are similar to those of the linear model with $\sigma^2 = 0.01$ and are shown in Appendix \ref{sec:extension}.
	
	We experiment with three different exploration rates. The first two are fixed rates with $\varepsilon_t = 0.1$ and $0.2$, and the other one is a decreasing rate $\varepsilon_t = t^{-0.3}\vee 0.1$. We use the first $50$ samples for pure exploration as a burn-in period, so $\varepsilon_t = 1$ for $t\le 50$ for all exploration rates. The learning rate is set as $\alpha_t = \alpha t^{-0.501}$ and $\alpha$ is tuned by comparing the loss $\ell(\bar{\beta}_{t-1}; O_t)$ for each setting.
	
	\subsection{Parameter and value estimation}\label{sec:simulation}
	
	\begin{figure}[!htbp]
	    \centering
	    \begin{subfigure}{0.75\textwidth}
	        \centering
    	    \includegraphics[width=\linewidth]{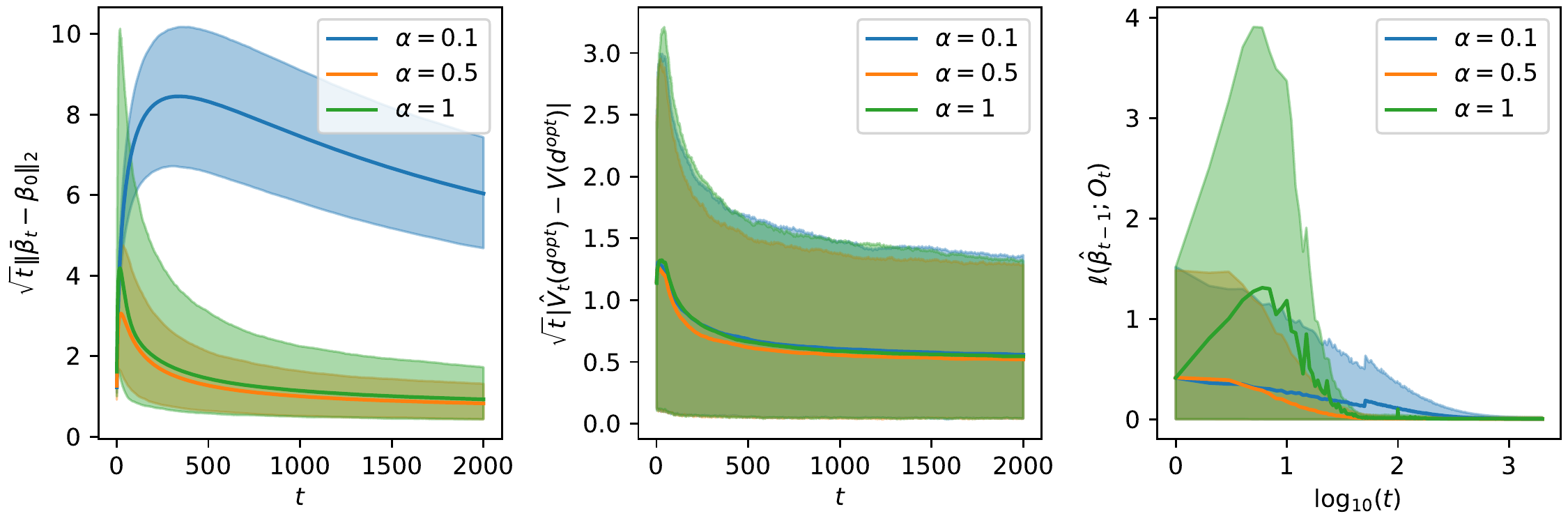}
    	    \caption{Linear reward model}
    	\end{subfigure}%
    	\\
    	\begin{subfigure}{0.75\textwidth}
	        \centering
    	    \includegraphics[width=\linewidth]{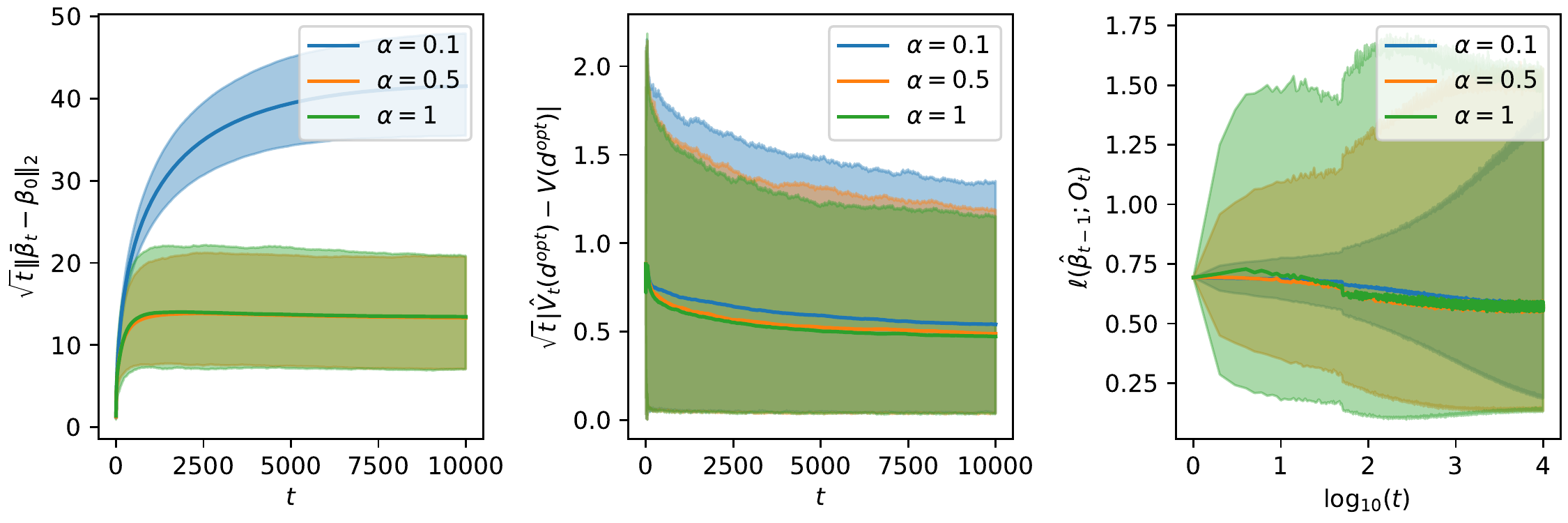}
    	    \caption{Logistic reward model}
    	\end{subfigure}
	    \caption{Performance of the online decision making algorithm with different learning rates. The exploration rate is $\varepsilon_t = 0.2$. All experiments are repeated 5000 times. The solid lines are mean outcomes and the shaded regions are bounded by 5\% and 95\% percentiles of the outcomes.}
	    \label{fig:tuning_e2}
	\end{figure}
	
	Figure \ref{fig:tuning_e2} compares the estimation results and the loss for $\alpha = 0.1, 0.5$ and $1$ in finite $t$ when $\varepsilon_t = 0.2$. The results for the other exploration rates are similar and shown in the Appendix. The first two columns show the effect of $\alpha$ on the bias of the parameter and value estimators. The true value under the optimal decision rule $V(d^{opt})$ is calculated as the mean reward of $10^6$ i.i.d. users following the oracle decision (\ref{eq:oracle}). It can be seen that the value estimation is not sensitive to the choice of $\alpha$ but the non-asymptotic performance of the parameter estimator does depend on it. A bigger $\alpha$ will lead to higher variance since the step size is bigger. However, when $\alpha$ is too small, such as $0.1$ in the example, the parameter estimator will take longer to converge. Therefore a good choice of $\alpha$ should achieve balance between the bias and the variance. In practice, we can choose $\alpha$ by running several experiments with different values of $\alpha$ and comparing the loss. As shown in the last column of Figure \ref{fig:tuning_e2}, $\alpha = 0.5$ achieves the lowest average loss for both reward model settings when $\varepsilon_t = 0.2$ and it is the same for the other exploration rates. Therefore we will use $\alpha = 0.5$ for the experiments below.

	\begin{figure}[!htbp]
		\centering
		\begin{subfigure}{0.5\textwidth}
			\centering
			\includegraphics[width=\linewidth]{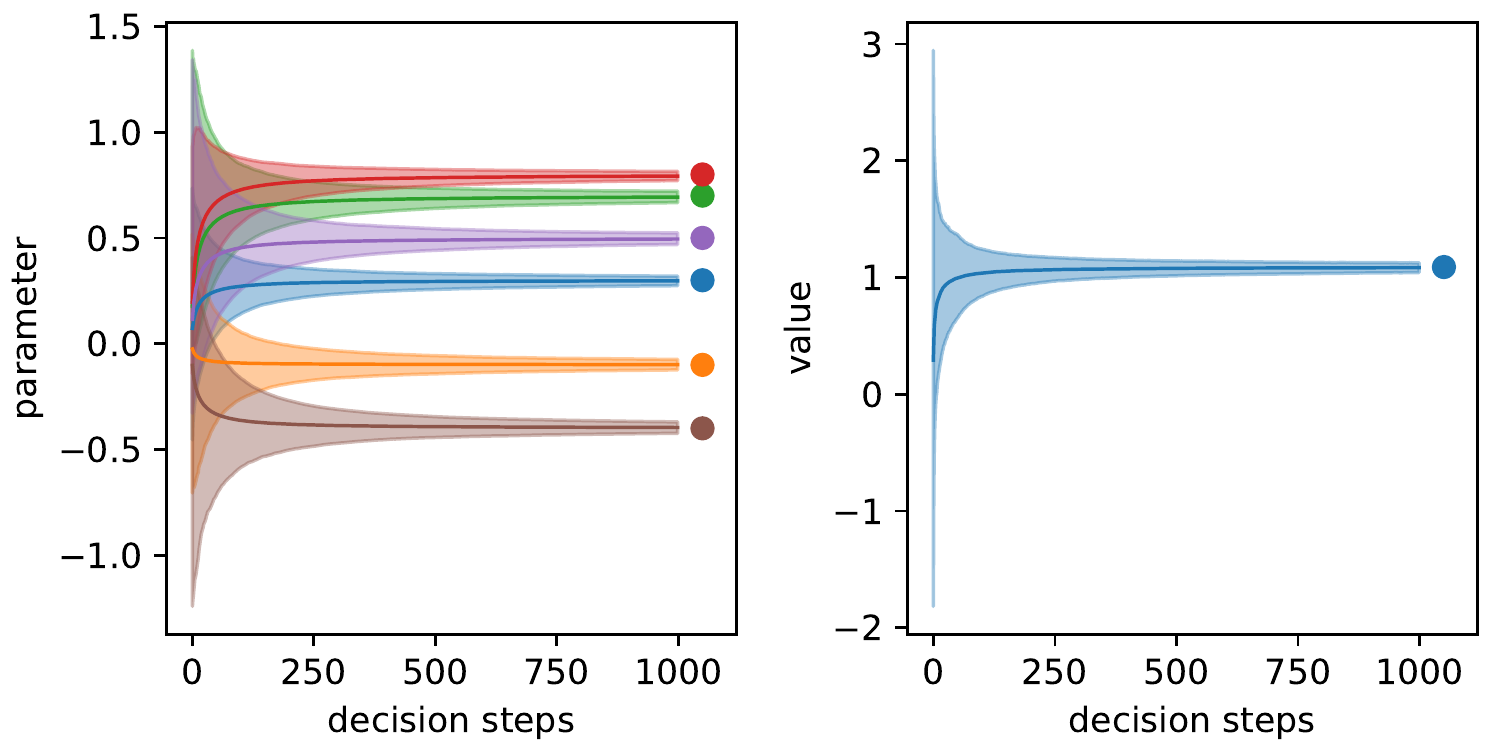}
			\caption{Linear reward model}
		\end{subfigure}%
		\begin{subfigure}{0.5\textwidth}
			\centering
			\includegraphics[width=\linewidth]{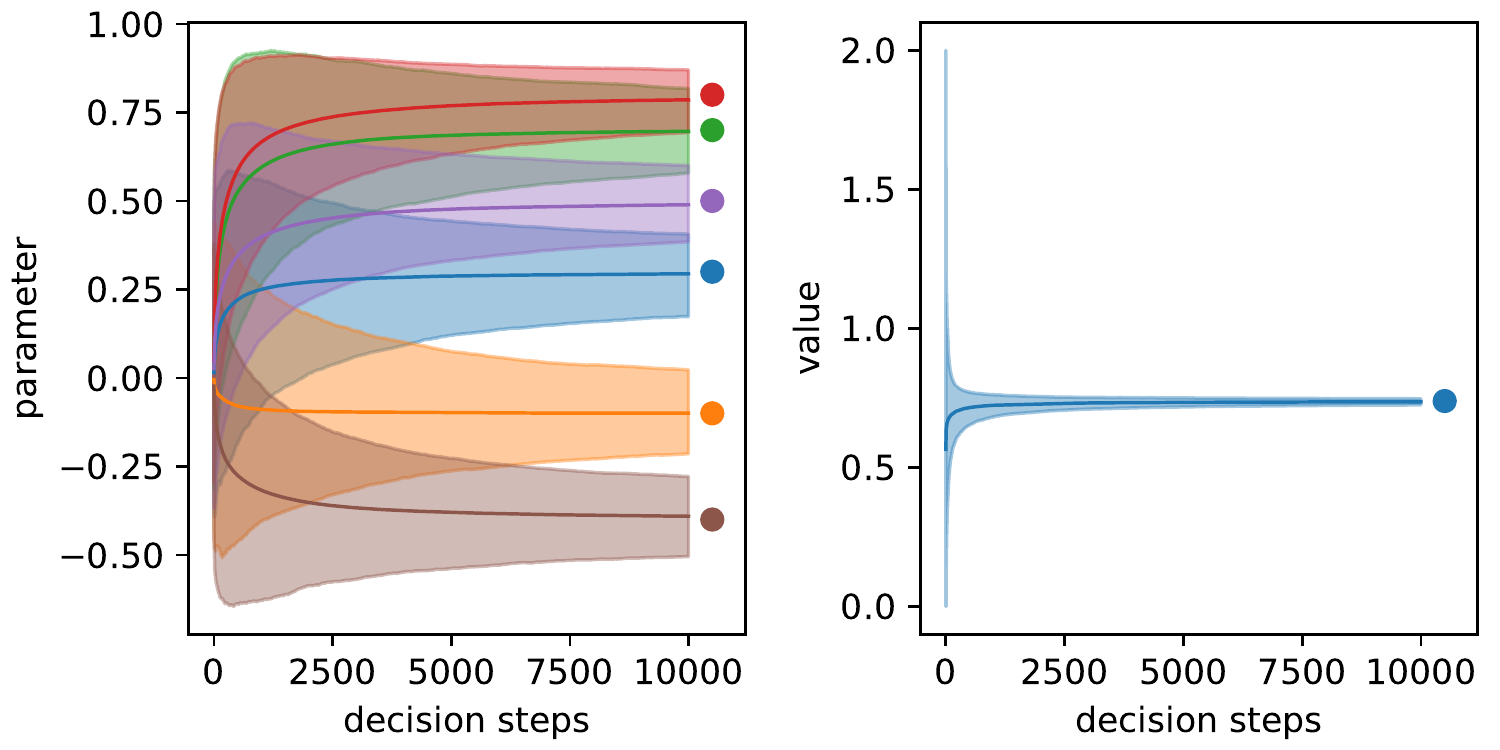}
			\caption{Logistic reward model}
		\end{subfigure}
		\caption{{Parameter and optimal value estimation from 5000 repeated experiments following the proposed SGD method with IPW gradients. The learning rate is $\alpha_t= 0.5t^{-0.501}$ and the exploration rate is $\varepsilon_t = 0.2$.} The solid lines are mean estimates and the shaded regions are bounded by 2.5\% and 97.5\% percentiles of the estimates. The points at the end of the lines mark the true value.}
		\label{fig:consist}
	\end{figure}
	
	As shown in Figure \ref{fig:tuning_e2}, the parameter and value estimators both converge to the truth for different choices of $\alpha$'s. Figure \ref{fig:consist} further illustrates the convergence of the parameter and value estimators when choosing $\alpha_t = 0.5t^{-0.501}$ and $\varepsilon_t = 0.2$. The results are similar for the other exploration rates. The parameter estimator of the linear reward model converges faster than that of the logistic reward model because the noise scale for the linear model is relatively small ($\sigma^2 = 0.01$). In the logistic setting, the value estimator converges much faster than the parameter estimator. This means the model parameter has a limited effect on the decision rule in our setting. Even though the parameter estimate is far from the truth in the early stage, the decision based on it is already the same as that based on the true parameter. In fact, if the reward model is linear in $\beta$, then $\beta$ and $k\beta$ ($k>0$) will result in the same decision. This simple fact also holds for generalized linear models, e.g., the logistic model we used here. When all parameter estimators converge at the same rate, they are very close to $k\beta_0$ for some $k > 0$ if the initial estimator $\hat{\beta}_0$ is set to zero, then chances are good that the decision based on them will coincide with the optimal decision. 
	
	\begin{figure}[!htbp]
		\centering
		\begin{subfigure}{0.5\textwidth}
			\centering
			\includegraphics[width=\linewidth]{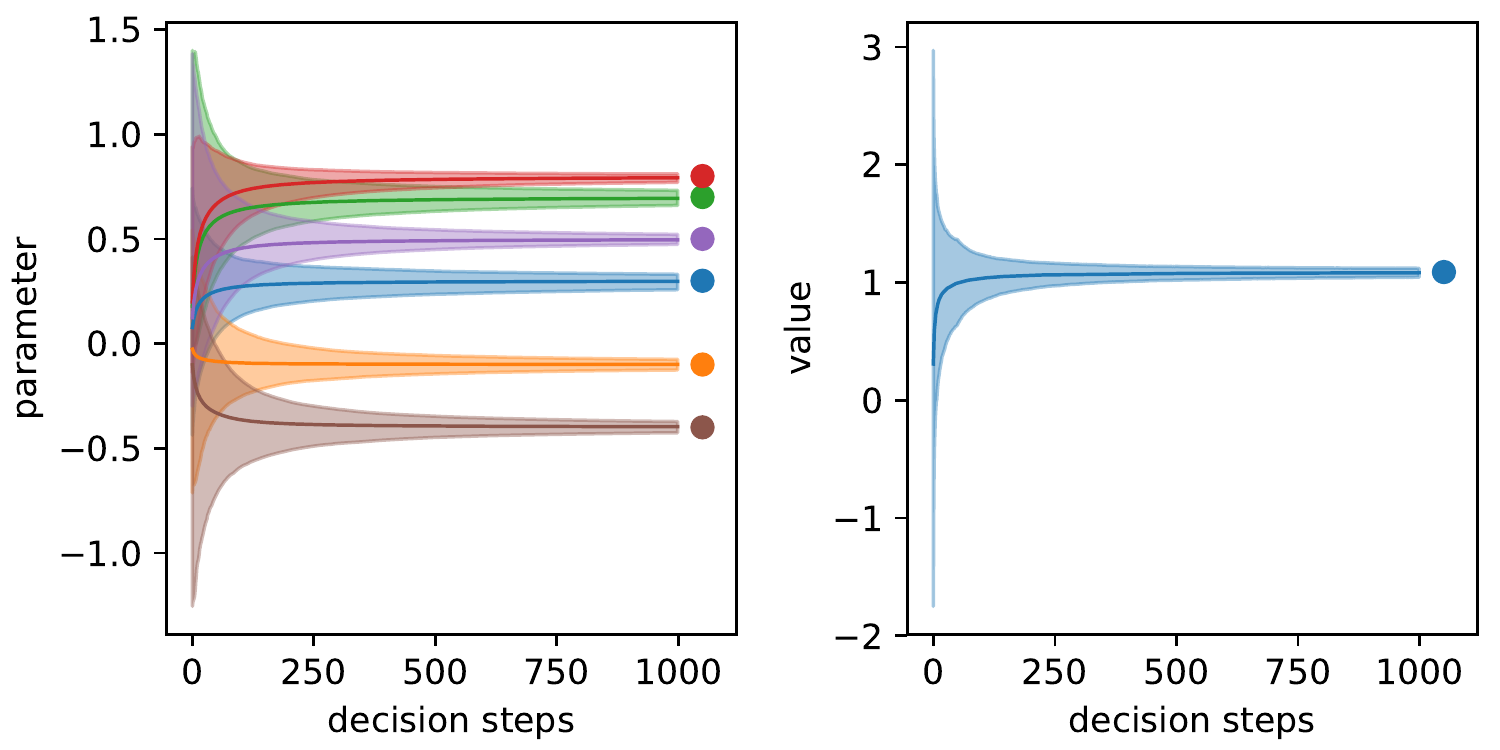}
			\caption{Linear reward model}
		\end{subfigure}%
		\begin{subfigure}{0.5\textwidth}
			\centering
			\includegraphics[width=\linewidth]{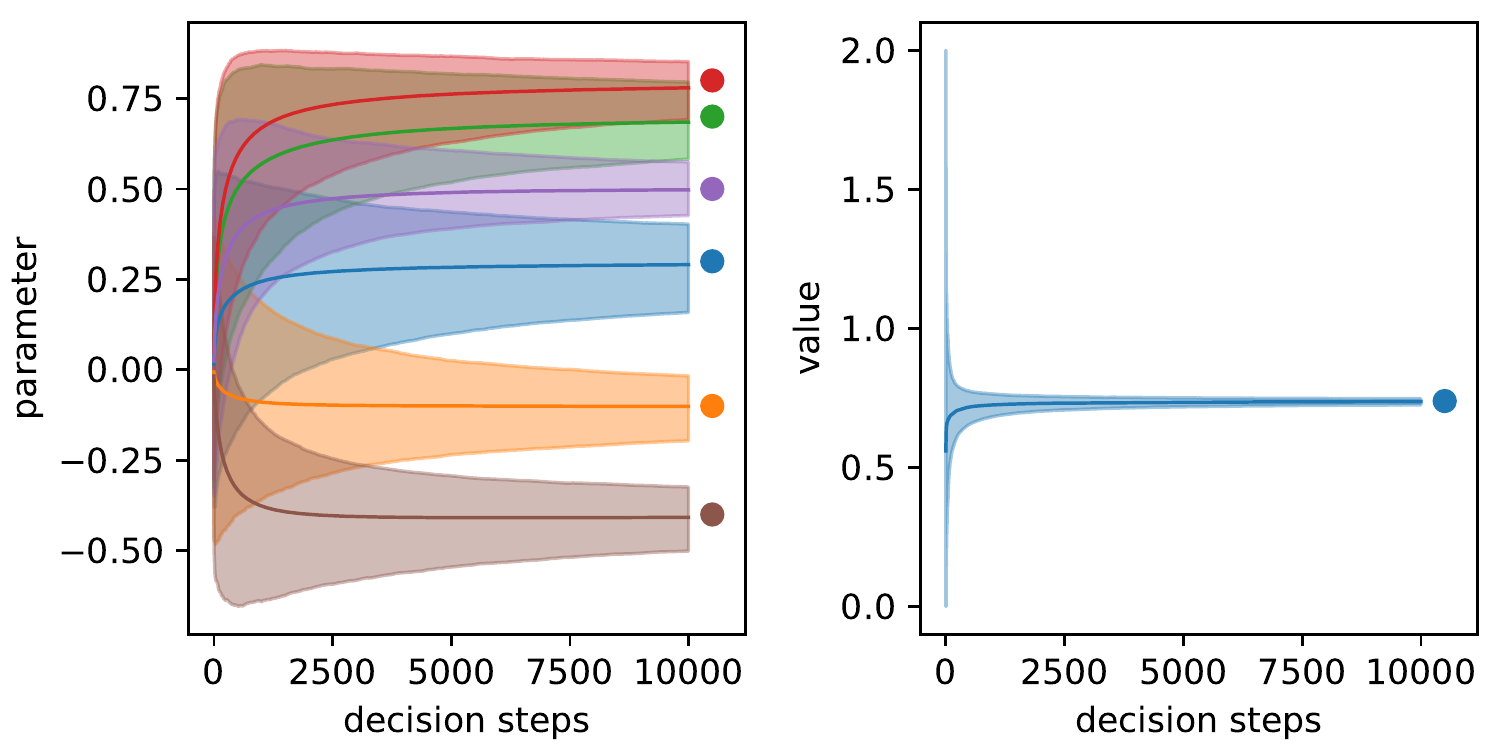}
			\caption{Logistic reward model}
		\end{subfigure}
		\caption{Parameter and optimal value estimation from 5000 repeated experiments following the conventional SGD method. The learning rate is $\alpha_t= 0.5t^{-0.501}$ and the exploration rate is $\varepsilon_t = 0.2$. The solid lines are mean estimates and the shaded regions are bounded by 2.5\% and 97.5\% percentiles of the estimates. The points at the end of the lines mark the true value.}
		\label{fig:consist_nowt}
	\end{figure}
	
	In practice, the conventional SGD method might also be used to estimate the reward model parameters. That is, $\hat{\beta}_t$ is updated using the Robbins-Monro rule in (\ref{eq:sgd}) without the IPW adjustment of the gradient. Intuitively, the conventional SGD parameter estimator will have a lower variance than the proposed estimator because our updating rule may occasionally take big steps due to the inverse probability weighting. 
	However, it is hard to derive the asymptotic distribution of the conventional SGD parameter estimator as discussed in Section \ref{sec:algo2}, and thus we can only study its performance empirically and compare it with our proposed estimator. To this end, we also conduct experiments for the conventional method under the same settings and the results for $\varepsilon_t = 0.2$ are shown in Figure \ref{fig:consist_nowt}. The results for other exploration rates are similar and are shown in the Appendix. It can be seen that both the parameter and value estimators from the conventional method converge to the truth and the quantiles of the value estimators are very similar to the results shown in Figure \ref{fig:consist}. For the linear reward model setting, the quantiles of the parameter estimators from the two methods are almost the same. For the logistic reward model setting, the parameter estimators from the conventional method have smaller standard deviations than those from our method, which agrees with our intuition, but they are still of the same scale. Therefore, our modification of the parameter estimator updating rule only has a minor effect on the efficiency of the parameter estimator.
	
    \subsection{Variance estimation}\label{sec:variance}
	The finite sample properties of the online plugin variance estimator are shown in Figure \ref{fig:exploration}. We repeat the experiment 5000 times for each reward model setting and each exploration rate type and plot the results for sample size $t = 10^3$, $10^4$ and $10^5$. For each setting, the first plot shows the ratio of the averaged estimated standard errors, which are calculated from the plugin variance estimators given in the main theorems, to the standard deviation of the 5000 estimates. The ratio should be close to one if the plugin estimate approximates the true variance well. The second plot shows the coverage probability of the 95\% Wald confidence interval and it should be around 0.95 if the distributions of the parameter and value estimators are approximately normal. The Monte Carlo standard error for a coverage probability around $0.9$ is about $0.004$ for 5000 repetitions. The third plot shows the average length of the 95\% Wald confidence interval of the parameter or value. All results for the parameter estimation are averaged across the six parameters for a clear presentation and the original results are available in Appendix \ref{sec:table}.
	
	\begin{figure}[!htbp]
	    \centering
	    \begin{subfigure}{0.75\textwidth}
	        \centering
    	    \includegraphics[width=\linewidth]{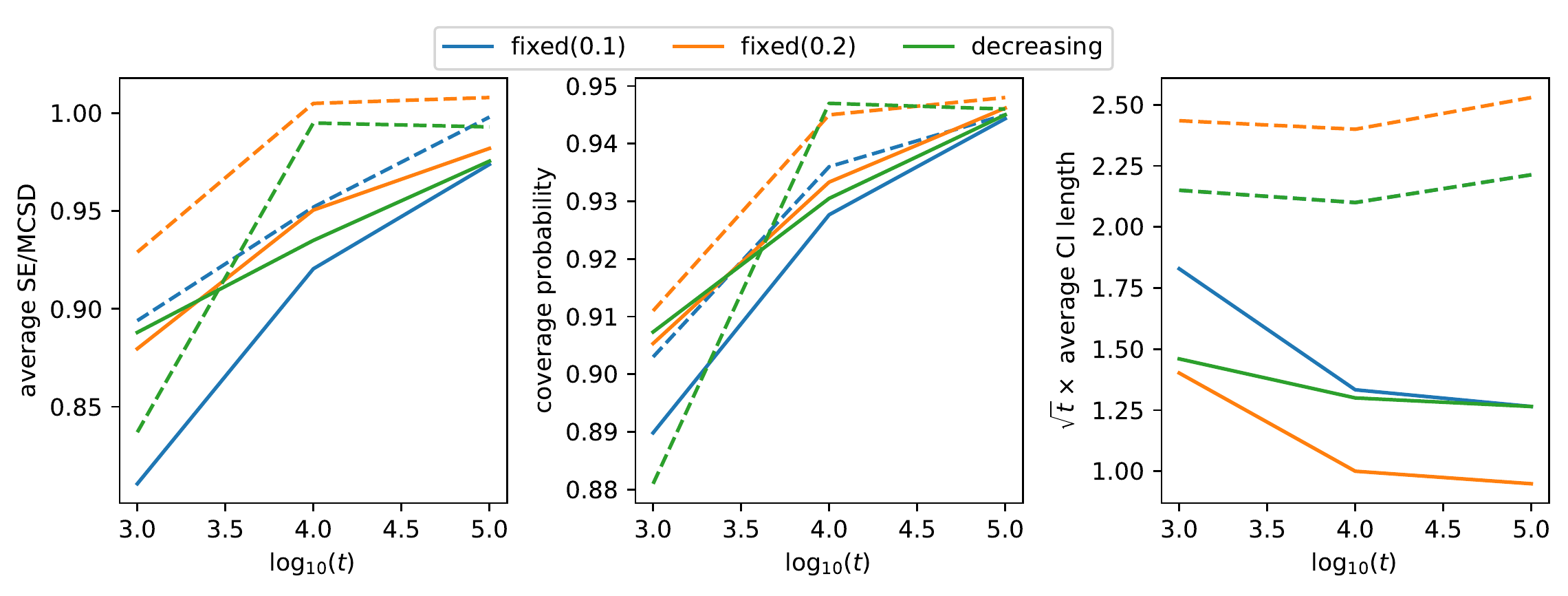}
    	    \caption{Linear reward model}
    	\end{subfigure}%
    	\\
    	\begin{subfigure}{0.75\textwidth}
	        \centering
    	    \includegraphics[width=\linewidth]{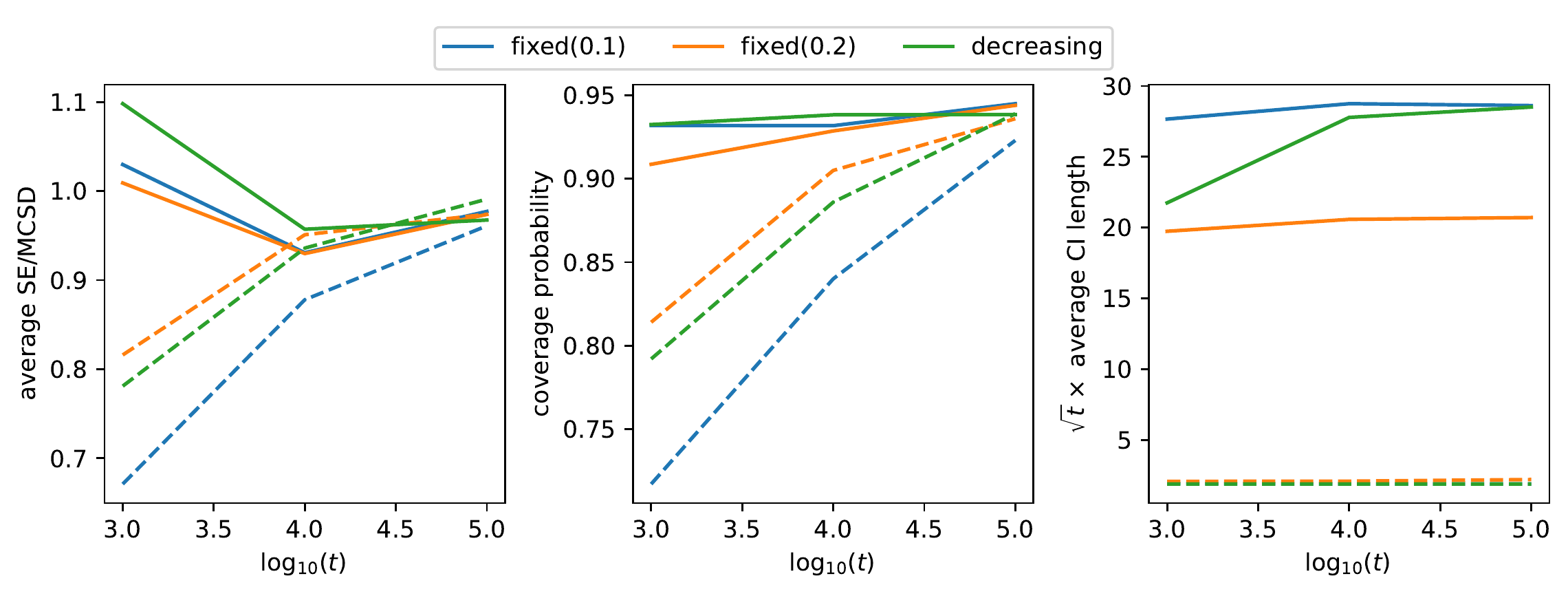}
    	    \caption{Logistic reward model}
    	\end{subfigure}
	    \caption{Online plugin variance estimation for the parameter and value estimators with different exploration rates. The learning rate is $\alpha_t = 0.5t^{-0.501}$. The solid lines are the average results of the six parameters and the dashed lines are the results for the value estimation.}
	    \label{fig:exploration}
	\end{figure}
	
	In the linear setting, the finite sample distributions of the parameter and value estimators are close to the asymptotic distributions after $10^4$ decision steps for all three types of exploration rates. The advantage of having more exploration ($\varepsilon_t = 0.2$ versus $\varepsilon_t = 0.1$) is more clear when the sample size is small because the linear model estimators converge fast. The decreasing exploration rate with limit $\varepsilon_\infty = 0.1$ can outperform the fixed rate with the same limit in terms of parameter estimation. But there is not much gain from exploring more for the value estimation since the estimated decision rule is very close to the optimal rule. In the logistic setting where estimators converge more slowly, using $\varepsilon_t = 0.2$ is always better than $\varepsilon_t = 0.1$ for the sample size we consider. The decreasing exploration rate can accelerate the parameter estimator convergence in the early stage and hence benefit the estimation of the optimal value and its variance.
		
	According to Theorem \ref{thm:paradist}, the asymptotic variance of $\bar{\beta}_t$ is roughly of the order of $(t\varepsilon_\infty)^{-1/2}$. This is supported by the plots in the third column of Figure \ref{fig:exploration} where the average length of the 95\% confidence interval for $\beta_0$ is of the order of $(t\varepsilon_t)^{-1/2}$. Given the sample size and a desired length of the confidence interval of $\beta_0$, we can apply this result to choose $\varepsilon_\infty$ in real applications. If the sample size is big enough, we should choose $\varepsilon_\infty$ as small as possible to achieve higher cumulative rewards. It can also be seen that the choice of exploration rate has little effect on the variance of the value estimator, which agrees with Theorem \ref{thm:valdist}.
	
	These simulation results validate our asymptotic properties provided in the main theorems. In practice, however, many other factors besides the learning and exploration rates can affect the actual convergence speed although the parameter and value estimators both converge at $t^{1/2}$ rate. In the linear reward setting, for example, more samples are required for the estimators to achieve the same performance if we increase the noise $\sigma^2$ or the dimension $p$ and keep other factors fixed. 
	
	We also compare the plugin variance estimation method with the batch-means method and resampling method proposed by \citet{chen2016statistical} and \citet{fang2018online} respectively. As shown in Figure \ref{fig:valueest}, both methods cannot outperform the simple online plugin estimator under our settings. The underestimation problem of the batch-means method is more serious in online decision making because it neglects the correlation brought by the data dependence structure. The resampling method tends to overestimate the variance in finite samples. Theoretical justifications for transferring these methods to online decision making are not studied and it will be an interesting future work to improve these methods for online decision making.
	
	\begin{figure}[!htbp]
		\centering
		\begin{subfigure}{0.5\textwidth}
			\centering
			\includegraphics[width=\linewidth]{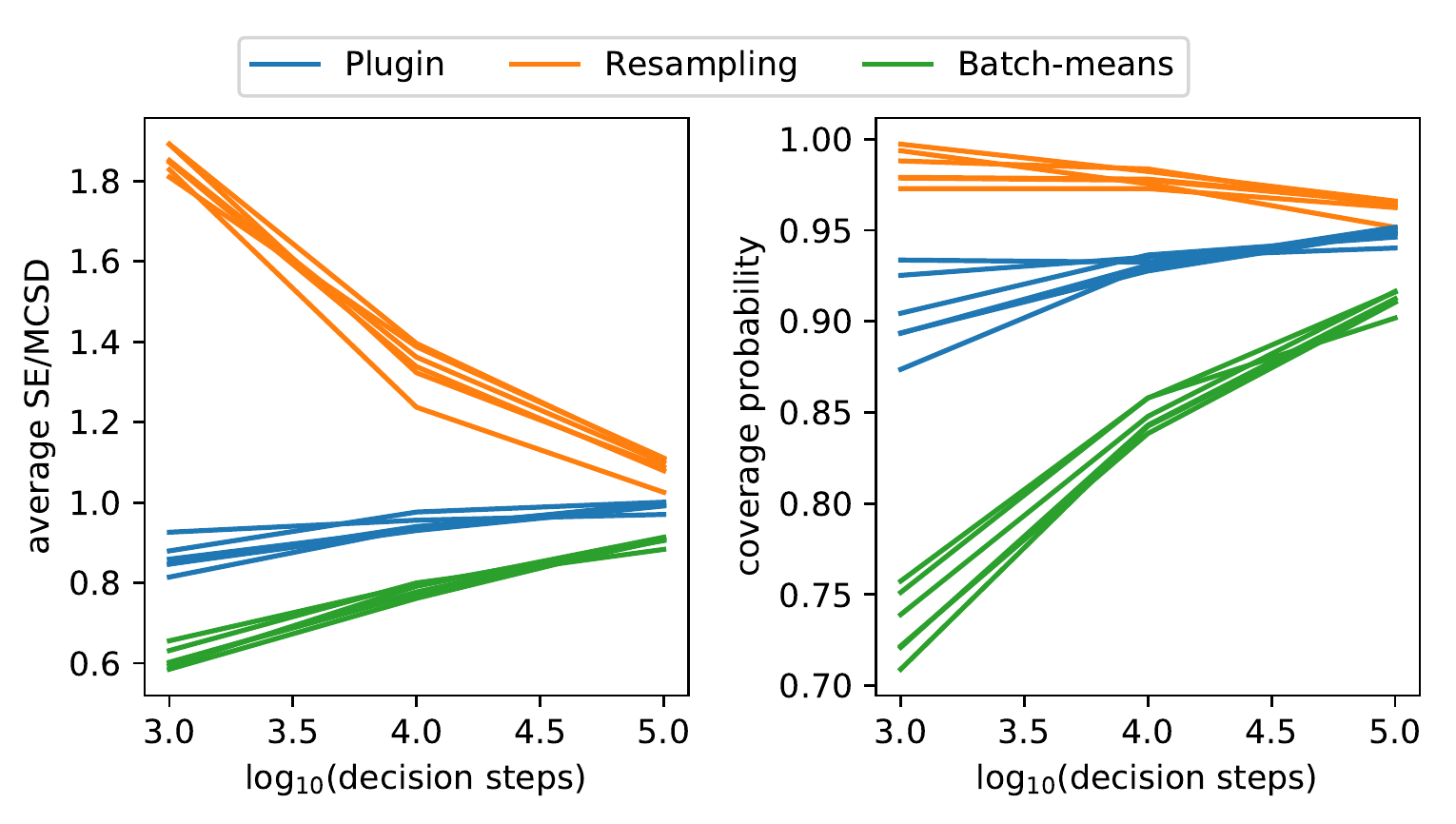}
			\caption{Linear reward model}
		\end{subfigure}%
		\begin{subfigure}{0.5\textwidth}
			\centering
			\includegraphics[width=\linewidth]{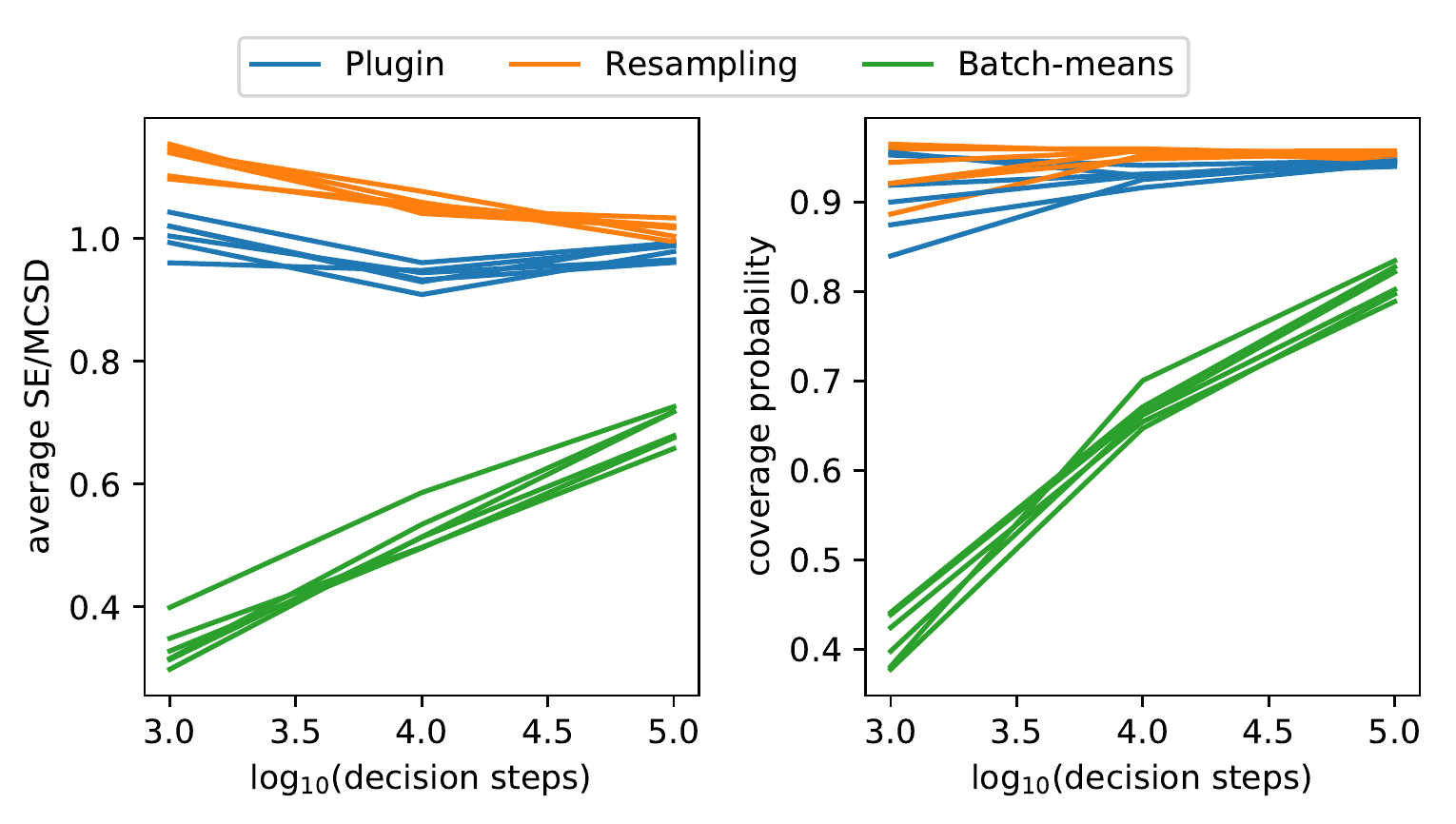}
			\caption{Logistic reward model}
		\end{subfigure}
		\caption{Comparison of variance estimation methods. The average standard error to Monte Carlo standard deviation and coverage probability are calculated from 5000 repeated experiments following the proposed SGD method with IPW gradients. The learning rate is $\alpha_t= 0.5t^{-0.501}$ and the exploration rate is $\varepsilon_t = 0.2$.}
		\label{fig:valueest}
	\end{figure}
	
	\section{Real Data Analysis}\label{sec:real}
	In this section, we apply our algorithm to the Yahoo! Today module user click log data and provide statistical inference for the reward model parameters and the optimal value. The dataset records the news article recommendation and user response from May 1st to May 10th, 2009. We choose the two most recommended article on May 1st, No.109510 and No.109520 for analysis so that the action space is binary. We code $A_t = 1$ for recommending No.109510 and $0$ for the other article. The articles were recommended randomly in the original experiment and the two together were recommended 405888 times on May 1st. The reward is coded as $Y_t = 1$ for clicking on the article link and $0$ for not clicking. The raw user features such as demographic, geographic and behavioral information are processed using a dimension reduction procedure described in \cite{li2010contextual} and the final user features have six continuous covariates between zero and one. The first five sums to one and the sixth one is always one. So we keep the second to the fifth covariates as $X_{t,2}$ to $X_{t,5}$ and set $X_{t,1} = 1$ to form our feature vectors $X_t\in\R^5$. Since the reward $Y_t$ has binary outcome, we posit the logistic reward model (\ref{eq:logistic}) for its conditional mean.
	
	In order to apply the online decision making algorithm, we have to simulate the data generating process from the offline data. Given an entry from the dataset, we will make a decision according to the currently estimated rule. If our decision matches the observed action in that entry, the reward from the same entry will be kept and used to update our decision rule. Otherwise, we will drop the entry and read the next one. The process is repeated until all entries are kept or dropped. Since the original recommendations are randomized, about a half of the entries will be matched and the selected users are still representative of the user population.
	
	\begin{table}[!htbp]
		\centering
		\begin{tabular}{cccccrc}
			\hline
			&estimate &s.e. &\multicolumn{2}{c}{Wald 95\% CI} &$t$ value & $P(>|t|)$\\
			\hline
			$\beta_{01}$  &-2.8226 &0.0810 &-2.9814 &-2.6637 &-34.83  &0.0000\\
			$\beta_{02}$  &-0.4220 &0.1498 &-0.7155 &-0.1285 & -2.82  &0.0048\\
			$\beta_{03}$  &-0.4022 &0.1275 &-0.6521 &-0.1522 & -3.15  &0.0016\\
			$\beta_{04}$  & 0.1689 &0.1282 &-0.0824 & 0.4202 &  1.32  &0.1878\\
			$\beta_{05}$  &-1.0909 &0.1111 &-1.3087 &-0.8731 & -9.82  &0.0000\\
			$\beta_{11}$  &-2.6272 &0.0354 &-2.6965 &-2.5578 &-74.24  &0.0000\\
			$\beta_{12}$  &-0.2787 &0.0572 &-0.3909 &-0.1665 & -4.87  &0.0000\\
			$\beta_{13}$  &-0.3692 &0.0498 &-0.4668 &-0.2717 & -7.42  &0.0000\\
			$\beta_{14}$  &-0.0562 &0.0883 &-0.2294 & 0.1170 & -0.64  &0.5247\\
			$\beta_{15}$  &-1.0890 &0.0567 &-1.2002 &-0.9778 &-19.20  &0.0000\\
			$V(d^{opt})$  & 0.0515 &0.0005 & 0.0505 & 0.0525 & --  & --\\              
			\hline
		\end{tabular}
		\caption{Parameter and value estimation for the Yahoo! data using fixed exploration rate.}
		\label{tab:yahoo}
	\end{table}
	
	The learning rate is the same as that from the simulation studies in Section \ref{sec:simulation} since different $\alpha$'s achieve similar loss for the big sample size. We consider a fixed exploration rate with $\varepsilon_t = 0.2$ and a decreasing rate with $\varepsilon_t = t^{-0.3}\vee 0.1$ to illustrate the results of our algorithm. For the fixed exploration rate, 203909 entries are matched and the learned model parameters and optimal value at the final step $t = 203909$ are listed in Table \ref{tab:yahoo}. Along with the parameter and value estimation, we also provide their standard errors based on the online plugin variance estimator. It can be seen that all parameters are significant except those associated with the fourth covariate. The optimal value, i.e. the expected click rate following the optimal decision rule is estimated to be 5.15\%. Its 95\% confidence interval is higher than the 4.71\% click rate under the original random assignment, meaning the optimal decision rule is significantly better than the random rule. In practice, our algorithm achieves a 5.08\% click rate over all the matched entries, which is a bit lower than the estimated optimal value due to the strictly positive exploration rate. Similar estimates can be obtained using the decreasing exploration rate and the results are shown in Table \ref{tab:yahoo_decrease}. Here $203445$ entries are matched and the estimated optimal value is 5.09\%. The cumulative click rate from our algorithm is 5.06\% and it is closer to the estimated optimal value than in the $\varepsilon_t = 0.2$ case since asymptotically fewer explorations are made with the decreasing exploration rate.
	
	\begin{table}[!htbp]
		\centering
		\begin{tabular}{cccccrc}
			\hline
			&estimate &s.e. &\multicolumn{2}{c}{Wald 95\% CI} &$t$ value & $P(>|t|)$\\
			\hline
			$\beta_{01}$  &-2.7967  &0.1066 &-3.0057 &-2.5878 &-26.23 &0.0000\\
			$\beta_{02}$  &-0.3927  &0.2056 &-0.7957 & 0.0103 & -1.91 &0.0562\\
			$\beta_{03}$  &-0.4327  &0.1696 &-0.7650 &-0.1003 & -2.55 &0.0107\\
			$\beta_{04}$  & 0.1488  &0.1221 &-0.0904 & 0.3880 &  1.22 &0.2227\\
			$\beta_{05}$  &-1.0987  &0.1503 &-1.3933 &-0.8041 & -7.31 &0.0000\\
			$\beta_{11}$  &-2.5926  &0.0606 &-2.7115 &-2.4738 &-42.77 &0.0000\\
			$\beta_{12}$  &-0.2181  &0.0812 &-0.3773 &-0.0589 & -2.68 &0.0073\\
			$\beta_{13}$  &-0.3879  &0.0837 &-0.5520 &-0.2239 & -4.64 &0.0000\\
			$\beta_{14}$  &-0.2176  &0.1823 &-0.5748 & 0.1397 & -1.19 &0.2326\\
			$\beta_{15}$  &-1.0555  &0.1013 &-1.2541 &-0.8569 &-10.42 &0.0000\\
			$V(d^{opt})$  & 0.0509  &0.0005 & 0.0499 & 0.0519 & -- & --\\           
			\hline
		\end{tabular}
		\caption{Parameter and value estimation for the Yahoo! data using decreasing exploration rate.}
		\label{tab:yahoo_decrease}
	\end{table}
	
	\section{Discussions}\label{sec:discuss}
	In this paper, we utilize SGD to provide a completely online algorithm for decision making and value estimation. We also make statistical inference such as interval estimation and hypothesis testing possible for the parameter and value estimators from our algorithm by constructing consistent online estimators for their variance. The algorithm scales well for large streaming data and only requires the storage of $\mathcal{O}(p^2)$ data. Here we discuss several potential extensions that could possibly improve the performance of our method or adapt our algorithm to solve more general problems.
	
	\begin{enumerate}
	    \item \textit{Different exploration strategies.} Other contextual bandit solutions such as the upper confidence bound method \citep{auer2002using} can also update its rule online if the reward model is assumed to be linear. But its nonlinear extension \citep{valko2013finite} would still require the storage of the historical data and thus does not scale for big data. Besides that, the statistical inference of these methods remains unexplored. We think it is worthwhile to develop online updating algorithms for these methods and study their inferential properties as they provide a more elegant way of exploring and can maximize the cumulative rewards asymptotically.
	    \item \textit{Semi- and non-parametric reward models.} It is assumed that the Q-function has a parametric form but chances are that the parametric model is misspecified in real applications. Semi-parametric methods try to alleviate this problem by assuming a parametric form for the bleep function $\E(Y|1, X) - \E(Y|0, X)$ only and using a non-parametric model for $\E(Y|0, X)$. Furthermore, non-parametric methods such as those studied by \cite{yang2002randomized} and \cite{qian2016kernel} can avoid the model misspecification problem since they impose no restriction on the form of the Q-function at all. However, both semi- and non-parametric models require the storage of all the historical data and how to adapt these methods for the online decision making problem with big streaming data is still an open question.
	    \item \textit{Other variance estimation methods.} It has been shown in Section \ref{sec:variance} that both the batch-means and resampling estimators cannot outperform the simple online plugin estimator, but we should note that both of them are originally proposed for the i.i.d. setting and it may be possible to modify them so that they could perform better under the online decision making setting. The resampling method is worth paying more attention to since it can handle the situation where the loss function is not twice-differentiable.
	    \item \textit{Augmented IPW value estimator.} We estimate the value of the optimal decision rule $V(d^{opt})$ using the IPW estimator. It is also possible to incorporate the estimated Q-function and construct an augmented IPW (AIPW) estimator \citep{zhang2012robust} for $V(d^{opt})$, that is,
        $$
            \tilde{V}_t(d^{opt}) = \frac{1}{t}\sum_{s = 1}^t \left\{\frac{C_sY_s}{\pi_{C,s}} - \frac{C_s - \pi_{C,s}}{\pi_{C,s}} \mu(\hat{d}^{opt}_s(X_s), X_s; \bar{\beta}_s)\right\}.
        $$
        The AIPW estimator is also consistent and may be more efficient than the IPW estimator. Its asymptotic properties can be derived following similar strategies we use to show the asymptotic normality of $\hat{V}_t(d^{opt})$ and would be interesting to study in a  future work.
	    \item \textit{Adaptation to lagged responses.} Our algorithm requires that $Y_t$ is observed before taking the next action $A_{t+1}$. In many real applications, however, it is possible that $Y_t$ is observed after decision step $t+1$. When we have these lagged responses, the reward model cannot be updated until a response is collected, and we have to take actions according to the most recent available model. Suppose at step $t-1$, $Y_{t-s}, \cdots, Y_{t-1}$ has not been observed for some $s \ge 1$. Then, we should first update the model with the newly observed responses $Y_{t-s}, \cdots$ at step $t$ and take action afterwards. Algorithm \ref{alg:2} presents such a modified version of Algorithm \ref{alg:1} to adjust for lagged responses.

        \begin{algorithm}[!htbp]
	        \DontPrintSemicolon
	        \SetKwInOut{Input}{Initialize}
	        \caption{Online decision Making with lagged responses}\label{alg:2}
	
	        \KwIn{$\hat{\beta}_0=\bar{\beta}_0=0$, $\pi_0=1/2$, $\alpha_t$,$\varepsilon_t$}
	        \BlankLine
	        $s = 0$\;
	        \For{$t = 1$ \KwTo $T$}{
		        Observe and store $X_t$\;
		        \While{$Y_{t-s}$ is observed}{
		            Form $O_{t-s} = (X_{t-s}, A_{t-s}, Y_{t-s})$\;
		            Calculate the IPW gradient 
		            $$
		            g(\hat{\beta}_{t-s-1}; O_{t-s}) = \frac{\nabla\ell(\hat{\beta}_{t-s-1}; O_{t-s})I\{A_t = 1\}}{2\pi_{t-s-1}(X_t)} + \frac{\nabla\ell(\hat{\beta}_{t-s-1}; O_{t-s})I\{A_t = 0\}}{2\{1 - \pi_{t-s-1}(X_t)\}}
		            $$\;
		            Update $\hat{\beta}_{t-s} = \hat{\beta}_{t-s-1} - \alpha_{t-s} g(\hat{\beta}_{t-s-1}; O_{t-s})$\;
		            Update $\bar{\beta}_{t-s} =  \{\hat{\beta}_{t-s} + (t-s-1)\bar{\beta}_{t-s-1}\}/(t-s)$\;
		            $\pi_{t-s}(X) = \pi_{t-s-1}(X)$\;
		            $s = s - 1$\;
		        }
		        $s = s + 1$\;
		        Update $\pi_{t-s}(X) = (1-\varepsilon_{t-s})I\{\mu(1, X, \bar{\beta}_{t-s}) > \mu(0, X, \bar{\beta}_{t-s})\} + \varepsilon_{t-s}/2$\;
		        Sample $A_t$ from Bernoulli$(\pi_{t-s}(X_t))$ and store it\;
	        }
        \end{algorithm}

        Algorithm \ref{alg:2} is very flexible as we allow observing $Y_t$ at any time after $t$. The parameter estimator $\bar{\beta}_t$ will still converge to the true parameter $\beta_0$ but its convergence rate will depend on the mechanism of the response lag. In the ideal cases where $s$ does not grow with $t$, or $s = \mathcal{O}(1)$, we would expect that $\bar{\beta}_t$ still converges at the $\sqrt{t}$ rate. Moreover, when $s$ is a constant for $t \ge s$, similar theoretical results can be derived following our proof. The setting we studied is a special case where $s=1$.
	\end{enumerate}

	\newpage
	
	\appendix
	\renewcommand{\thefigure}{A\arabic{figure}}  
    \renewcommand{\thetable}{A\arabic{table}}
    \setcounter{figure}{0}
    \setcounter{table}{0}
    
	\section{Proof of Main Results}
	\cite{polyak1992acceleration} considered the following stochastic approximation problem. Let $R(\beta) : \R^{2p} \to \R^{2p}$ be some unknown function and $\beta^*$ be the unique solution of $R(\beta) = 0$. At any point $\hat{\beta}_{t-1}$, an approximation of $R(\hat{\beta}_{t-1})$ can be observed as $\hat{R}(\hat{\beta}_{t-1}) = R(\hat{\beta}_{t-1}) + \xi_t$, where $\xi_t$ is some random noise. With an initial estimation $\hat{\beta}_0$ and learning rates $\alpha_t$, the stochastic approximation with averaging algorithm for finding $\beta^*$ is 
	$$
	\hat{\beta}_t = \hat{\beta}_{t-1} + \alpha_t \hat{R}(\hat{\beta}_{t-1}); \quad \bar{\beta}_t = \frac{1}{t}\sum_{s=1}^t \hat{\beta}_s.
	$$
	If the learning rate is taken as $\alpha_t = \alpha t^{-\gamma}$ with $1/2 < \gamma < 1$, the consistency and asymptotic normality of $\bar{\beta}_t$ can be shown under the following assumptions.
	\begin{assumpB}\label{as:p1}
		There exists a function $V(\beta): \R^{2p} \to \R$ such that for all $\beta, \beta'$, and some $\lambda>0, L_0>0, l_0>0, \delta>0$ the conditions $V(0) = 0$, $\nabla V(0) = 0$, $V(\beta) \ge \lambda\lVert \beta \rVert^2$, $\lVert V(\beta) - V(\beta') \rVert^2 \le L_0\lVert \beta - \beta' \rVert^2$ hold true. Moreover, $\nabla V(\beta - \beta^*)^T R(\beta) > 0$ for all $\beta \ne \beta^*$, and $\nabla V(\beta - \beta^*)^T R(\beta) > l_0 V(\beta - \beta^*)$ for all $\lVert\beta - \beta^*\rVert < \delta$.
	\end{assumpB}
	
	\begin{assumpB}\label{as:p2}
		There exists a positive definite matrix $H \in \R^{2p\times 2p}$ and some $K_1>0, \delta>0$ such that 
		$$
		\lVert R(\beta) - H(\beta - \beta^*) \rVert^2 \le K_1 \lVert \beta - \beta^* \rVert^2
		$$
		for all $\lVert\beta - \beta^*\rVert < \delta$.
	\end{assumpB}
	
	\begin{assumpB}\label{as:p3}
		$\{\xi_t\}_{t\ge 1}$ is a martingale difference process wrt $\{\F_t\}_{t\ge 1}$, i.e., $\E(\xi_t | \F_{t-1}) = 0$. Moreover, there exist some $K_2>0$ such that 
		$$
		\E(\lVert \xi_t \rVert^2 | \F_{t-1}) + \lVert R(\hat{\beta}_{t-1})\rVert^2 \le K_2(1 + \lVert \hat{\beta}_{t-1} - \beta^*\rVert^2)\; a.s.
		$$
		for all $t \ge 1$.
	\end{assumpB}
	
	\begin{assumpB}\label{as:p4}
		Decompose $\xi_t$ as $\xi_t^* + \zeta_t(\hat{\beta}_{t-1})$, where 
		\begin{gather*}
			\E(\xi_t^* | \F_{t-1}) = 0\; a.s.,\\
			\E(\xi_t^*\xi_t^{*T} | \F_{t-1}) \overset{p}{\to} S\; \text{as } t\to\infty; S \succ 0,\\
			\sup_{t\ge 1}\E\{\lVert \xi_t^*\rVert^2 I(\lVert \xi_t^*\rVert > C)| \F_{t-1}\} \overset{p}{\to} 0\; \text{as } C\to\infty,
		\end{gather*}
		and there exists some $K_3 >0$ such that for all $t$ large enough,
		$$
		\E\{\lVert \zeta_t(\hat{\beta}_{t-1})\rVert^2 | \F_{t-1}\} \le K_3 \lVert \hat{\beta}_{t-1} - \beta^* \rVert^2.
		$$
	\end{assumpB}
	
	\begin{proposition}\label{prop:pj}
		If Assumptions \ref{as:p1} to \ref{as:p3} are satisfied, then $\bar{\beta}_t \to \beta^*$ almost surely. Furthermore, if Assumption \ref{as:p4} is also satisfied, 
		$$
		\sqrt{t}(\bar{\beta}_t - \beta^*) \overset{d}{\to} \mathcal{N}(0, H^{-1}SH^{-1}).
		$$
	\end{proposition}
	
	\subsection{Proof of Theorem \ref{thm:paradist}}\label{sec:proofparadist}
	\begin{proof}
		Recall $R(\beta) = \nabla L(\beta)$ and $H = \nabla^2 L(\beta^*)$. Let $V(\Delta) = L(\beta^* + \Delta) - L(\beta^*) + \lambda\lVert\Delta\rVert^2$ for some $\lambda > 0$, Assumptions \ref{as:p1} and \ref{as:p2} are verified under Assumption \ref{as:L} by \cite{fang2018online}. 
		Decompose the random noise as $\xi_t = \xi_t^* + \zeta_t(\hat{\beta}_{t-1})$ where $\xi_t^* = -g(\beta^*; O_t)$ and $\zeta_t(\hat{\beta}_{t-1}) = R(\hat{\beta}_{t-1})-\{g(\hat{\beta}_{t-1}; O_t) - g(\beta^*; O_t)\}$. Then 
		$$
		\E_{\p_{O}^\pi}(\xi_t^* | \F_{t-1}) = -\nabla L(\beta^*) = 0.
		$$
		Denote $\Sigma_i(\beta; X) = \E[\nabla\ell(\beta; X, i, Y)\{\nabla\ell(\beta; X, i, Y)\}^T| X]$ for $i=0,1$. Then
		\begin{equation}\label{eq:Exixi}
			\begin{split}
				&\E_{\p_{O}^\pi}(\xi_t^*\xi_t^{*T} | \F_{t-1})\\
				=&\E_{\p_{O}^\pi}[g(\beta^*; O_t)\{g(\beta^*; O_t)\}^{T} | \F_{t-1}]\\
				=&\E_{\p_{O}^\pi}\bigg[\frac{\nabla\ell(\beta^*; O_t)\{\nabla\ell(\beta^*; O_t)\}^T I(A_t=1)}{4\{\hat{\pi}_{t-1}(X_t)\}^2}\\ 
				&+ \frac{\nabla\ell(\beta^*; O_t)\{\nabla\ell(\beta^*; O_t)\}^T I(A_t=0)}{4\{1 - \hat{\pi}_{t-1}(X_t)\}^2} \bigg| \F_{t-1}\bigg]\\
				=&\E\bigg(\frac{\E[\nabla\ell(\beta^*; X_t, 1, Y_t)\{\nabla\ell(\beta^*; X_t, 1, Y_t)\}^T |\F_{t-1}, X_t]}{4\hat{\pi}_{t-1}(X_t)} \\
				&+ \frac{\E[\nabla\ell(\beta^*; X_t, 0, Y_t)\{\nabla\ell(\beta^*; X_t, 0, Y_t)\}^T |\F_{t-1}, X_t]}{4\{1 - \hat{\pi}_{t-1}(X_t)\}} \bigg| \F_{t-1}\bigg)\\
				=&\int\left[ \frac{\Sigma_1(\beta^*; X)}{4\hat{\pi}_{t-1}(X)} + \frac{\Sigma_0(\beta^*; X)}{4\{1 - \hat{\pi}_{t-1}(X)\}}\right]d\p_X.
			\end{split}
		\end{equation}
		Similarly, 
		\begin{equation}\label{eq:ineq1}
			\begin{split}
				\E_{\p_{O}^\pi}(\lVert\xi_t^*\rVert^2 | \F_{t-1}) &= \mathrm{tr}\{\E_{\p_{O}^\pi}(\xi_t^*\xi_t^{*T} | \F_{t-1})\}\\ &\le \frac{1}{2\varepsilon_\infty}\mathrm{tr}\bigg[\int \{ \Sigma_1(\beta^*; X) + \Sigma_0(\beta^*; X)\}d\p_X\bigg],
			\end{split}
		\end{equation}
		and hence by Dominated Convergence Theorem,
		$$
		\sup_t \E_{\p_{O}^\pi}\{\lVert\xi_t^*\rVert^2 I(\lVert\xi_t^*\rVert > K) | \F_{t-1}\} \overset{p}{\to} 0 \text{ as } K\to\infty.
		$$
		Note that 
		\begin{equation}\label{eq:ineq2}
			\lVert R(\hat{\beta}_{t-1})\rVert^2 = \lVert R(\hat{\beta}_{t-1}) - R(\beta^*)\rVert^2 \leq L_1\lVert\hat{\beta}_{t-1} - \beta^*\rVert^2
		\end{equation} 
		by the Lipschitz continuity assumption in \ref{as:L}, and
		\begin{align*}
			&\E_{\p_{O}^\pi}\{\lVert g(\hat{\beta}_{t-1}; O_t) - g(\beta^*; O_t)\rVert^2 | \F_{t-1}\}\\
			=&\E\bigg[\frac{\E\{\lVert\nabla\ell(\hat{\beta}_{t-1}; X_t, 1, Y_t) - \nabla\ell(\beta^*; X_t, 1, Y_t)\rVert^2 |\F_{t-1}, X_t\}}{4\hat{\pi}_{t-1}(X_t)} \\
			&+ \frac{\E\{\lVert\nabla\ell(\hat{\beta}_{t-1}; X_t, 0, Y_t) - \nabla\ell(\beta^*; X_t, 0, Y_t)\rVert^2 |\F_{t-1}, X_t\}}{4\{1 - \hat{\pi}_{t-1}(X_t)\}} \bigg| \F_{t-1}\bigg]\\
			\le& \frac{1}{\varepsilon_\infty}\E_{\p_{O}^r}\{\lVert \nabla\ell(\hat{\beta}_{t-1}; O_t) - \nabla\ell(\beta^*; O_t)\rVert^2\}\\
			\le& \frac{L_3}{\varepsilon_\infty}\lVert\hat{\beta}_{t-1} - \beta^*\rVert^2
		\end{align*}
		by \ref{as:diff}. It follows that 
		\begin{equation}\label{eq:ineq3}
			\begin{split}
				&\E_{\p_{O}^\pi}\{\lVert \zeta_t(\hat{\beta}_{t-1})\rVert^2 | \F_{t-1}\}\\ 
				\le &2\E_{\p_{O}^\pi}\{\lVert g(\hat{\beta}_{t-1}; O_t) - g(\beta^*; O_t)\rVert^2 | \F_{t-1}\} + 2\lVert R(\hat{\beta}_{t-1})\rVert^2 \\
				\le &K_3\lVert\hat{\beta}_{t-1} - \beta^*\rVert^2.
			\end{split}
		\end{equation}
		Combine inequalities from (\ref{eq:ineq1}), (\ref{eq:ineq2}) and (\ref{eq:ineq3}) we have
		\begin{align*}
			&\E_{\p_{O}^\pi}\{\lVert \xi_t(\hat{\beta}_{t-1})\rVert^2 | \F_{t-1}\} + \lVert R(\hat{\beta}_{t-1})\rVert^2\\
			\le& 2\E_{\p_{O}^\pi}(\lVert \xi_t^*\rVert^2 | \F_{t-1}) + 2\E_{\p_{O}^\pi}\{\lVert \zeta_t(\hat{\beta}_{t-1}) \rVert^2 | \F_{t-1}\} + \lVert R(\hat{\beta}_{t-1})\rVert^2\\
			\le& K_2(1+\lVert\hat{\beta}_{t-1} - \beta^*\rVert^2)
		\end{align*}
		for some $K_2 > 0$. Thus Assumption \ref{as:p3} is satisfied and it follows from Proposition \ref{prop:pj} that $\bar{\beta}_t \to \beta^*$ almost surely. Apply Continuous Mapping Theorem, we have
		\begin{align*}
			\E_{\p_{O}^\pi}(\xi_t^*\xi_t^{*T} | \F_{t-1})=&\int\left[ \frac{\Sigma_1(\beta^*; X)}{4\hat{\pi}_{t-1}(X)} + \frac{\Sigma_0(\beta^*; X)}{4\{1 - \hat{\pi}_{t-1}(X)\}}\right]d\p_X\\
			\overset{p}{\to}&\frac{1}{4}\int\left\{ \frac{\Sigma_1(\beta^*; X)}{\pi^*(X)} + \frac{\Sigma_0(\beta^*; X)}{1 - \pi^*(X)}\right\}d\p_X =: S
		\end{align*}
		where $\pi^*(X) = (1 - \varepsilon_\infty) I\{\mu(1, X, \beta^*) > \mu(0, X, \beta^*)\} + \varepsilon_\infty/2$. Therefore Assumption \ref{as:p4} is verified and the asymptotic normality of $\bar{\beta}_t$ is established by Proposition \ref{prop:pj}. 
	\end{proof}
	
	\subsection{Proof of Theorem \ref{thm:paravar}}
	We first introduce the following lemma which will be frequently used in the proofs come after.
	\begin{lemma}\label{lem:trivial}
		Let $X_1, \cdots X_n$ be a sequence of bounded random variables. If $X_n\overset{p}{\to}a$, then $n^{-1}\sum_{i=1}^n X_i\overset{p}{\to}a$. Furthermore, if non-random sequences $\{b_n\}$ and $\{c_n\}$ are bounded and $b_n\to b$, $c_n\to c$, then $n^{-1}\sum_{i=1}^n (b_iX_i +c_i)\overset{p}{\to}ab+c$.
	\end{lemma}
	\begin{proof}
		Let $R_n = X_n - a$ so $R_n\overset{p}{\to} 0$. By Markov's Inequality, for any $\epsilon > 0$,
		$$P\left\{\left|\frac{1}{n}\sum_{i=1}^n R_i \right| > \epsilon\right\} \le \frac{1}{\epsilon}\E \left|\frac{1}{n}\sum_{i=1}^n R_i\right| = \frac{1}{\epsilon n}\sum_{i=1}^n \E |R_i|,$$
		which converges to zero by the convergence of Cesaro means since $\E |R_n| \to 0$. Therefore, $n^{-1}\sum_{i=1}^n X_i\overset{p}{\to}a$. Furthermore, $b_nX_n +c_n\overset{p}{\to}ab+c$ by the convergence of Cesaro means and Continuous Mapping Theorem. It follows that $n^{-1}\sum_{i=1}^n (b_iX_i +c_i)\overset{p}{\to}ab+c$.
	\end{proof}
	With Lemma \ref{lem:trivial} in place, we are ready to prove Theorem \ref{thm:paravar}.
	\begin{proof}
		Denote $W_s = I(A_s = 1)\{2\hat{\pi}_{s-1}(X_s)\}^{-1} + I(A_s = 0)\{2-2\hat{\pi}_{s-1}(X_s)\}^{-1}$. We can apply Theorem 2.19 from \citet{hall1980martingale} and get
		\begin{gather}
			\hat{H}_t - \frac{1}{t} \sum_{s=1}^t \E_{\p_O^\pi}\{\nabla^2\ell(\bar{\beta}_{s-1}; O_s)W_s | \F_{s-1}\}\overset{p}{\to} 0,\label{eq:Shat11}\\
			\hat{S}_t - \frac{1}{t} \sum_{s=1}^t \E_{\p_O^\pi}\left[g(\bar{\beta}_{s-1}; O_s)\{g(\bar{\beta}_{s-1}; O_s)\}^T | \F_{s-1}\right]\overset{p}{\to} 0,\label{eq:Shat12}
		\end{gather}
		but we need to verify that the elements of the summands of $\hat{S}_t$ and $\hat{H}_t$ have bounding random variables. Since the summands are symmetric, it suffices to show that there exist random variables $R_H$ and $R_S$ and constants $C_H, C_S > 0$ such that $\E |R_H| <\infty$, $\E |R_S| <\infty$ and for any $v\in \R^{2p}$, $\varkappa>0$,
		\begin{gather}
			P_{\p_O^\pi}\{W_s v^T\nabla^2\ell(\bar{\beta}_{s-1}; O_s)v > \varkappa\} \le C_H P(|R_H|>\varkappa),\label{eq:CH}\\
			P_{\p_O^\pi}[v^T g(\bar{\beta}_{s-1}; O_s)\{g(\bar{\beta}_{s-1}; O_s)\}^T v > \varkappa] \le C_S P(|R_S|>\varkappa),\label{eq:CS}
		\end{gather}
		Note that $0 < W_s\le \varepsilon_\infty^{-1}$, if we take $R_H = \varepsilon_\infty^{-1}v^T\nabla^2\ell(\beta^*; O)v$ and $R_S = \varepsilon_\infty^{-2} v^T\nabla\ell(\beta^*; O) \{\nabla\ell(\beta^*; O)\}^T v$, then $\E |R_H| <\infty$ and $\E |R_S| <\infty$ under Assumption \ref{as:L}. Conditional on the feature and historical data we have
		\begin{align*}
			&P\{W_s v^T\nabla^2\ell(\bar{\beta}_{s-1}; O_s)v > \varkappa | X_s = x, \bar{O}_{s-1} = \bar{o}_{s-1}\}\\
			=&\sum_{a\in\mathcal{A}} P\{W_s v^T\nabla^2\ell(\bar{\beta}_{s-1}; O_s)v > \varkappa | A_s = a, X_s = x, \F_{s-1}\}P(A_s = a |X_s = a, \F_{s-1})\\
			\le& \sum_{a\in\mathcal{A}} 2C_H P(R_H>\varkappa | A = a, X = x)\cdot P(A = a| X = x)\\
			=& 2C_HP(R_H>\varkappa | X = x),
		\end{align*}
		where the inequality follows from Assumption \ref{as:var} and the fact that $P(A = a| X = x) = 1/2$. Then integrating out $x$ and $\bar{o}_{s-1}$ gives (\ref{eq:CH}). Similar arguments with $g(\bar{\beta}_{s-1}; O_s)\{g(\bar{\beta}_{s-1}; O_s)\}^T$ in place of $W_s \nabla^2\ell(\bar{\beta}_{s-1}; O_s)$ gives (\ref{eq:CS}). Therefore (\ref{eq:Shat11}) and (\ref{eq:Shat12}) hold.
		
		Note that $\E_{\p_O^\pi}\{\nabla^2\ell(\bar{\beta}_{s-1}; O_s)W_s| \F_{s-1}\}\overset{p}{\to} H$, and by the same argument as in (\ref{eq:Exixi}),
		\begin{equation*}
			\begin{split}
				&\E\left[g(\bar{\beta}_{s-1}; O_s)\{g(\bar{\beta}_{s-1}; O_s)\}^T | \F_{s-1}\right]\\
				=&\int\left[ \frac{\Sigma_1(\bar{\beta}_{s-1}; X)}{4\hat{\pi}_{t-1}(X)} + \frac{\Sigma_0(\bar{\beta}_{s-1}; X)}{4\{1 - \hat{\pi}_{t-1}(X)\}}\right]d\p_X\\
				\overset{p}{\to}&\frac{1}{4}\int\left\{ \frac{\Sigma_1(\beta^*; X)}{\pi^*(X)} + \frac{\Sigma_0(\beta^*; X)}{1 - \pi^*(X)}\right\}d\p_X = S.
			\end{split}
		\end{equation*}
		It follows from Lemma \ref{lem:trivial} that
		\begin{gather}
			\frac{1}{t} \sum_{s=1}^t \E_{\p_O^\pi}\{\nabla^2\ell(\bar{\beta}_{s-1}; O_s)W_s\}\overset{p}{\to} H,\label{eq:Shat21}\\
			\frac{1}{t} \sum_{s=1}^t \E\left[g(\bar{\beta}_{s-1}; O_s)\{g(\bar{\beta}_{s-1}; O_s)\}^T | \F_{s-1}\right]\overset{p}{\to} S.\label{eq:Shat22}
		\end{gather}
		Add (\ref{eq:Shat21}) to (\ref{eq:Shat11}) and (\ref{eq:Shat22}) to (\ref{eq:Shat12}) gives $\hat{H}_t \overset{p}{\to} H$ and $\hat{S}_t \overset{p}{\to} S$. Therefore the plugin variance estimator is consistent.
	\end{proof}
	
	\subsection{Proof of Theorems \ref{thm:valdist} and \ref{thm:valvar}}\label{sec:proofvaldist}
	\begin{proof}
		Denote $\hat{V}_t(d^{opt})$ and $V(d^{opt})$ as $\hat{V}_t$ and $V$ for short. To apply the results for martingales, we partition $\sqrt{t}(\hat{V}_t-V)$ into two parts:
		\begin{align}
			\sqrt{t}(\hat{V}_t - V) = & \frac{1}{\sqrt{t}}\sum_{s=1}^t \left\{\frac{C_s Y_s}{\pi_{C, s}} - \E\left(\frac{C_s Y_s}{\pi_{C, s}}\bigg|\mathcal{F}_{s-1}\right)\right\}\label{eq:part1}\\
			& + \frac{1}{\sqrt{t}}\sum_{s=1}^t \left\{ \E\left(\frac{C_s Y_s}{\pi_{C, s}}\bigg|\mathcal{F}_{s-1}\right)- V\right\}.\label{eq:part2}
		\end{align}
		Use the law of iterated expectations and the fact that $C_s = I(C_s = 1)$, we have
		\begin{equation}\label{eq:expectation}
			\begin{split}
				\E\left(\frac{C_s Y_s}{\pi_{C, s}}\bigg|\mathcal{F}_{s-1}\right)
				&=\E\left\{\E\left(\frac{C_s Y_s}{\pi_{C, s}}\bigg|\mathcal{F}_{s-1}, X_s\right)\bigg|\mathcal{F}_{s-1}\right\}\\
				&=\E\{\E(Y_s | \mathcal{F}_{s-1}, X_s, C_s = 1) | \mathcal{F}_{s-1}\}\\
				&=\E[\hat{d}^{opt}_{s-1}(X_s)\mu(1, X_s; \beta_0) + \{1-\hat{d}^{opt}_{s-1}(X_s)\}\mu(0, X_s; \beta_0) | \mathcal{F}_{s-1}]\\
				&=\int [\hat{d}^{opt}_{s-1}(X)\mu(1, X; \beta_0) + \{1-\hat{d}^{opt}_{s-1}(X)\}\mu(0, X; \beta_0)]d\p_X\\
				&=V(\hat{d}^{opt}_{s-1}).
			\end{split}   
		\end{equation}
		
		Recall the value of the optimal decision rule is
		\begin{equation*}
			\E[\E\{Y|d^{opt}(X), X\}] = \int [d^{opt}(X)\mu(1, X; \beta_0) + \{1-d^{opt}(X)\}\mu(0, X; \beta_0)]d\p_X.
		\end{equation*}
		So (\ref{eq:part2}) can be rearranged as
		\begin{align}\label{eq:part2'}
			&\frac{1}{\sqrt{t}}\sum_{s=1}^t\left[\int\{\hat{d}^{opt}_{s-1}(X) - d^{opt}(X)\}\{\mu(1, X; \beta_0) - \mu(0, X; \beta_0)\} d\p_X\right]\nonumber\\
			=&\frac{1}{\sqrt{t}}\sum_{s=1}^t\left(\int \left[I\{\delta(X; \bar{\beta}_{s-1}) \ge 0\} - I\{\delta(X; \beta_0) \ge 0\}\right]I\{\delta(X; \beta_0) \ne 0\}\delta(X; \beta_0) d\p_X\right),
		\end{align}
		where we define $\delta(X; \beta) = \mu(1, X; \beta) - \mu(0, X; \beta)$ as the contrast function. Since (\ref{eq:part1}) is the summation of martingale differences, its asymptotic normality can be established using Martingale Central Limit Theorem. We want to show the extra term (\ref{eq:part2'}) is $o_p(1)$ so that $\sqrt{t}(\hat{V}_t-V)$ is also asymptotic normal. Note that (\ref{eq:part2'}) is always non-positive and 
		\begin{equation}\label{eq:subtract}
			\frac{1}{\sqrt{t}}\sum_{s=1}^t\left(\int [I\{\delta(X; \bar{\beta}_{s-1}) \ge 0\} - I\{\delta(X; \beta_0) \ge 0\}]I\{\delta(X; \beta_0) \ne 0\}\delta(X; \bar{\beta}_{s-1})d\p_X\right)
		\end{equation}
		is always non-negative. It suffices to provide an upper bound for (\ref{eq:subtract})$-$(\ref{eq:part2'}), or
		\begin{equation}\label{eq:part2''}
			\frac{1}{\sqrt{t}}\sum_{s=1}^t\left[\int I\{\delta(X; \beta_0) \ne 0\}\mathcal{I}_{s-1}(X)\Delta_{s-1}(X) d\p_X\right], 
		\end{equation}
		where we denote $\mathcal{I}_{s-1}(X) = I\{\delta(X; \bar{\beta}_{s-1}) \ge 0\} - I\{\delta(X; \beta_0) \ge 0\}$ and $\Delta_{s-1}(X) = \delta(X; \bar{\beta}_{s-1}) - \delta(X; \beta_0)$ for simplicity. We can split (\ref{eq:part2''}) into two parts:
		\begin{align*}
			J_1 &= \frac{1}{\sqrt{t}}\sum_{s=1}^t\left[\int I\{0 < |\delta(X; \beta_0)| \le t^{-\frac{1}{4}}\} \mathcal{I}_{s-1}(X)\Delta_{s-1}(X) d\p_X\right],\\
			J_2 &= \frac{1}{\sqrt{t}}\sum_{s=1}^t\left[\int I\{|\delta(X; \beta_0)| > t^{-\frac{1}{4}}\}\mathcal{I}_{s-1}(X)\Delta_{s-1}(X) d\p_X\right].
		\end{align*}
		Note that $\Delta_{s-1}(X) = X^T\{(-1, 1)\otimes I_p\}(\bar{\beta}_{s-1} - \beta_0)$ under our model setting. By Cauchy-Schwarz inequality, 
		$$
		\int |\Delta_{s-1}(X)| d\p_X \le \sqrt{2} \E\lVert X \rVert \lVert \bar{\beta}_{s-1} - \beta_0 \rVert.
		$$
		Then using Assumptions \ref{as:margin} and \ref{as:data} we have,
		\begin{align*}
			J_1 \le \frac{1}{\sqrt{t}}\sum_{s=1}^t\left[\int I\{0 < |\delta(X; \beta_0)| \le t^{-\frac{1}{4}}\}|\Delta_{s-1}(X)| d\p_X\right] \le \sqrt{2}C\E\lVert X \rVert t^{-\frac{1}{2}-\frac{\tau}{4}}\sum_{s=1}^t\lVert\bar{\beta}_{s-1}-\beta_0\rVert.
		\end{align*}
		By Theorem \ref{thm:paradist}, $\bar{\beta}_{s-1} - \beta_0 = O_P(s^{-\frac{1}{2}})$, so $\lVert\bar{\beta}_{s-1} - \beta_0\rVert = O_P(s^{-\frac{1}{2}}) = o_P(s^{\frac{\tau}{4}-\frac{1}{2}})$ for $\tau > 0$. Then apply Lemma 6 in \citet{luedtke2016statistical}, we have
		$t^{-1}\sum_{s=1}^t \lVert\bar{\beta}_{s-1} - \beta_0\rVert = o_P(t^{\frac{\tau}{4}-\frac{1}{2}})$ and hence $t^{-\frac{1}{2}-\frac{\tau}{4}}\sum_{s=1}^t\lVert\bar{\beta}_{s-1} - \beta_0\rVert = o_P(1)$.
		
		For $J_2$, use the fact that
		$$
		I\{|\Delta_{s-1}(X)| > |\delta(X; \beta_0)|\} \ge I\{\delta(X; \bar{\beta}_{s-1}) \ge 0\} - I\{\delta(X; \beta_0) \ge 0\} = \mathcal{I}_{s-1}(X),
		$$
		we have
		\begin{align*}
			J_2 &\le \frac{1}{\sqrt{t}}\sum_{s=1}^t\left[\int I\{|\delta(X; \beta_0)| > t^{-\frac{1}{4}}\}I\{|\Delta_{s-1}(X)| > |\delta(X; \beta_0)|\}|\Delta_{s-1}(X)| d\p_X\right]\\
			&\le  \frac{1}{\sqrt{t}}\sum_{s=1}^t\left[\int I\{|\delta(X; \beta_0)| > t^{-\frac{1}{4}}\}\frac{|\Delta_{s-1}(X)|^2}{|\delta(X; \beta_0)|} d\p_X\right]\\
			&\le  t^{-\frac{1}{4}}\sum_{s=1}^t\left[\int |\Delta_{s-1}(X)|^2 d\p_X\right] \le  2\E\lVert X \rVert^2 t^{-\frac{1}{4}}\sum_{s=1}^t \lVert\bar{\beta}_{s-1}-\beta_0\rVert^2.
		\end{align*}
		Similarly, $\lVert\bar{\beta}_{s-1} - \beta_0\rVert^2 = o_P(s^{-\frac{3}{4}})$ and $t^{-\frac{1}{4}}\sum_{s=1}^t \lVert\bar{\beta}_{s-1}-\beta_0\rVert^2 = o_P(1)$. Therefore (\ref{eq:part2}) is $o_P(1)$ and the asymptotic distribution of the value estimator depends only on (\ref{eq:part1}).
		
		Denote $D_s = C_s Y_s/\pi_{C, s} - \E(C_s Y_s/\pi_{C, s}|\mathcal{F}_{s-1})$. The conditional variance is
		\begin{align*}
			&\frac{1}{t}\sum_{s=1}^t \E(D_s^2|\mathcal{F}_{s-1})\\
			=& \frac{1}{t}\sum_{s=1}^t[\E\{(C_s Y_s/\pi_{C, s})^2|\mathcal{F}_{s-1}\} - \{\E(C_s Y_s/\pi_{C, s}|\mathcal{F}_{s-1})\}^2]\\
			=& \frac{1}{t}\sum_{s=1}^t\left[\frac{2}{2-\varepsilon_s}\int\theta^2(\hat{d}^{opt}_{s-1}(X), X)d\p_X-\{V(\hat{d}^{opt}_{s-1})\}^2\right].
		\end{align*}
		By Assumption \ref{as:data}, the summands are bounded continuous function of $\varepsilon_s$ and $\bar{\beta}_{s-1}$, so by Theorem \ref{thm:paradist} and Lemma \ref{lem:trivial}, we have $t^{-1}\sum_{s=1}^t \E(D_s^2|\mathcal{F}_{s-1}) \overset{p}{\to} \eta^2$ and
		\begin{equation*}
			\eta^2 = \frac{2}{2-\varepsilon_\infty}\int\theta^2(d^{opt}(X), X)d\p_X-V^2.
		\end{equation*}
		For any $\kappa > 0$, $\E \left(D_s^2 I\{D_s^2 > \kappa s\} \big| \mathcal{F}_{s-1}\right) \to 0$
		as $s \to \infty$ under Assumption \ref{as:data}. Then 
		$$
		\frac{1}{t} \sum_{s=1}^t \E \left(D_s^2 I\{D_s^2 > \kappa t\} \big| \mathcal{F}_{s-1}\right) \le \frac{1}{t} \sum_{s=1}^t \E \left(D_s^2 I\{D_s^2 > \kappa s\} \big| \mathcal{F}_{s-1}\right) \to 0
		$$
		by Lemma \ref{lem:trivial}. So the conditional Lindeberg condition is verified. Therefore the asymptotic normality of $\hat{V}_t$ is established by Martingale Central Limit Theorem. Now the consistency of the plugin estimator can be shown providing
		\begin{equation}\label{eq:converge}
			\frac{1}{t}\sum_{s=1}^t \frac{C_s Y_s^2}{\pi_{C, s}} \overset{p}{\to}\int\theta^2(d^{opt}(X), X)d\p_X.
		\end{equation}
		Follow the same technique used in (\ref{eq:expectation}), we can show that
		$$
		\E\left(\frac{C_s Y_s^2}{\pi_{C, s}}\bigg| \F_{s-1}\right) = \int\theta^2(\hat{d}_{s-1}^{opt}(X), X)d\p_X,
		$$
		which converges to $\int\theta^2(d^{opt}(X), X)d\p_X$ as a continuous function of $\bar{\beta}_{s-1}$. Then (\ref{eq:converge}) follows by repeating the argument used from (\ref{eq:Shat11}) to (\ref{eq:Shat22}).
	\end{proof}
	\newpage
	\section{Extended Simulation Results}\label{sec:extension}
	\subsection{Effect of different learning rates}
	
	\begin{figure}[!htbp]
	    \centering
	    \begin{subfigure}{0.75\textwidth}
	        \centering
    	    \includegraphics[width=\linewidth]{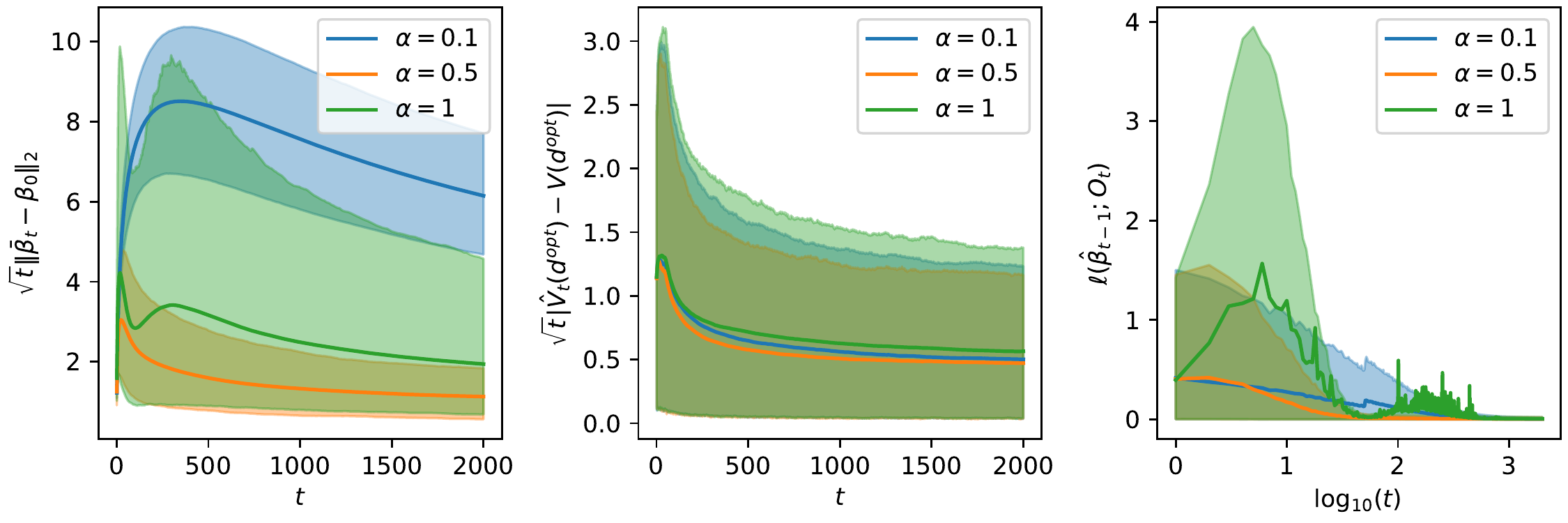}
    	    \caption{Linear reward model}
    	\end{subfigure}%
    	\\
    	\begin{subfigure}{0.75\textwidth}
	        \centering
    	    \includegraphics[width=\linewidth]{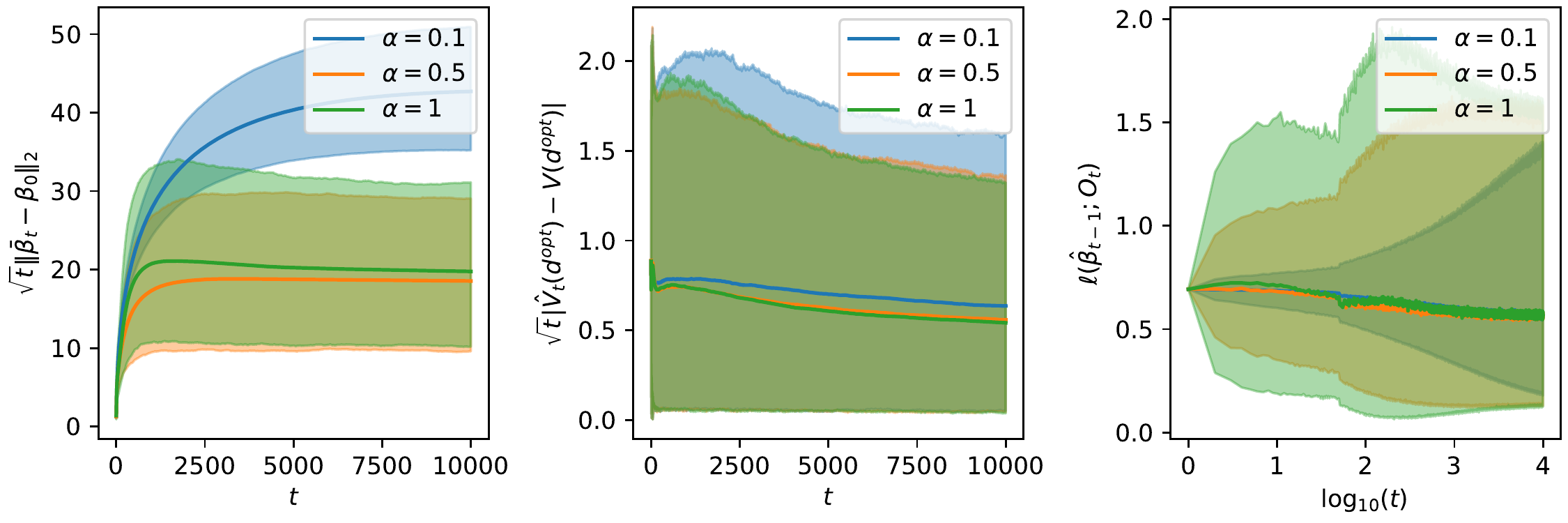}
    	    \caption{Logistic reward model}
    	\end{subfigure}
	    \caption{Performance of the online decision making algorithm with different learning rates. The exploration rate is $\varepsilon_t = 0.1$. All experiments are repeated 5000 times. The solid lines are mean outcomes and the shaded regions are bounded by 5\% and 95\% percentiles of the outcomes.}
	    \label{fig:tuning_e1}
	\end{figure}
	
	\begin{figure}[!htbp]
	    \centering
	    \begin{subfigure}{0.75\textwidth}
	        \centering
    	    \includegraphics[width=\linewidth]{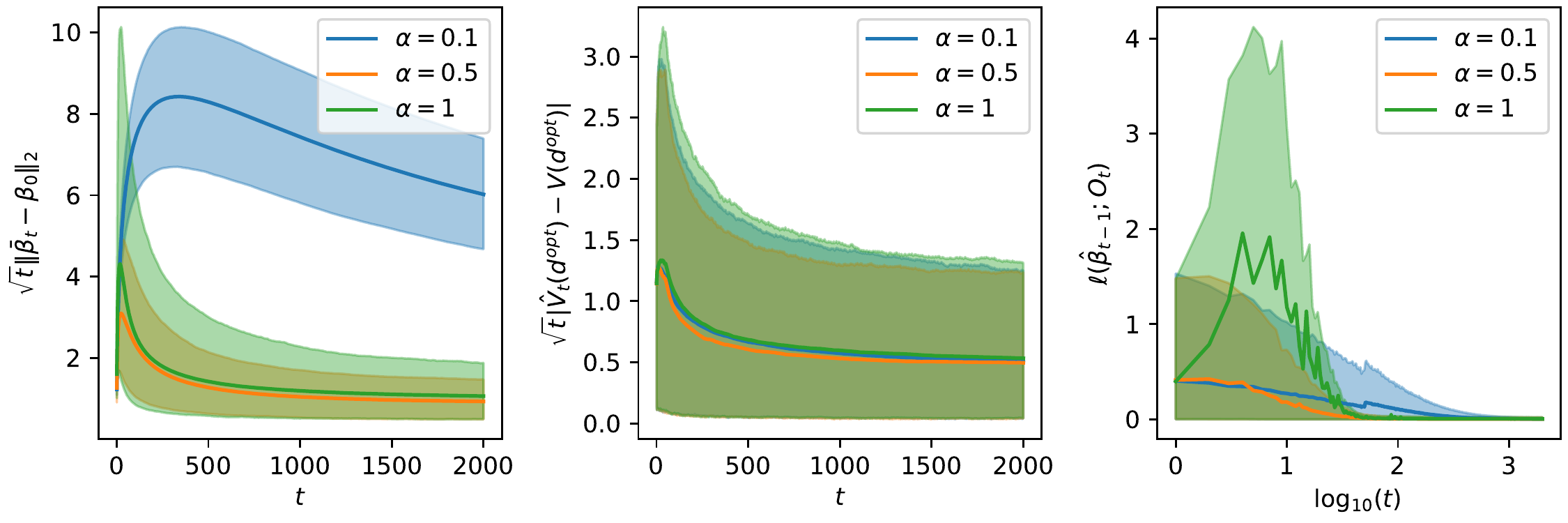}
    	    \caption{Linear reward model}
    	\end{subfigure}%
    	\\
    	\begin{subfigure}{0.75\textwidth}
	        \centering
    	    \includegraphics[width=\linewidth]{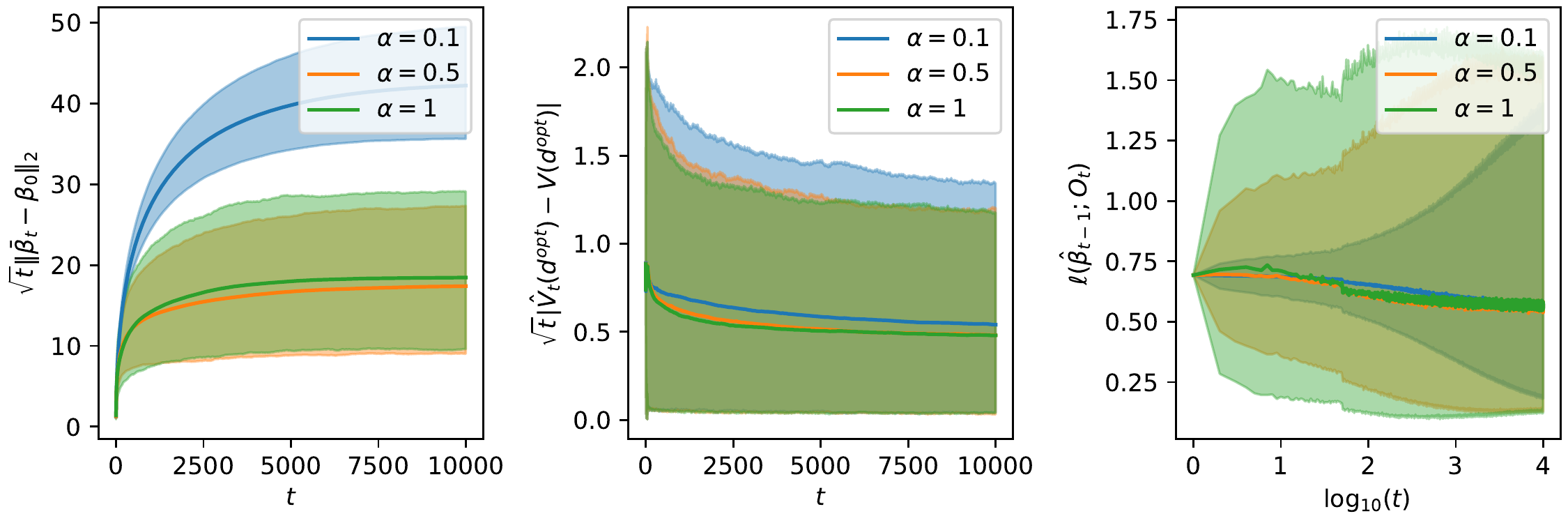}
    	    \caption{Logistic reward model}
    	\end{subfigure}
	    \caption{Performance of the online decision making algorithm with different learning rates. The exploration rate is $\varepsilon_t = t^{-0.3}\vee 0.1$. All experiments are repeated 5000 times. The solid lines are mean outcomes and the shaded regions are bounded by 5\% and 95\% percentiles of the outcomes.}
	    \label{fig:tuning_e3}
	\end{figure}
	
	\newpage
	\subsection{Parameter and value convergence for the other exploration rates}
	
	\begin{figure}[!htbp]
		\centering
		\begin{subfigure}{0.5\textwidth}
			\centering
			\includegraphics[width=\linewidth]{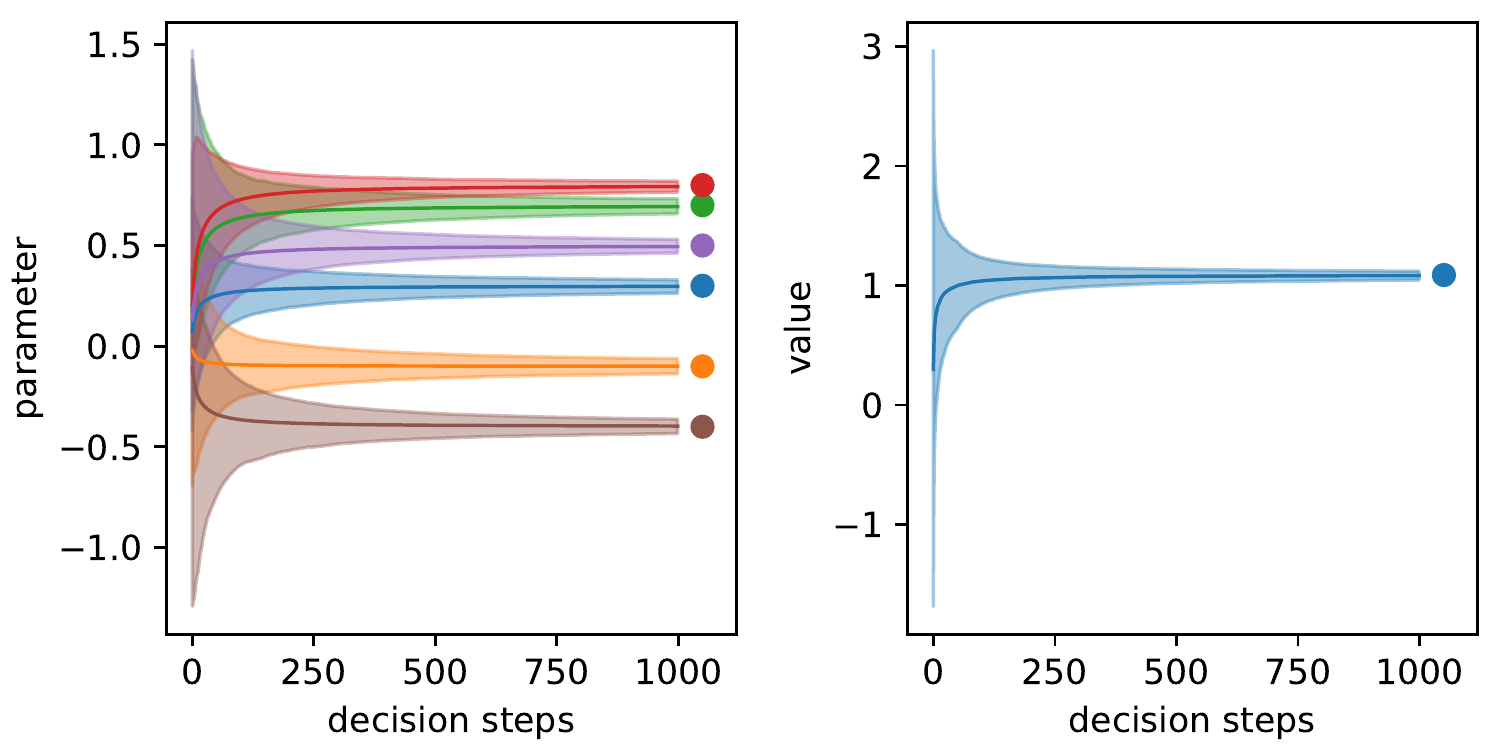}
			\caption{Linear model weighted SGD}
		\end{subfigure}%
		\begin{subfigure}{0.5\textwidth}
			\centering
			\includegraphics[width=\linewidth]{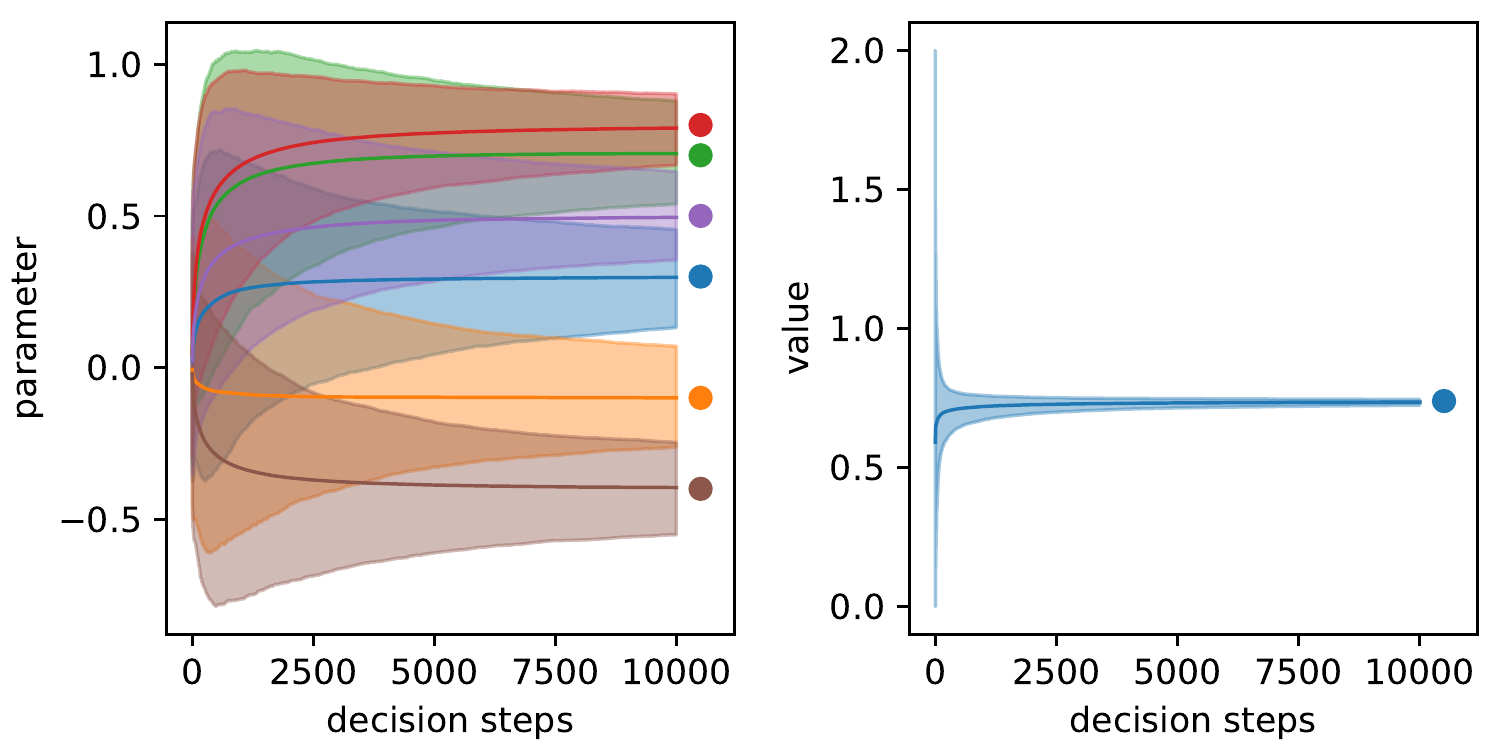}
			\caption{Logistic model weighted SGD}
		\end{subfigure}
		\\
		\begin{subfigure}{0.5\textwidth}
			\centering
			\includegraphics[width=\linewidth]{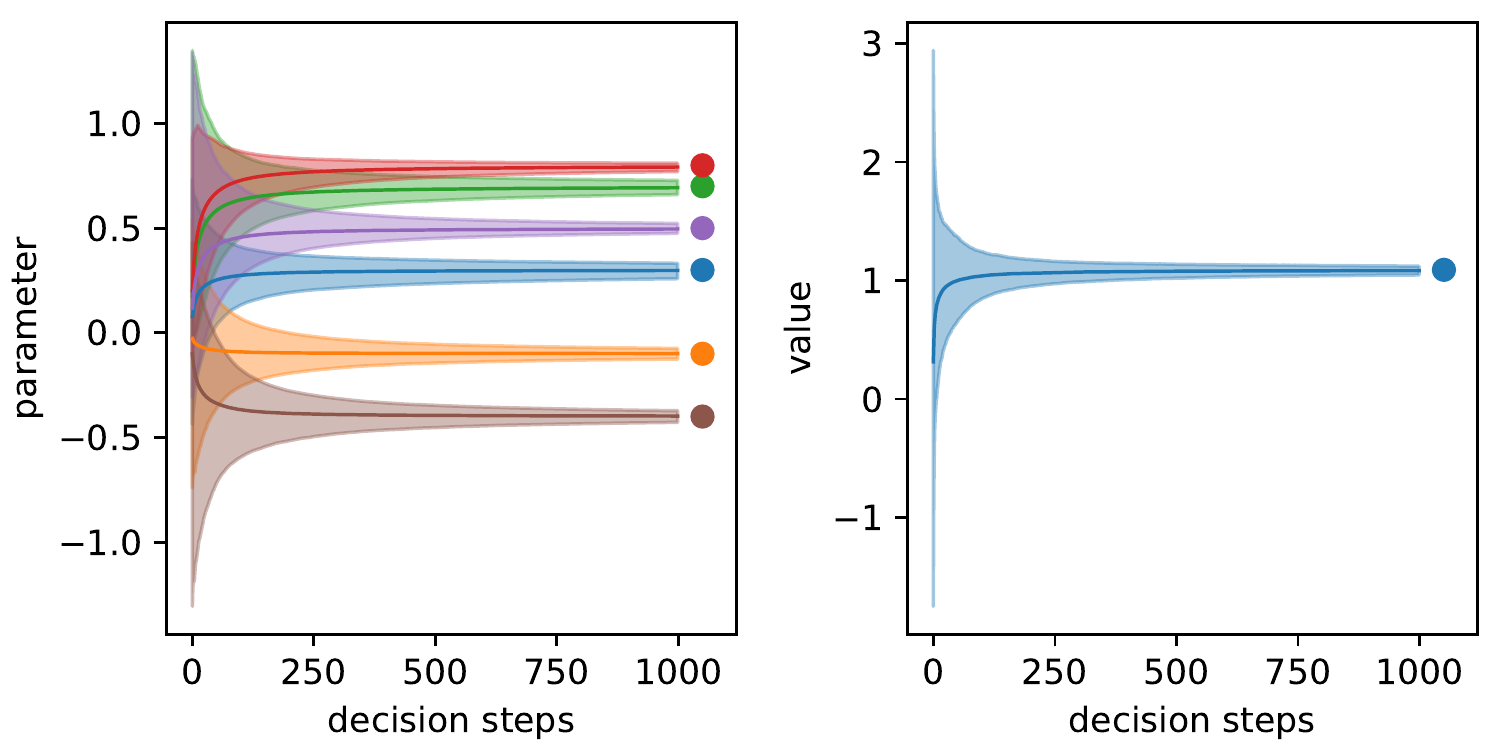}
			\caption{Linear model SGD}
		\end{subfigure}%
		\begin{subfigure}{0.5\textwidth}
			\centering
			\includegraphics[width=\linewidth]{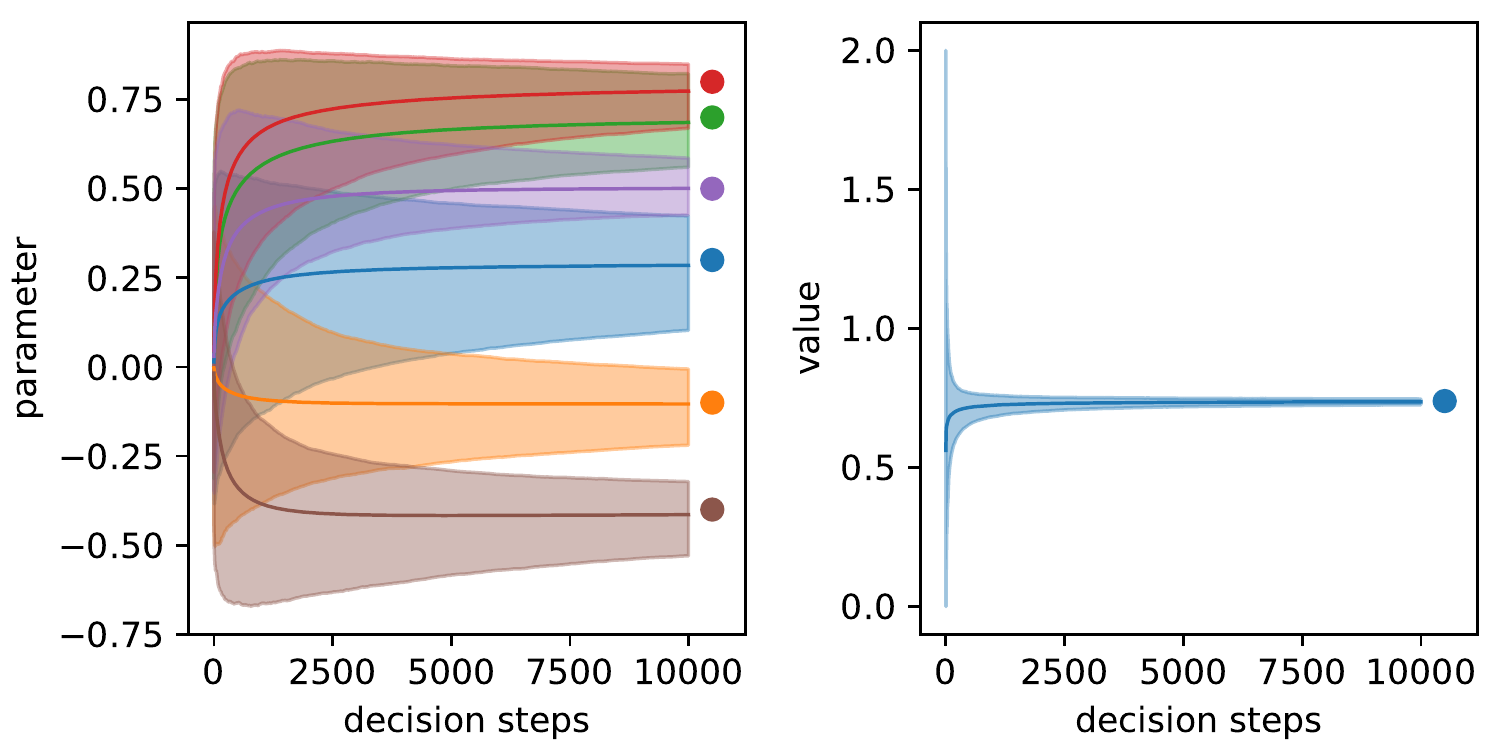}
			\caption{Logistic model SGD}
		\end{subfigure}
		\caption{Parameter and optimal value estimation from 5000 repeated experiments following the proposed weighted SGD method with IPW gradients and the conventional SGD method. The learning rate is $\alpha_t= 0.5t^{-0.501}$ and the exploration rate is $\varepsilon_t = 0.1$. The solid lines are mean estimates and the shaded regions are bounded by 2.5\% and 97.5\% percentiles of the estimates. The points at the end of the lines mark the true value.}
		\label{fig:consist_e1}
	\end{figure}
	
	\begin{figure}[!htbp]
		\centering
		\begin{subfigure}{0.5\textwidth}
			\centering
			\includegraphics[width=\linewidth]{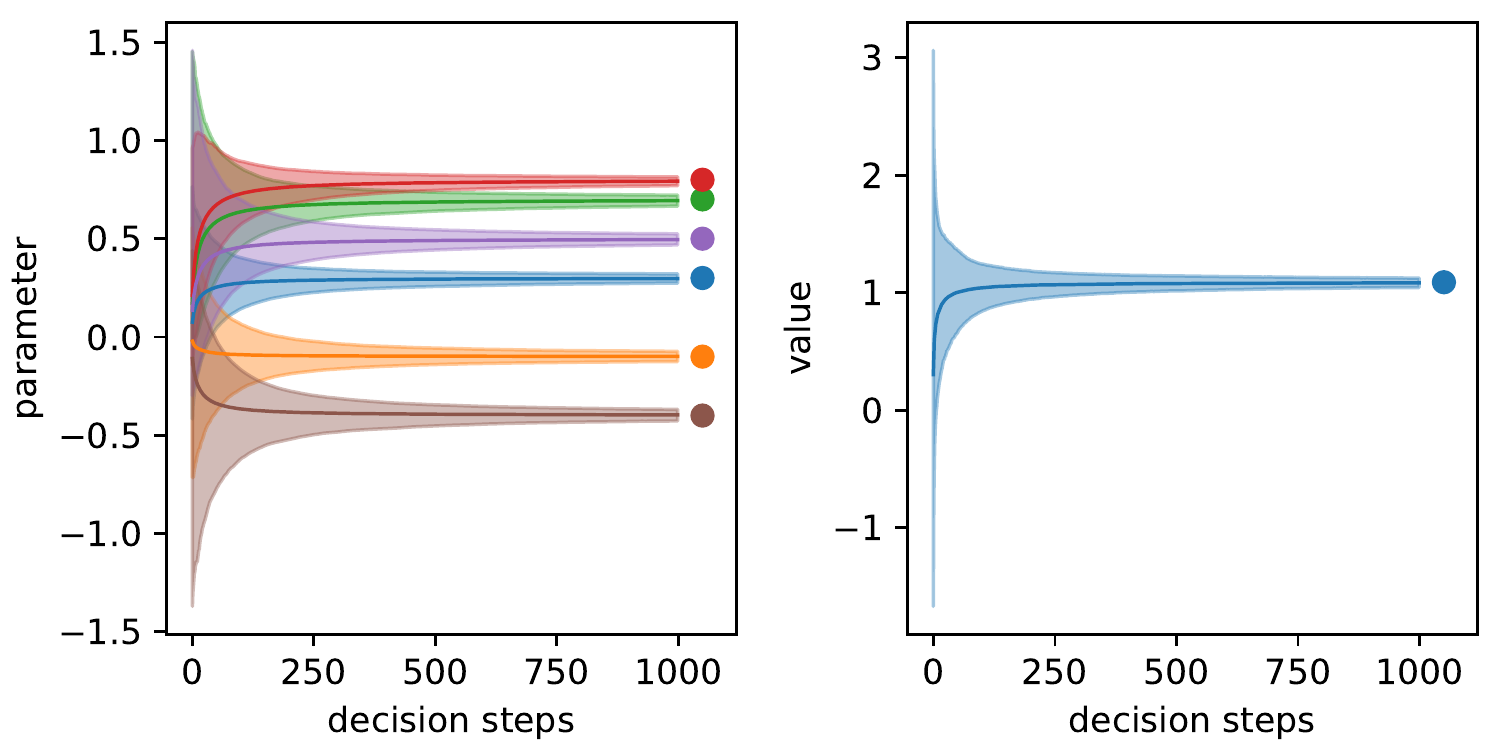}
			\caption{Linear model weighted SGD}
		\end{subfigure}%
		\begin{subfigure}{0.5\textwidth}
			\centering
			\includegraphics[width=\linewidth]{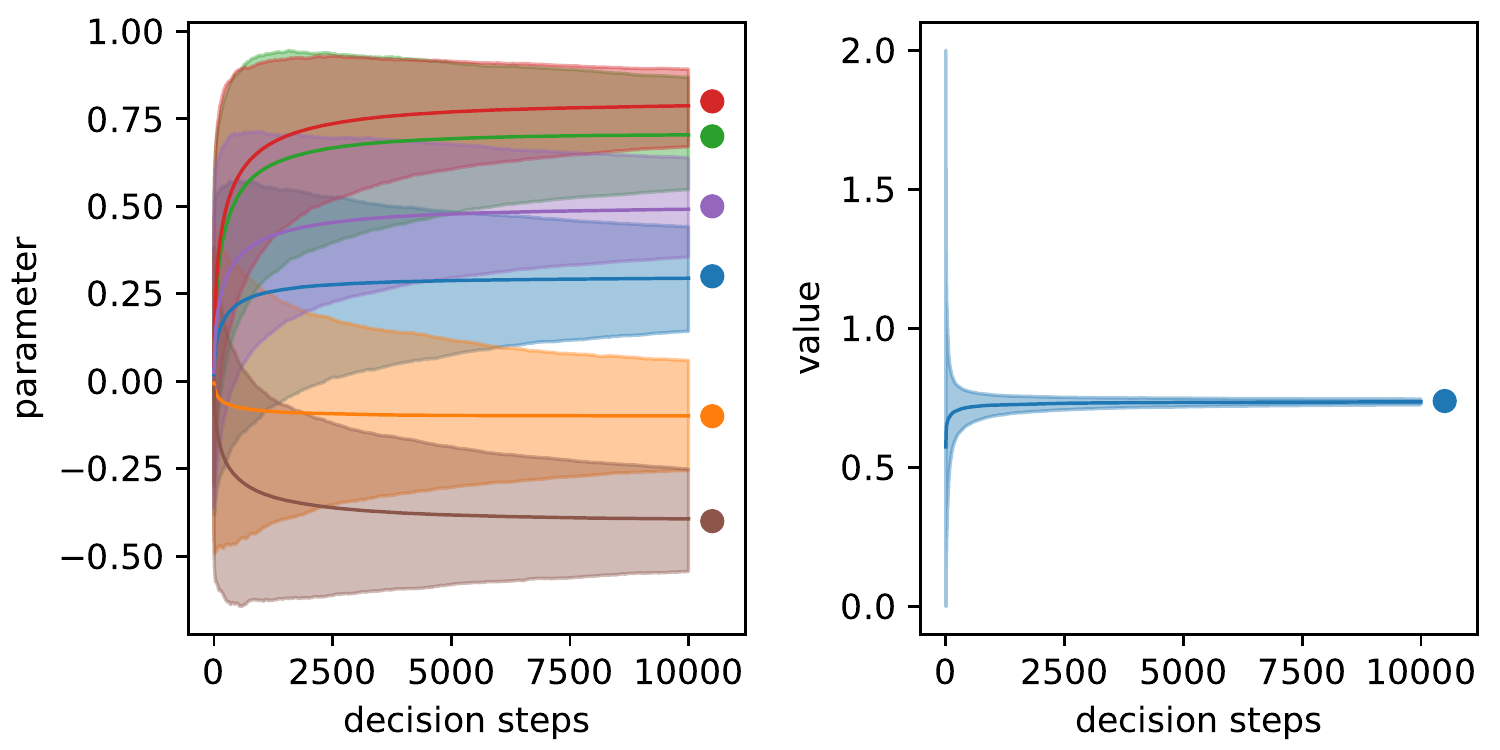}
			\caption{Logistic model weighted SGD}
		\end{subfigure}
		\\
		\begin{subfigure}{0.5\textwidth}
			\centering
			\includegraphics[width=\linewidth]{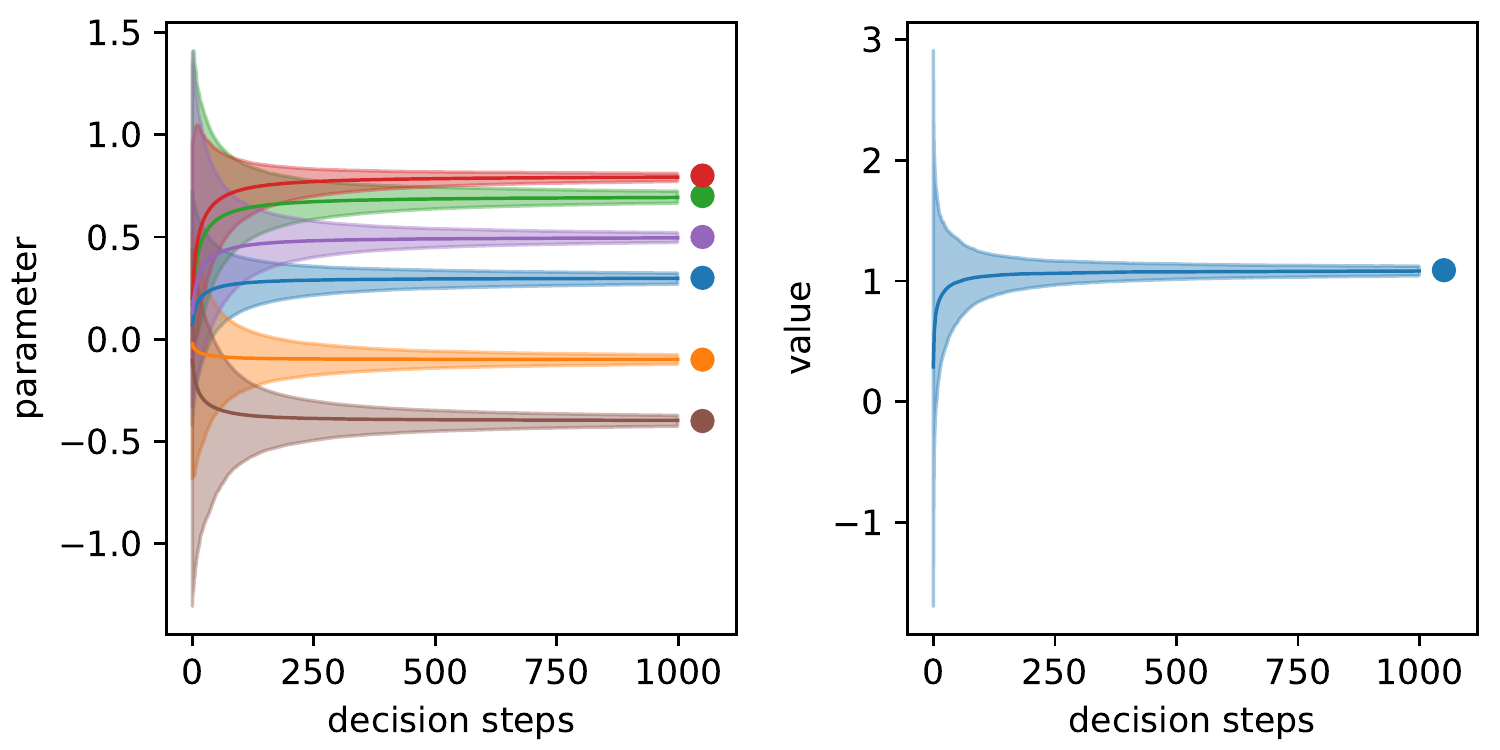}
			\caption{Linear model SGD}
		\end{subfigure}%
		\begin{subfigure}{0.5\textwidth}
			\centering
			\includegraphics[width=\linewidth]{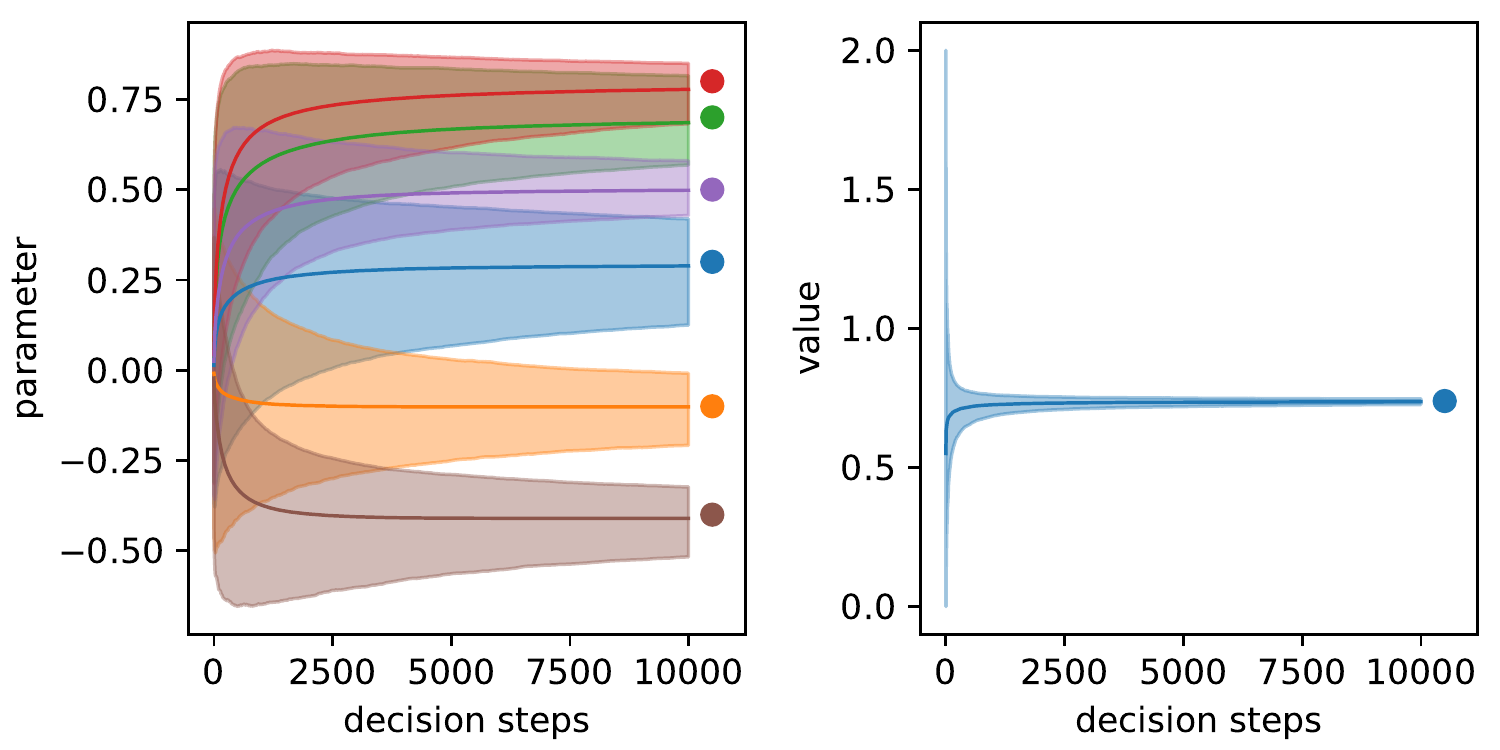}
			\caption{Logistic model SGD}
		\end{subfigure}
		\caption{Parameter and optimal value estimation from 5000 repeated experiments following the proposed weighted SGD method with IPW gradients and the conventional SGD method. The learning rate is $\alpha_t= 0.5t^{-0.501}$ and the exploration rate is $\varepsilon_t = t^{-0.3}\vee 0.1$. The solid lines are mean estimates and the shaded regions are bounded by 2.5\% and 97.5\% percentiles of the estimates. The points at the end of the lines mark the true value.}
		\label{fig:consist_e3}
	\end{figure}
	
	\newpage
	\subsection{Original data for Figure \ref{fig:exploration}}\label{sec:table}
	\begin{table}[!htbp]
		\centering
		\begin{tabular}{cccccccccc}
			\hline
			$\varepsilon_t$ & &$t$ &$\bar{\beta}_{t,01}$ &$\bar{\beta}_{t,02}$ &$\bar{\beta}_{t,03}$ &$\bar{\beta}_{t,11}$ &$\bar{\beta}_{t,12}$ &$\bar{\beta}_{t,13}$ & $\hat{V}_t(d^{opt})$\\
			\hline 
			           &R &       &0.876 &0.821 &0.796 &0.840 &0.815 &0.716 &0.894\\
			           &C &$10^3$ &0.919 &0.897 &0.875 &0.875 &0.885 &0.888 &0.903\\
			           &L &       &0.059 &0.062 &0.063 &0.049 &0.055 &0.059 &0.068\\
			           \cline{2-10}
			           &R &       &0.952 &0.904 &0.911 &0.945 &0.910 &0.901 &0.952\\
			Fixed      &C &$10^4$ &0.935 &0.928 &0.924 &0.933 &0.923 &0.923 &0.936\\
			0.1        &L &       &0.015 &0.015 &0.014 &0.011 &0.012 &0.013 &0.021\\
			           \cline{2-10}
			           &R &       &0.970 &0.973 &0.979 &0.974 &0.992 &0.955 &0.998\\
			           &C &$10^5$ &0.943 &0.945 &0.945 &0.944 &0.949 &0.940 &0.945\\
			           &L &       &0.005 &0.004 &0.004 &0.003 &0.004 &0.004 &0.007\\
			\hline 
			           &R &       &0.929 &0.907 &0.882 &0.892 &0.814 &0.854 &0.929\\
			           &C &$10^3$ &0.933 &0.934 &0.892 &0.857 &0.900 &0.916 &0.911\\
			           &L &       &0.043 &0.045 &0.048 &0.039 &0.045 &0.046 &0.077\\
			           \cline{2-10}
			           &R &       &0.963 &0.970 &0.957 &0.951 &0.938 &0.924 &1.005\\
			Fixed      &C &$10^4$ &0.938 &0.942 &0.934 &0.926 &0.933 &0.927 &0.945\\
			0.2        &L &       &0.011 &0.011 &0.010 &0.009 &0.009 &0.010 &0.024\\
			           \cline{2-10}
			           &R &       &0.983 &0.972 &0.961 &0.995 &0.988 &0.993 &1.008\\
			           &C &$10^5$ &0.947 &0.945 &0.940 &0.948 &0.945 &0.952 &0.948\\
			           &L &       &0.003 &0.003 &0.003 &0.003 &0.003 &0.003 &0.008\\
			\hline
			           &R &       &0.946 &0.911 &0.884 &0.905 &0.825 &0.857 &0.837\\
			           &C &$10^3$ &0.934 &0.931 &0.894 &0.870 &0.899 &0.916 &0.881\\
			           &L &       &0.045 &0.048 &0.050 &0.040 &0.046 &0.048 &0.068\\
			           \cline{2-10}
			           &R &       &0.963 &0.924 &0.923 &0.946 &0.929 &0.925 &0.995\\
			Decreasing &C &$10^4$ &0.942 &0.930 &0.924 &0.928 &0.931 &0.928 &0.947\\
			           &L &       &0.014 &0.014 &0.014 &0.011 &0.012 &0.013 &0.021\\
			           \cline{2-10}
			           &R &       &0.965 &0.975 &0.974 &0.987 &0.975 &0.977 &0.993\\
			           &C &$10^5$ &0.943 &0.948 &0.944 &0.944 &0.945 &0.946 &0.946\\
			           &L &       &0.005 &0.004 &0.004 &0.003 &0.004 &0.004 &0.007\\
			\hline
		\end{tabular}
		\caption{Average standard error to Monte Carlo standard deviation ratio (R), coverage probability (C), and average length of 95\% confidence interval (L) of parameter and value estimators in the linear model setting.}
		\label{tab:linear}
	\end{table}
	
	\begin{table}[!htbp]
		\centering
		\begin{tabular}{cccccccccc}
			\hline
			$\varepsilon_t$ &  &$t$ &$\bar{\beta}_{t,01}$ &$\bar{\beta}_{t,02}$ &$\bar{\beta}_{t,03}$ &$\bar{\beta}_{t,11}$ &$\bar{\beta}_{t,12}$ &$\bar{\beta}_{t,13}$ & $\hat{V}_t(d^{opt})$\\
			\hline
			           &R &       &1.071 &1.033 &1.063 &1.011 &0.994 &1.007 &0.671\\
			           &C &$10^3$ &0.962 &0.954 &0.945 &0.923 &0.898 &0.910 &0.717\\
			           &L &       &0.963 &0.977 &0.960 &0.708 &0.795 &0.846 &0.061\\
			           \cline{2-10}
			           &R &       &0.947 &0.925 &0.919 &0.953 &0.919 &0.922 &0.878\\
			Fixed      &C &$10^4$ &0.939 &0.934 &0.929 &0.941 &0.923 &0.925 &0.840\\
			0.1        &L &       &0.308 &0.312 &0.319 &0.224 &0.275 &0.287 &0.019\\
			           \cline{2-10}
			           &R &       &0.979 &0.963 &0.974 &0.994 &0.966 &0.987 &0.961\\
			           &C &$10^5$ &0.948 &0.941 &0.939 &0.948 &0.944 &0.950 &0.923\\
			           &L &       &0.097 &0.097 &0.100 &0.070 &0.088 &0.091 &0.006\\
			\hline
			           &R &       &1.037 &1.029 &1.016 &0.968 &0.992 &1.013 &0.816\\
			           &C &$10^3$ &0.955 &0.954 &0.909 &0.849 &0.876 &0.909 &0.814\\
			           &L &       &0.675 &0.684 &0.669 &0.521 &0.583 &0.614 &0.066\\
			           \cline{2-10}
			           &R &       &0.948 &0.924 &0.921 &0.940 &0.922 &0.924 &0.951\\
			Fixed      &C &$10^4$ &0.940 &0.926 &0.927 &0.929 &0.924 &0.926 &0.905\\
			0.2        &L &       &0.219 &0.220 &0.225 &0.166 &0.199 &0.206 &0.021\\
			           \cline{2-10}
			           &R &       &0.978 &0.966 &0.969 &0.993 &0.967 &0.970 &0.974\\
			           &C &$10^5$ &0.945 &0.945 &0.939 &0.946 &0.944 &0.945 &0.936\\
			           &L &       &0.069 &0.069 &0.072 &0.053 &0.064 &0.066 &0.007\\
			\hline
			           &R &       &1.117 &1.127 &1.127 &1.058 &1.059 &1.101 &0.781\\
			           &C &$10^3$ &0.967 &0.973 &0.944 &0.885 &0.899 &0.927 &0.792\\
			           &L &       &0.753 &0.758 &0.739 &0.566 &0.635 &0.676 &0.060\\
			           \cline{2-10}
			           &R &       &0.975 &0.944 &0.948 &0.960 &0.962 &0.954 &0.936\\
			Decreasing &C &$10^4$ &0.943 &0.938 &0.935 &0.942 &0.937 &0.935 &0.886\\
			           &L &       &0.298 &0.301 &0.306 &0.218 &0.266 &0.278 &0.019\\
			           \cline{2-10}
			           &R &       &0.976 &0.963 &0.978 &0.969 &0.953 &0.966 &0.991\\
			           &C &$10^5$ &0.944 &0.940 &0.935 &0.940 &0.935 &0.937 &0.939\\
			           &L &       &0.096 &0.097 &0.100 &0.070 &0.087 &0.091 &0.006\\
			\hline
		\end{tabular}
		\caption{Average standard error to Monte Carlo standard deviation ratio (R), coverage probability (C), and average length of 95\% confidence interval (L) of parameter and value estimators in the logistic model setting.}
		\label{tab:logistic}
	\end{table}
	
	\newpage
	\subsection{Simulation results of the linear reward model with $\sigma^2 = 0.25$}
	
	\begin{figure}[!htbp]
		\centering
		\begin{subfigure}{0.75\textwidth}
			\centering
			\includegraphics[width=\linewidth]{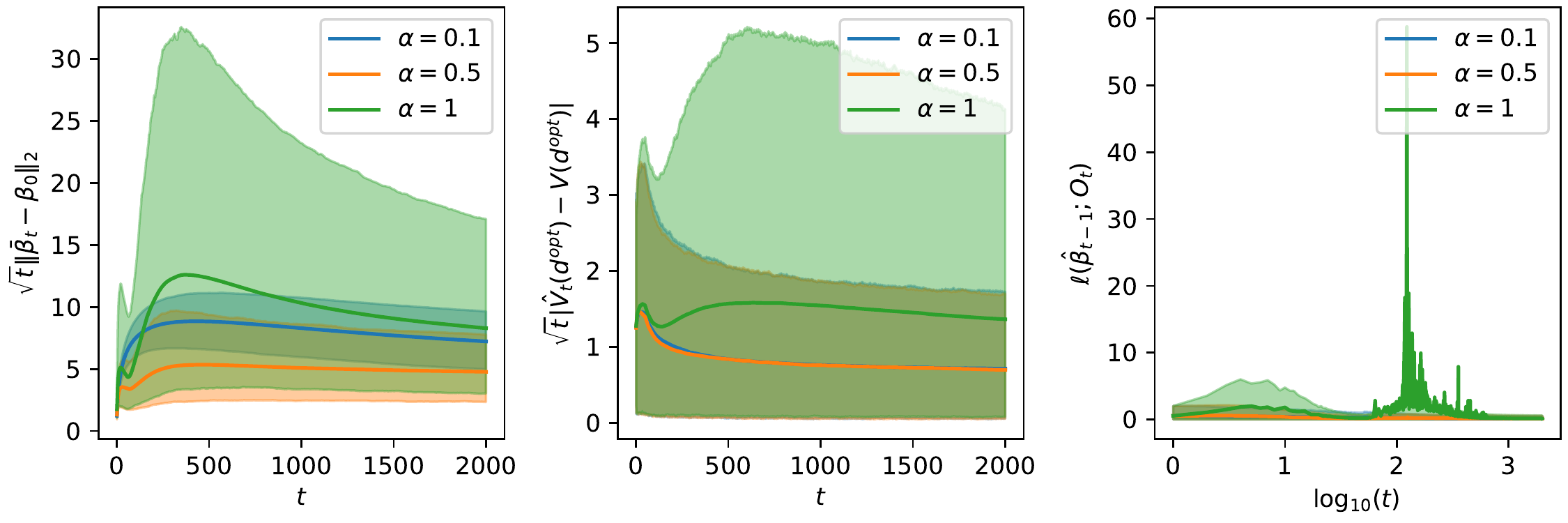}
			\caption{$\varepsilon_t = 0.1$}
		\end{subfigure}%
		\\
		\begin{subfigure}{0.75\textwidth}
			\centering
			\includegraphics[width=\linewidth]{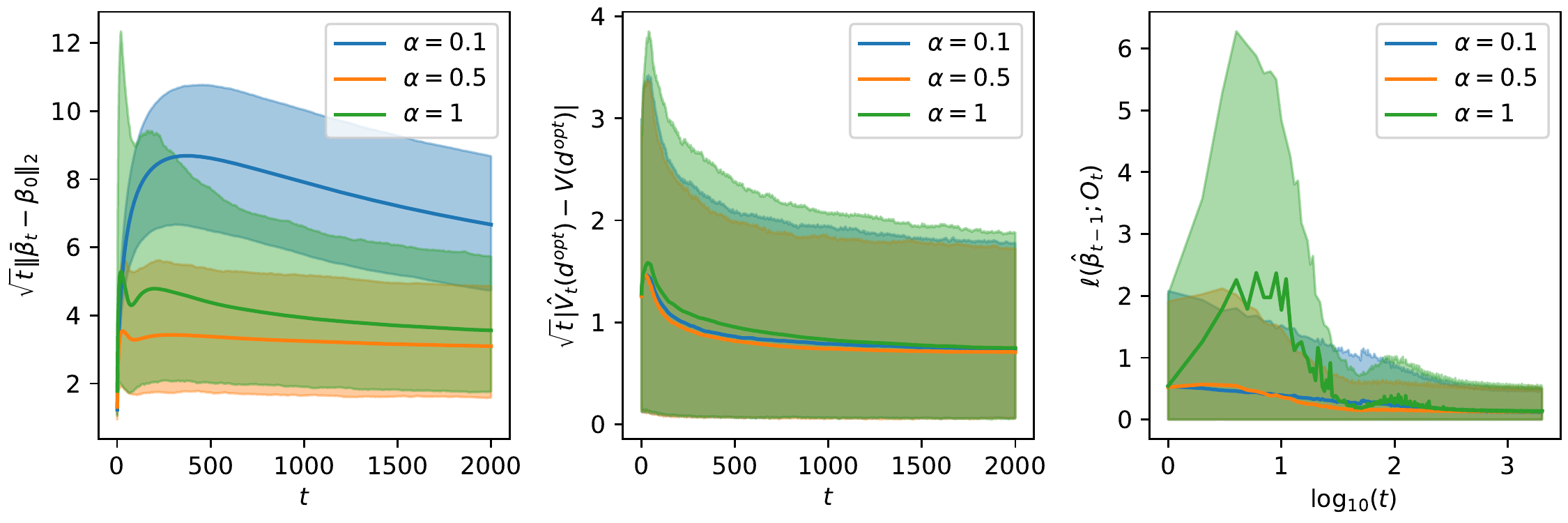}
			\caption{$\varepsilon_t = 0.2$}
		\end{subfigure}
		\\
		\begin{subfigure}{0.75\textwidth}
			\centering
			\includegraphics[width=\linewidth]{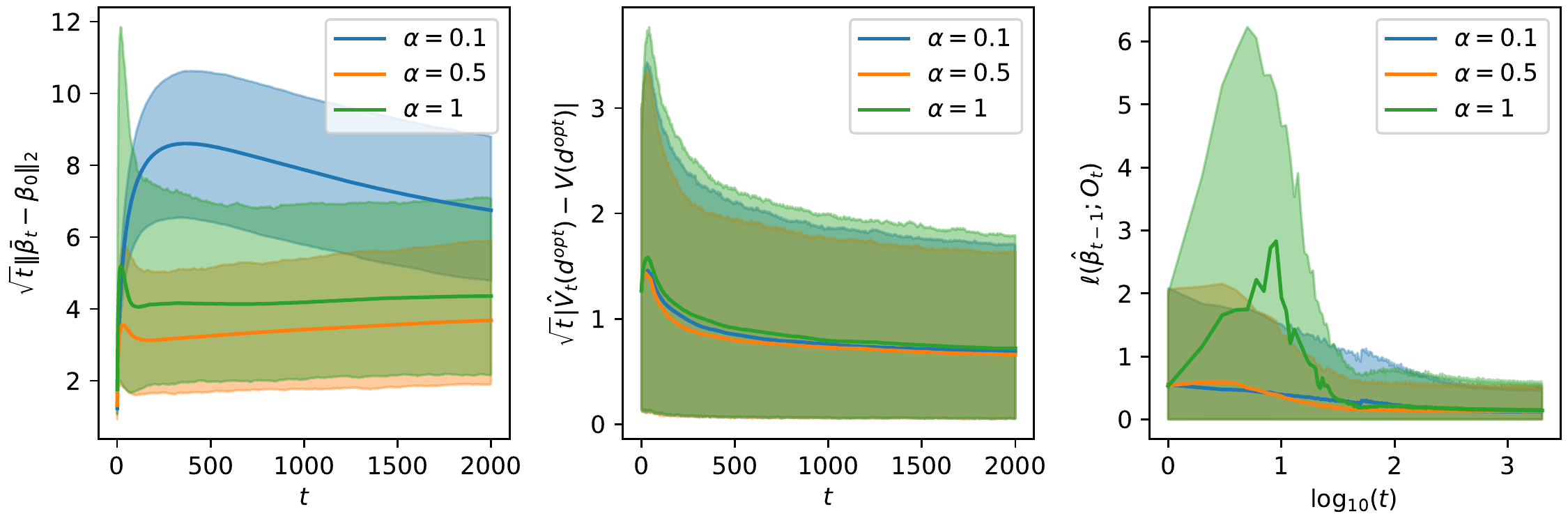}
			\caption{$\varepsilon_t = t^{-0.3}\vee 0.1$}
		\end{subfigure}
		\caption{Performance of the online decision making algorithm with different learning rates and exploration rates. The reward model is linear and $\sigma^2=0.25$. All experiments are repeated 5000 times. The solid lines are mean outcomes and the shaded regions are bounded by 5\% and 95\% percentiles of the outcomes.}
	\end{figure}
	
	\begin{figure}[!htbp]
		\centering
		\begin{subfigure}{0.5\textwidth}
			\centering
			\includegraphics[width=\linewidth]{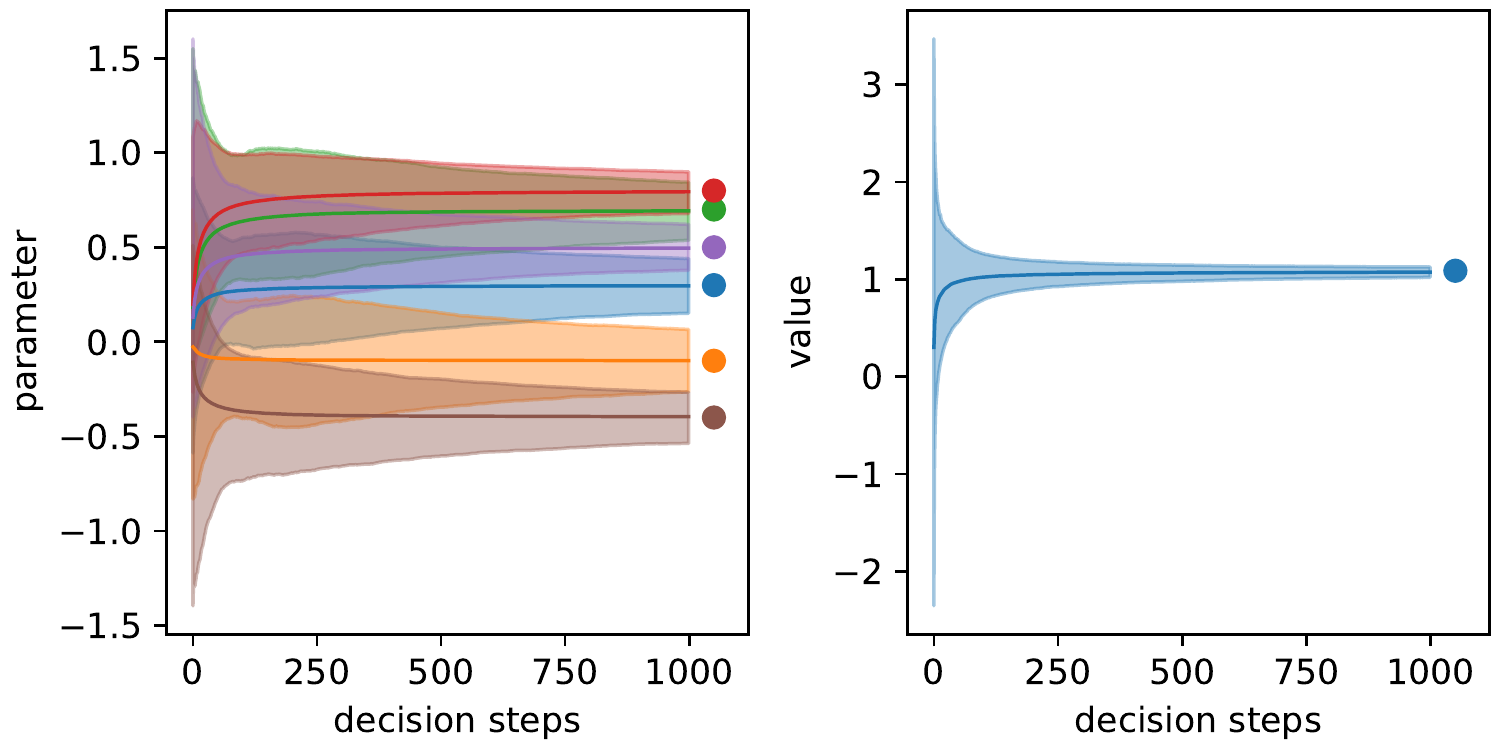}
			\caption{$\varepsilon_t = 0.1$}
		\end{subfigure}%
		\begin{subfigure}{0.5\textwidth}
			\centering
			\includegraphics[width=\linewidth]{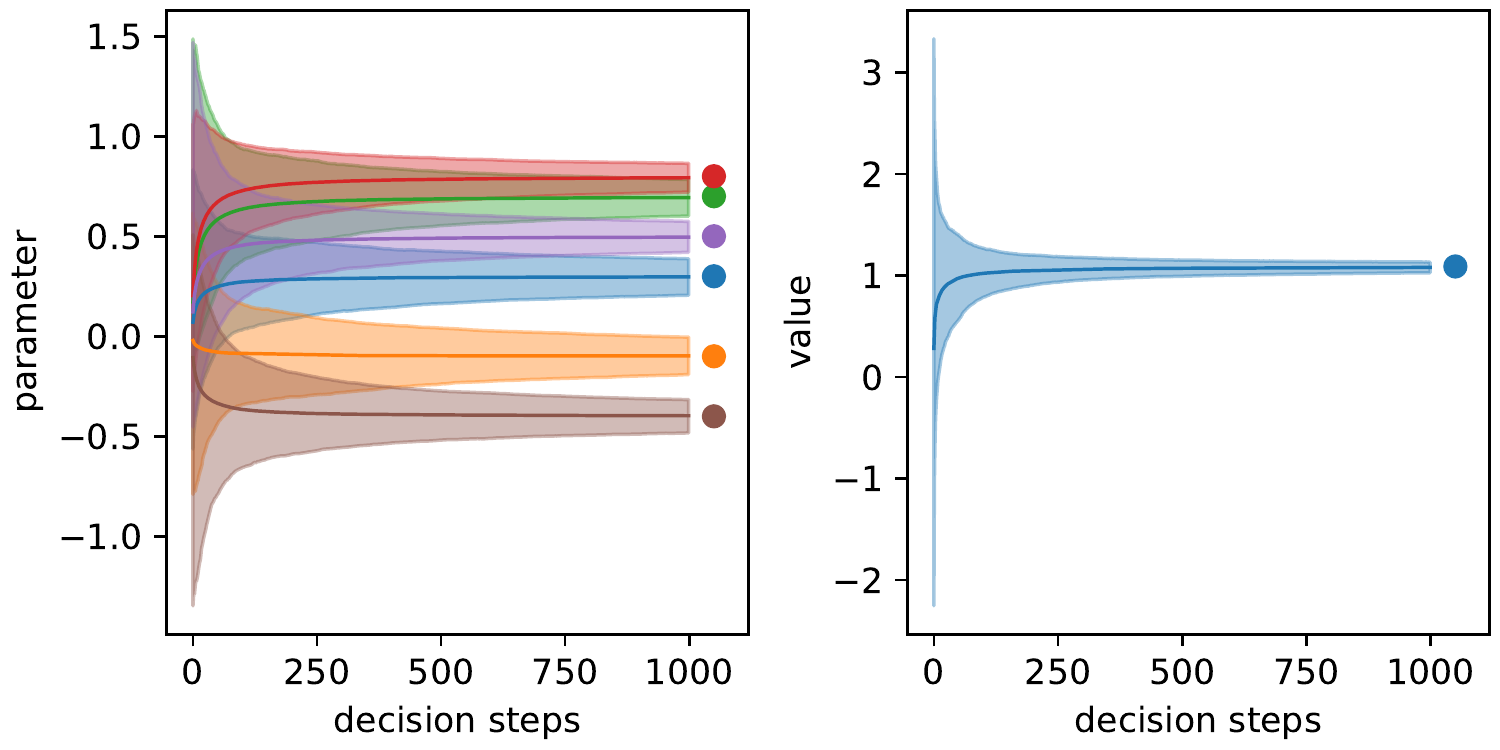}
			\caption{$\varepsilon_t = 0.2$}
		\end{subfigure}
		\\
		\begin{subfigure}{\textwidth}
			\centering
			\includegraphics[width=0.5\linewidth]{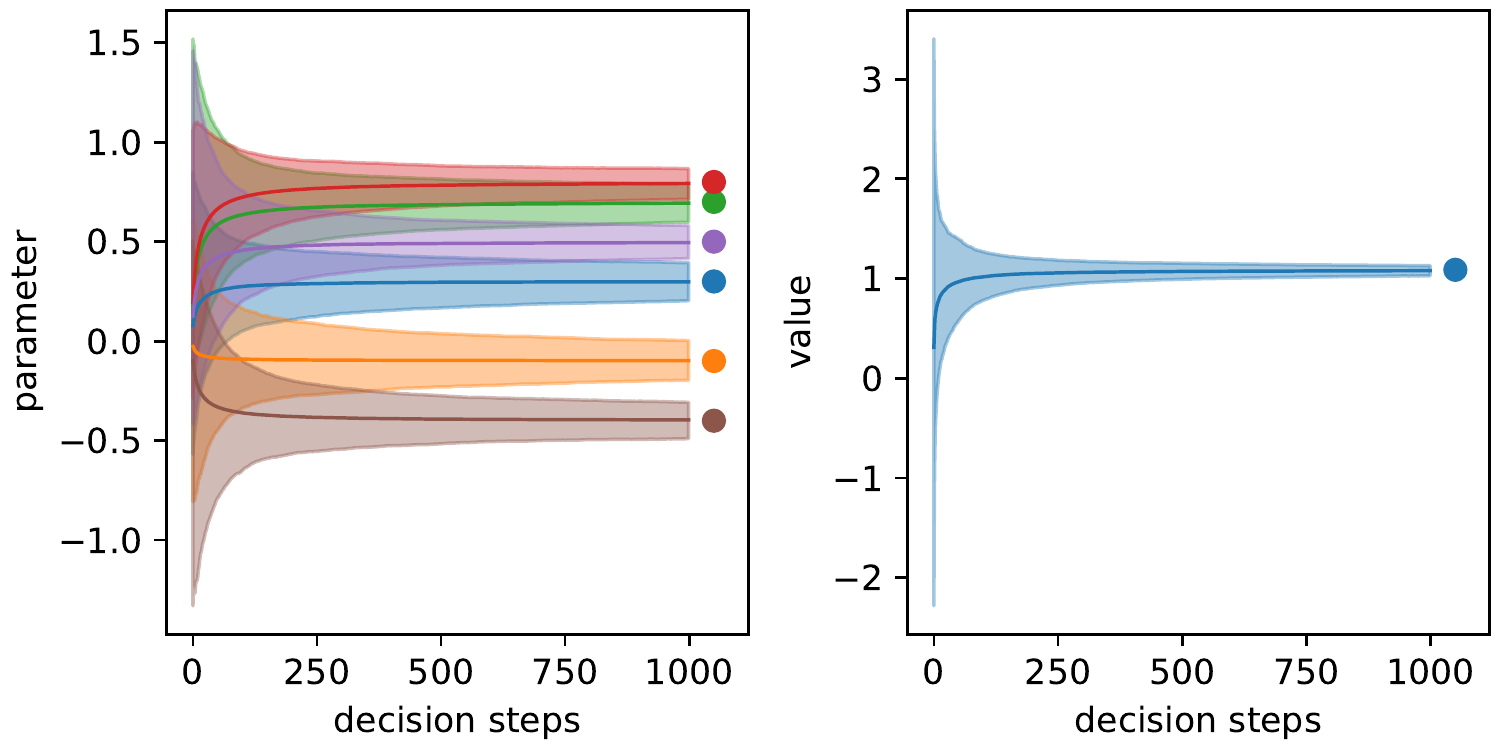}
			\caption{$\varepsilon_t = t^{-0.3}\vee 0.1$}
		\end{subfigure}
		\caption{Parameter and optimal value estimation from 5000 repeated experiments following the proposed weighted SGD method with IPW gradients. The reward model is linear and $\sigma^2=0.25$. The learning rate is $\alpha_t= 0.5t^{-0.501}$. The solid lines are mean estimates and the shaded regions are bounded by 2.5\% and 97.5\% percentiles of the estimates. The points at the end of the lines mark the true value.}
	\end{figure}
	
	\begin{figure}[!htbp]
		\centering
		\includegraphics[width=0.75\textwidth]{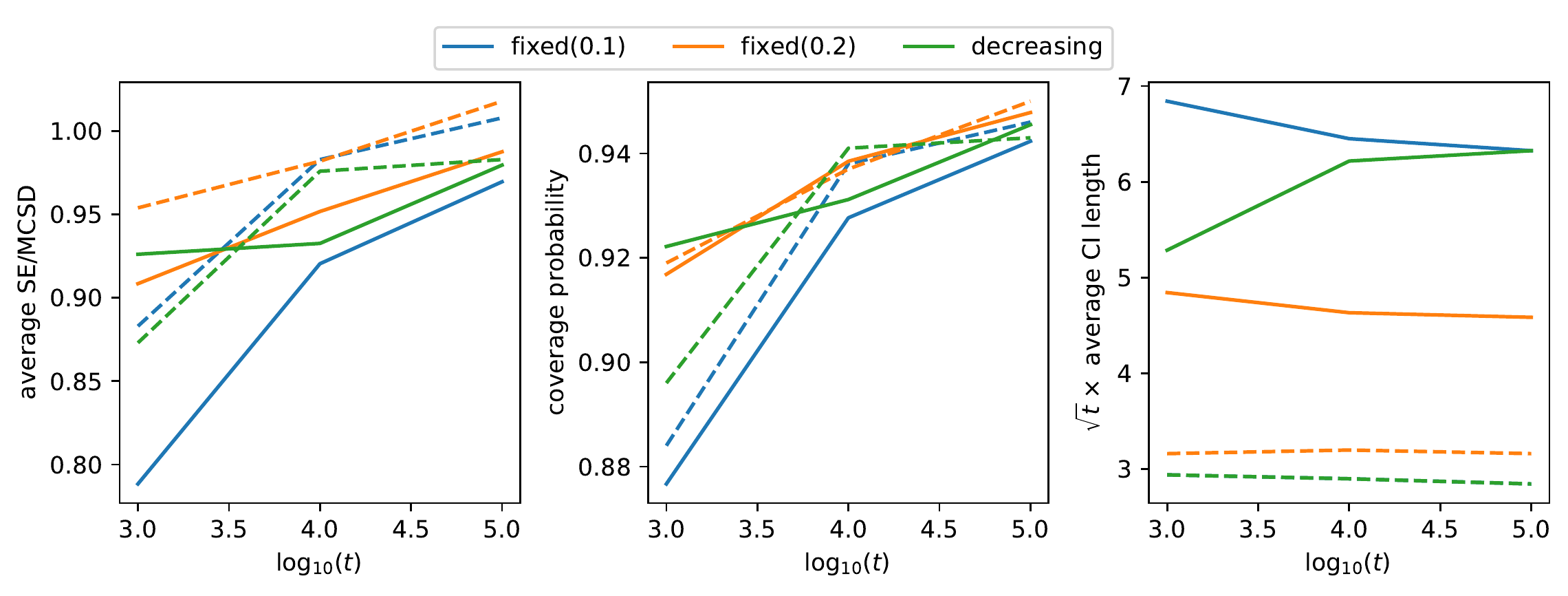}
		\caption{Online plugin variance estimation for the parameter and value estimators with different exploration rates. The reward model is linear and $\sigma^2=0.25$. The learning rate is $\alpha_t = 0.5t^{-0.501}$. The solid lines are the average results of the six parameters and the dashed lines are the results for the value estimation.}
	\end{figure}
	
	\begin{table}[!htbp]
		\centering
		\begin{tabular}{cccccccccc}
			\hline
			$\varepsilon_t$ & &$t$ &$\bar{\beta}_{t,01}$ &$\bar{\beta}_{t,02}$ &$\bar{\beta}_{t,03}$ &$\bar{\beta}_{t,11}$ &$\bar{\beta}_{t,12}$ &$\bar{\beta}_{t,13}$ & $\hat{V}_t(d^{opt})$\\
			\hline 
			&R &                &0.830 &0.754 &0.742 &0.855 &0.784 &0.765 &0.883\\
			&C &$10^3$          &0.897 &0.864 &0.858 &0.906 &0.865 &0.870 &0.884\\
			&L &                &0.246 &0.244 &0.234 &0.181 &0.187 &0.206 &0.093\\
			\cline{2-10}
			&R &                &0.962 &0.907 &0.907 &0.937 &0.900 &0.910 &0.983\\
			Fixed&C&$10^4$      &0.943 &0.920 &0.924 &0.932 &0.924 &0.923 &0.938\\
			0.1 &L &            &0.073 &0.072 &0.068 &0.055 &0.057 &0.062 &0.029\\
			\cline{2-10}
			&R &                &0.975 &0.953 &0.983 &0.970 &0.970 &0.967 &1.008\\
			&C &$10^5$          &0.944 &0.938 &0.949 &0.937 &0.941 &0.945 &0.946\\
			&L &                &0.023 &0.022 &0.021 &0.017 &0.018 &0.019 &0.009\\
			\hline
			&R &                &0.939 &0.889 &0.882 &0.934 &0.910 &0.897 &0.954\\
			&C &$10^3$          &0.931 &0.908 &0.908 &0.927 &0.916 &0.911 &0.919\\
			&L &                &0.169 &0.168 &0.161 &0.133 &0.139 &0.149 &0.100\\
			\cline{2-10}
			&R &                &0.966 &0.957 &0.951 &0.946 &0.942 &0.949 &0.982\\
			Fixed&C&$10^4$      &0.942 &0.941 &0.937 &0.937 &0.938 &0.936 &0.937\\
			0.2 &L &            &0.052 &0.051 &0.048 &0.040 &0.042 &0.045 &0.032\\
			\cline{2-10}
			&R &                &0.988 &0.987 &0.976 &1.000 &0.997 &0.978 &1.018\\
			&C &$10^5$          &0.945 &0.948 &0.945 &0.952 &0.950 &0.947 &0.950\\
			&L &                &0.016 &0.016 &0.015 &0.013 &0.013 &0.014 &0.010\\
			\hline
			&R &                &0.968 &0.911 &0.902 &0.956 &0.914 &0.906 &0.873\\
			&C &$10^3$          &0.940 &0.916 &0.911 &0.932 &0.920 &0.914 &0.896\\
			&L &                &0.187 &0.184 &0.177 &0.144 &0.149 &0.162 &0.093\\
			\cline{2-10}
			&R &                &0.950 &0.924 &0.921 &0.944 &0.930 &0.927 &0.976\\
			Decreasing&C&$10^4$ &0.940 &0.929 &0.931 &0.932 &0.931 &0.924 &0.941\\
			&L &                &0.071 &0.069 &0.065 &0.053 &0.055 &0.060 &0.029\\
			\cline{2-10}
			&R &                &1.000 &0.975 &0.994 &0.980 &0.970 &0.959 &0.983\\
			&C &$10^5$          &0.954 &0.941 &0.949 &0.945 &0.943 &0.941 &0.943\\
			&L &                &0.023 &0.022 &0.021 &0.017 &0.018 &0.019 &0.009\\
			\hline
		\end{tabular}
		\caption{Average standard error to Monte Carlo standard deviation ratio (R), coverage probability (C), and average length of 95\% confidence interval (L) of parameter and value estimators in the linear model setting with $\sigma^2=0.25$.}
		\label{tab:linear2}
	\end{table}
	
	\newpage
	
	\subsection{Comparison of variance estimation methods for the other exploration rates}
	
	\begin{figure}[!htbp]
		\centering
		\begin{subfigure}{0.5\textwidth}
			\centering
			\includegraphics[width=\linewidth]{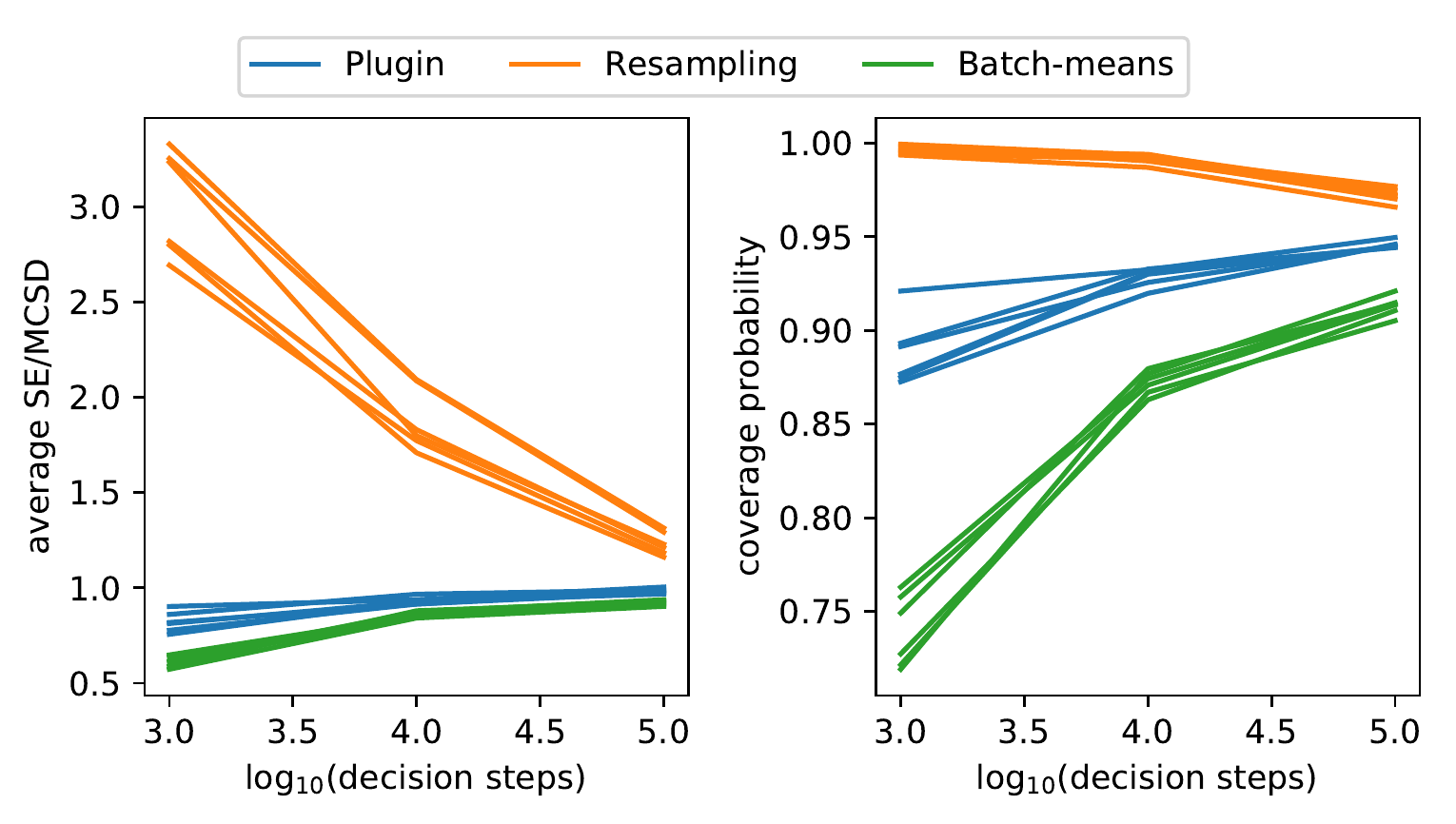}
			\caption{Linear reward model}
		\end{subfigure}%
		\begin{subfigure}{0.5\textwidth}
			\centering
			\includegraphics[width=\linewidth]{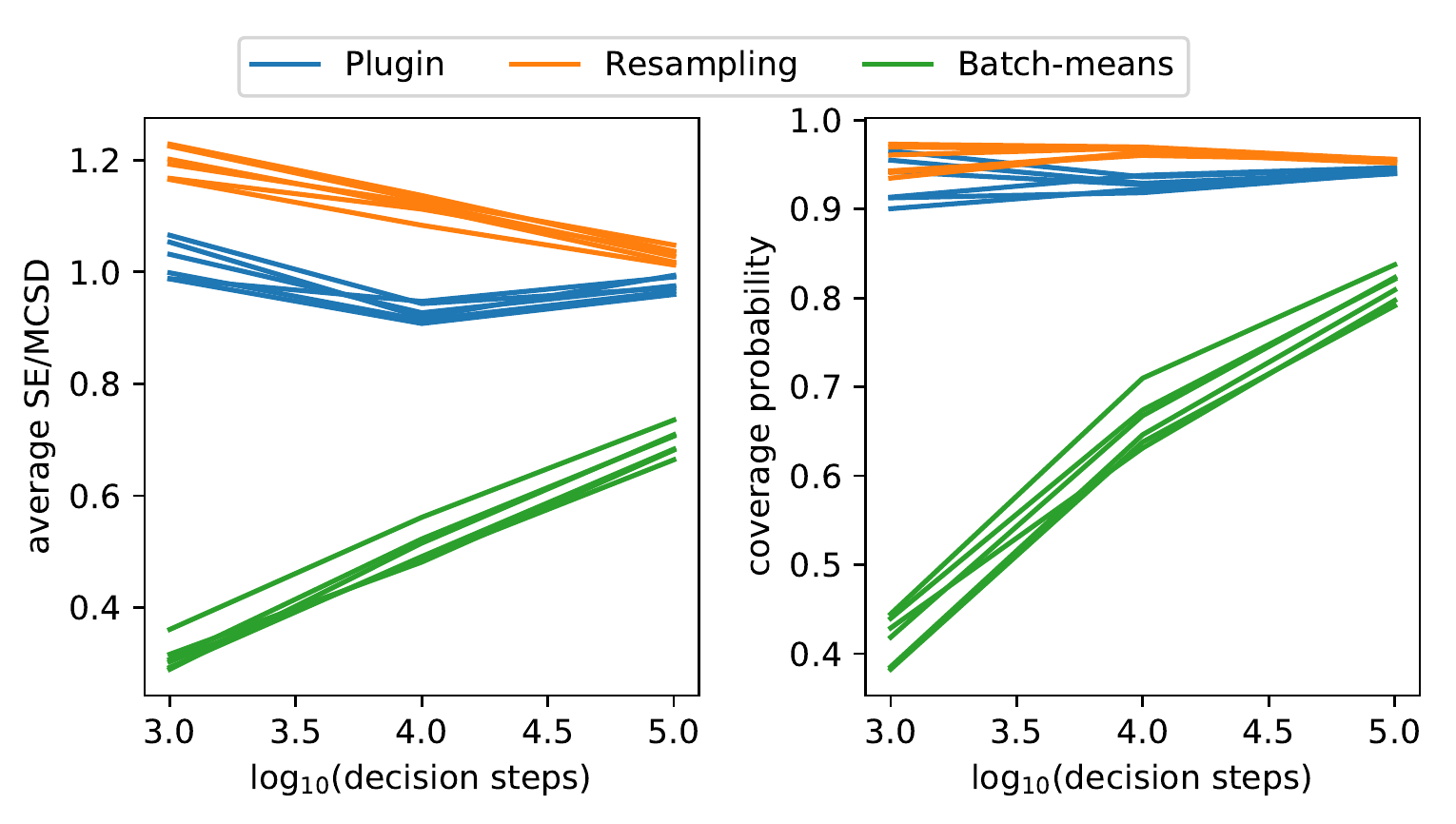}
			\caption{Logistic reward model}
		\end{subfigure}
		\caption{Comparison of variance estimation methods. The average standard error to Monte Carlo standard deviation and coverage probability are calculated from 5000 repeated experiments following the proposed SGD method with IPW gradients. The learning rate is $\alpha_t= 0.5t^{-0.501}$ and the exploration rate is $\varepsilon_t = 0.1$.}
		\label{fig:valueest_e1}
	\end{figure}
	
	\begin{figure}[!htbp]
		\centering
		\begin{subfigure}{0.5\textwidth}
			\centering
			\includegraphics[width=\linewidth]{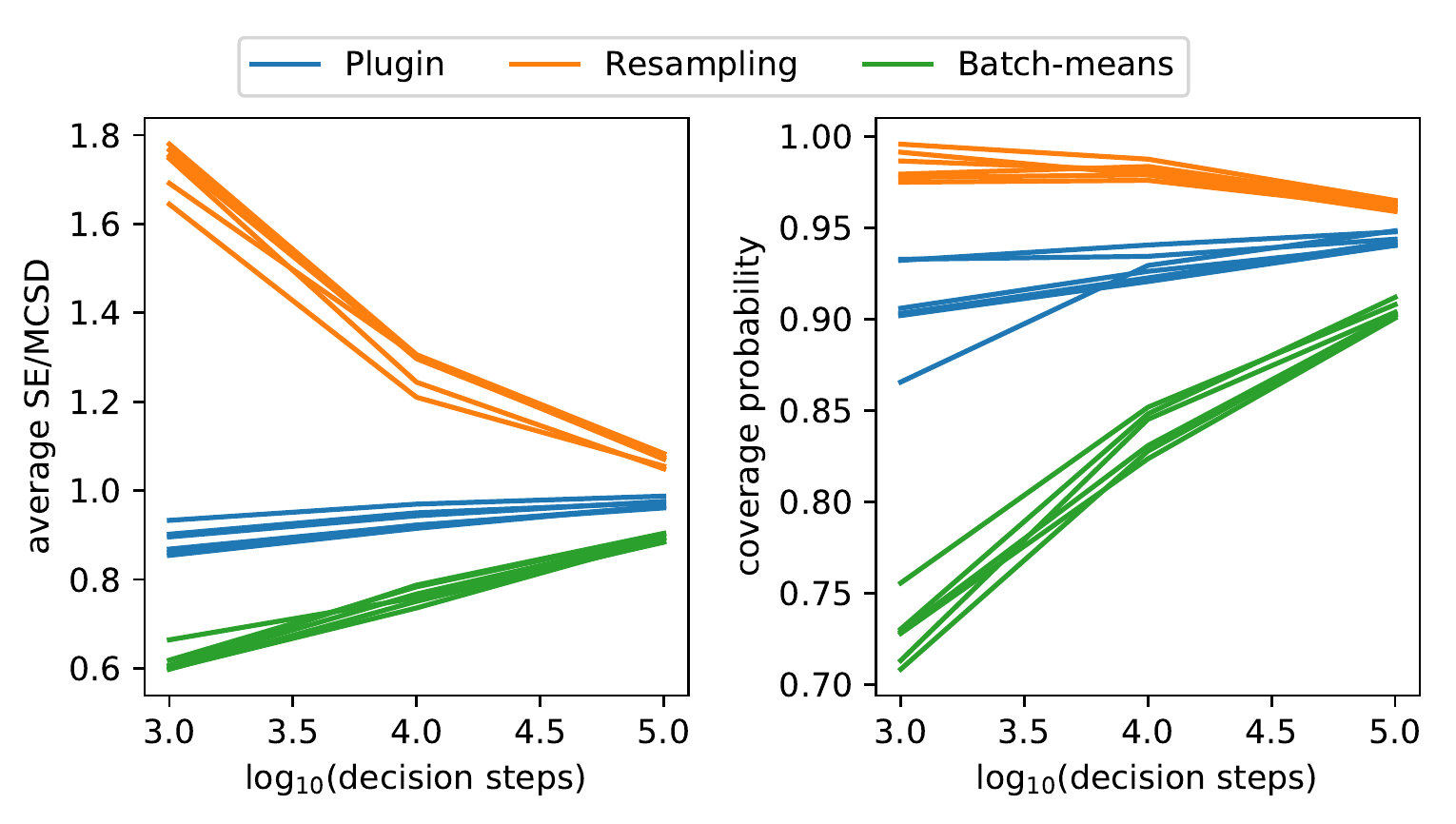}
			\caption{Linear reward model}
		\end{subfigure}%
		\begin{subfigure}{0.5\textwidth}
			\centering
			\includegraphics[width=\linewidth]{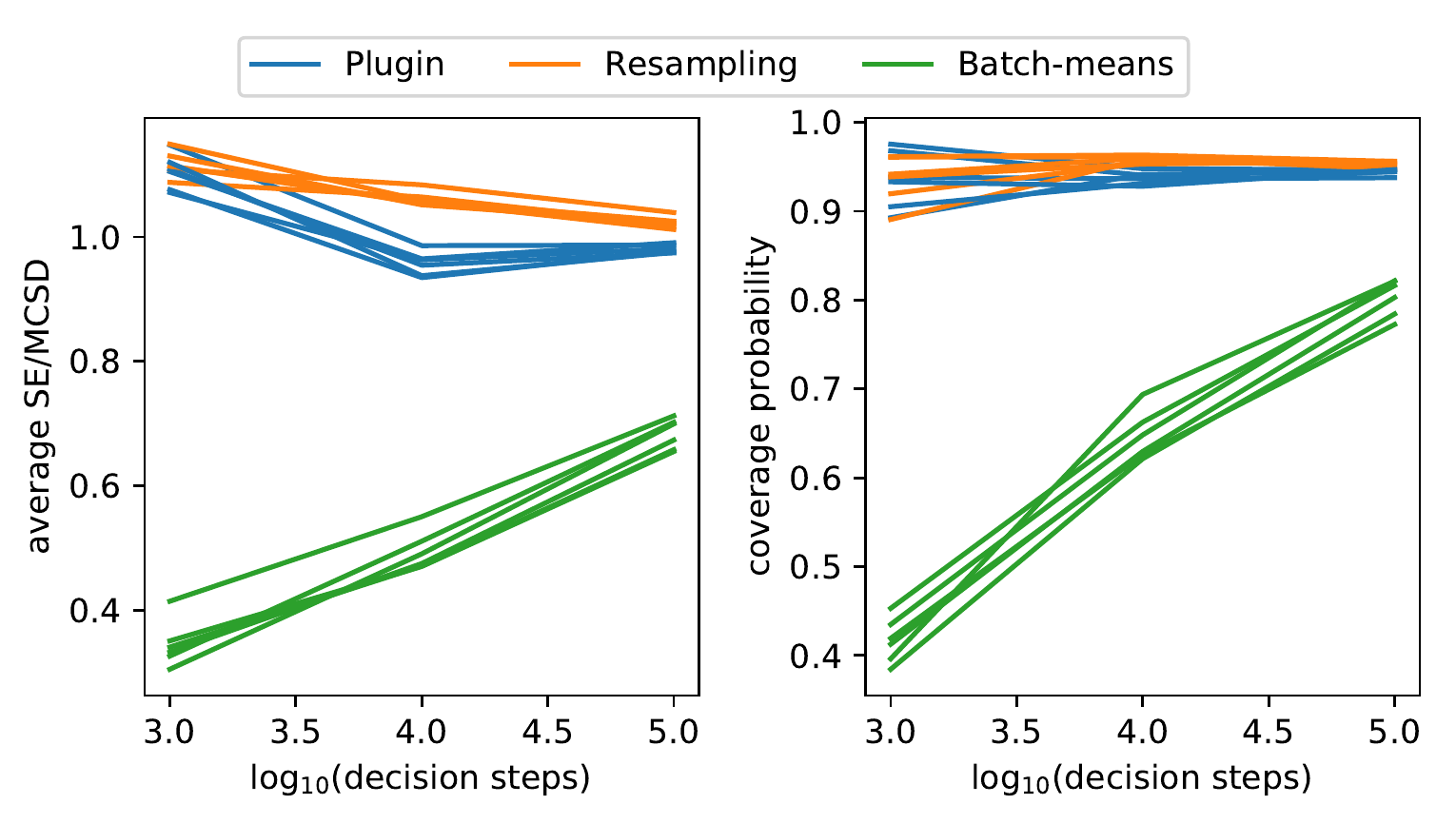}
			\caption{Logistic reward model}
		\end{subfigure}
		\caption{Comparison of variance estimation methods. The average standard error to Monte Carlo standard deviation and coverage probability are calculated from 5000 repeated experiments following the proposed SGD method with IPW gradients. The learning rate is $\alpha_t= 0.5t^{-0.501}$ and the exploration rate is $\varepsilon_t = t^{-0.3}\vee 0.1$.}
		\label{fig:valueest_e3}
	\end{figure}
	
	\newpage
	\bibliographystyle{apalike}
	\bibliography{manuscript_rev_unblinded}
	
\end{document}